\documentclass{article}

\usepackage{microtype}
\usepackage{graphicx}
\usepackage{subcaption}
\usepackage{booktabs} 

\usepackage{hyperref}

\usepackage[accepted]{icml2024}

\usepackage{amsmath}
\usepackage{amssymb}
\usepackage{mathtools}
\usepackage{amsthm}

\usepackage[capitalize,noabbrev]{cleveref}

\theoremstyle{plain}

\theoremstyle{definition}

\theoremstyle{remark}

\usepackage{placeins}
\usepackage{float}
\usepackage{enumitem}
\usepackage[linesnumbered,lined,ruled]{algorithm2e}
\usepackage{mdframed}
\SetKwComment{Comment}{$\triangleright$ }{}
\SetKwInOut{Input}{Input}
\SetKwInOut{Output}{Output}
\SetKwInOut{Return}{Return}
\SetKwFor{ForEach}{ForEach}{Do}{}
\SetKwFor{While}{While}{Do}{}
\SetKwIF{If}{ElseIf}{Else}{If}{Then}{Else If}{Else}{}

\usepackage{thmtools,thm-restate}
\theoremstyle{definition}
\newtheorem{example}{Example}
\AtBeginEnvironment{example}{%
  \pushQED{\qed}%
}
\AtEndEnvironment{example}{\popQED\endexample}

\newcommand{\X}{\mathbf{X}}

\newcommand\indep{\protect\mathpalette{\protect\independenT}{\perp}}
\def\independenT#1#2{\mathrel{\rlap{$#1#2$}\mkern2mu{#1#2}}}
\newcommand\notindependent{\!\perp\!\!\!\!\not\perp\!}

\newcommand{\smallcomment}[1]{\small{#1}}

\usepackage[textsize=tiny]{todonotes}

\icmltitlerunning{Detecting and Identifying Selection Structure in Sequential Data}

\begin{document}

\twocolumn[
\icmltitle{Detecting and Identifying Selection Structure in Sequential Data}



\icmlsetsymbol{equal}{*}

\begin{icmlauthorlist}
\icmlauthor{Yujia Zheng}{1}
\icmlauthor{Zeyu Tang}{1}
\icmlauthor{Yiwen Qiu}{1}
\icmlauthor{Bernhard Schölkopf}{2}
\icmlauthor{Kun Zhang}{1,3}
\end{icmlauthorlist}

\icmlaffiliation{1}{Carnegie Mellon University}
\icmlaffiliation{2}{Max Planck Institute for Intelligent Systems}
\icmlaffiliation{3}{MBZUAI}


\icmlkeywords{Machine Learning, ICML}

\vskip 0.3in
]



\printAffiliationsAndNotice{}  

\begin{abstract}
    \looseness=-1
   We argue that the selective inclusion of data points based on latent objectives is common in practical situations, such as music sequences. Since this selection process often distorts statistical analysis, previous work primarily views it as a bias to be corrected and proposes various methods to mitigate its effect. However, while controlling this bias is crucial, selection also offers an opportunity to provide a deeper insight into the hidden generation process, as it is a fundamental mechanism underlying what we observe. In particular, overlooking selection in sequential data can lead to an incomplete or overcomplicated inductive bias in modeling, such as assuming a universal autoregressive structure for all dependencies. Therefore, rather than merely viewing it as a bias, we explore the causal structure of selection in sequential data to delve deeper into the complete causal process. Specifically, we show that selection structure is identifiable without any parametric assumptions or interventional experiments. Moreover, even in cases where selection variables coexist with latent confounders, we still establish the nonparametric identifiability under appropriate structural conditions. Meanwhile, we also propose a provably correct algorithm to detect and identify selection structures as well as other types of dependencies. The framework has been validated empirically on both synthetic data and real-world music. 
\end{abstract}
\section{Introduction}

\looseness=-1
Selection arises from the preferential inclusion of certain data points based on latent variables that are dependent on some of the observed variables \citep{heckman1979sample}. This selection process not only introduces bias into statistical inferences but is also a fundamental aspect of the data-generating process in various applications. For instance, in composing music, composers are guided by specific artistic goals or themes, leading them to selectively choose certain patterns of musical combinations (as combinations of basic elements) from their mind, thereby introducing dependencies among the basic elements in the music sequences \citep{schoenberg1967musiccomp}. These intentional but unmeasured selections, together with the contextual information, shape the structure of the compositions. A comprehensive understanding of the selection structure is essential for uncovering the underlying causal process and making use of it.

In sequential data, the understanding of selection plays a vital role. One essential question is whether selection leaves unique data dependence patterns that cannot be well explained by direct causal relations or latent confounding. Interestingly, as we will see in this paper, the answer is yes. Consequently, overlooking selection in such data can result in the introduction of incomplete or overcomplicated dependence models for the data. For instance, due to the sequential nature of the data, an autoregressive structure is usually assumed for the entire sequence \citep{radford2018improving}, where each variable is considered primarily a function of its predecessors, perhaps with attention, and hence might be overcomplicated. Even when latent variables are incorporated to account for more complex dependencies, treating them solely as confounders amounts to a restricted generative view--as we will show, latent confounders and latent selection are distinguishable based on observational data, indicating their different footprints in the data. These perspectives neglect the potential influence of latent selection structures, such as the possibility that some dependencies among observed variables in the data sequence are due to a latent sampling criterion acting as the selection variable, rather than direct causal relations between variables. Confusing selection with latent confounding or direct causal relations can distort our understanding of the generative process. Therefore, discovering and integrating selection structure may be essential in the modeling of sequential data, implementing a more veridical inductive bias.

\looseness=-1
Previous research related to selection mainly focuses on controlling the impact of selection in data for inference or discovery \citep{spirtes1995causal, hernan2004structural, zhang2008completeness, bareinboim2014recovering, zhang2016identifiability, CorreaTianBareinboim2019, forre2020causal, versteeg2022local, chen2024modeling}. These studies propose various conditions to mitigate the distortions caused by selection, aiming to achieve reliable statistical analysis even in the presence of selection bias, which holds significant practical value. At the same time, while these works constitute significant progress, they primarily view selection as a bias to be alleviated or a distortion to be overcome, rather than as an integral component of the data generation or selection process that requires further structural understanding. The FCI algorithm \citep{spirtes1995causal, zhang2008completeness} aims to discover ancestral relations up to an equivalence class in the presence of latent confounders and selection bias, but significant ambiguities remain for the selection structure. To truly discover the mechanism underlying the data, it is natural to explore deeper into the structure of selection, such as differentiating between dependencies caused by selection and those arising from confounding or direct causal relations. Recent work by \citet{kaltenpoth2023identify} introduces methods to detect whether or not selection bias is an issue for the considered population under specific parametric assumptions. Nevertheless, the results do not indicate any structure of the selection process, leaving a gap in comprehensively understanding the role of selection.

\looseness=-1
In this paper, we explore the selection structure in sequential data to gain insights into underlying mechanisms through the additional lens of selection. We first prove the identifiability of the selection structure without any parametric assumptions and interventional experiments (Theorem~\ref{thm:1}). Specifically, we can detect and identify different types of dependencies, including selection, direct causal relations, or both, in sequential data. Furthermore, even in cases with latent confounders, which add significant complexity, we prove that the selection structure can still be identified under additional graphical conditions (Theorem~\ref{thm:2}). Meanwhile, we also propose a provably correct algorithm to identify selection structure as well as other types of dependencies. We thus establish one of the first frameworks for uncovering the general structure of selection in sequential data from only observational data. The theoretical claims as well as the scalability of the algorithm have been validated empirically in various settings. Results on real-world music not only align with existing knowledge but also hint at new discoveries.

\begin{figure}
    \centering
    \vspace{0.5em}
    \includegraphics[height=10em]{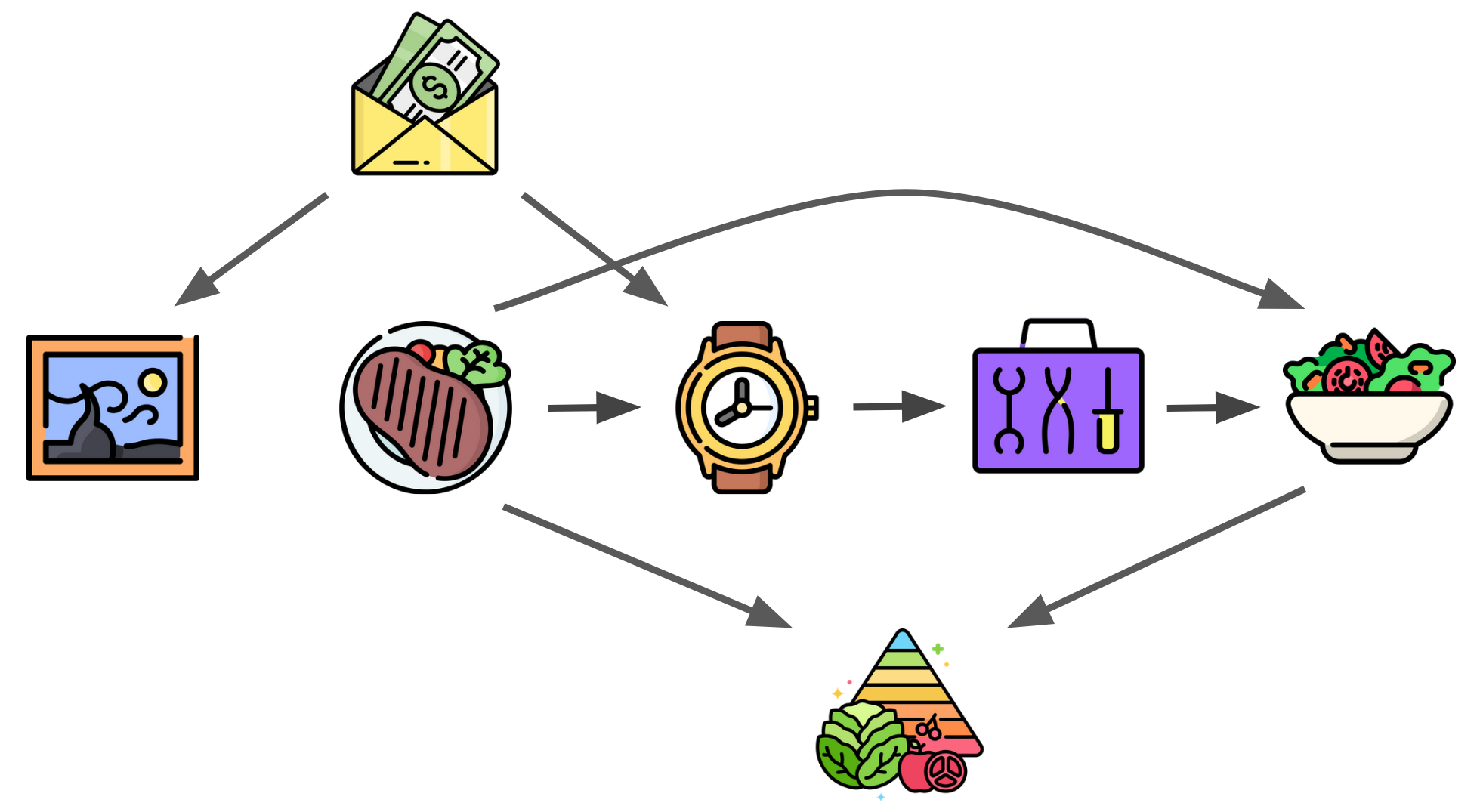}
    \caption{
    \textbf{Example of the selection structure} in the presence of latent confounders and direct relations. The figure depicts a sequence in a user's shopping history with a healthy lifestyle \includegraphics[width=2.5mm]{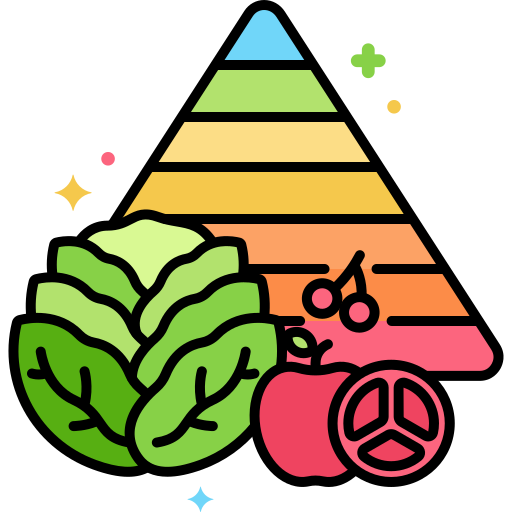}. Initially, the user buys an expensive painting \includegraphics[width=2.5mm]{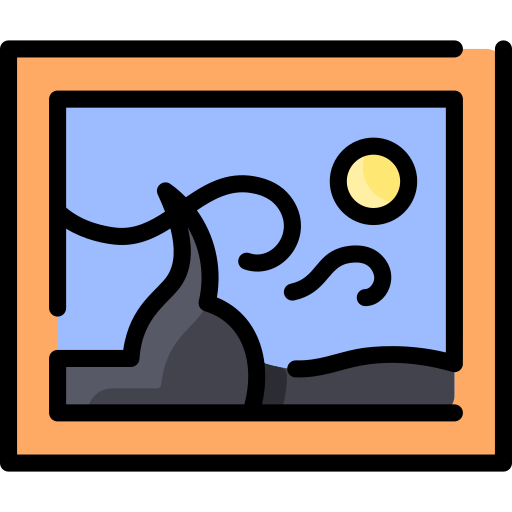} from a gallery. Then the user visits a mall and first enjoys a steak meal \includegraphics[width=2.5mm]{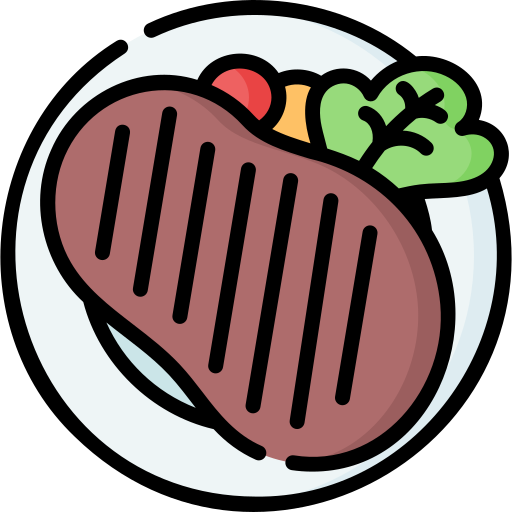} at a restaurant. After that, the user spots a nearby watch store and then buys a pricey watch \includegraphics[width=2.5mm]{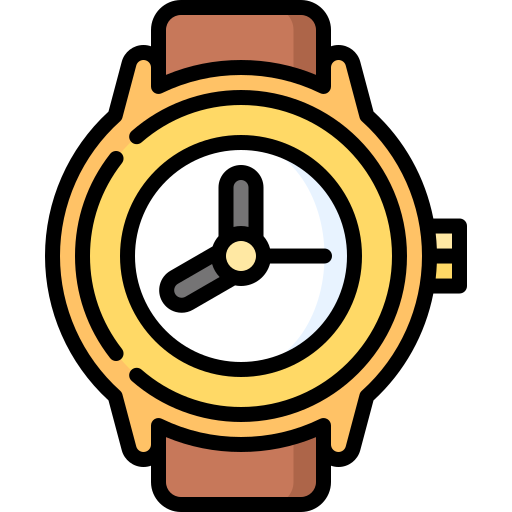} with its accompanying toolkit \includegraphics[width=2.5mm]{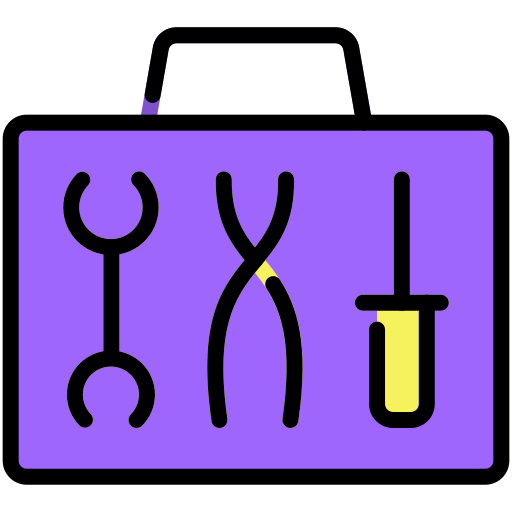} at the mall. Then, to adhere to a balanced diet, the user purchases salads \includegraphics[width=2.5mm]{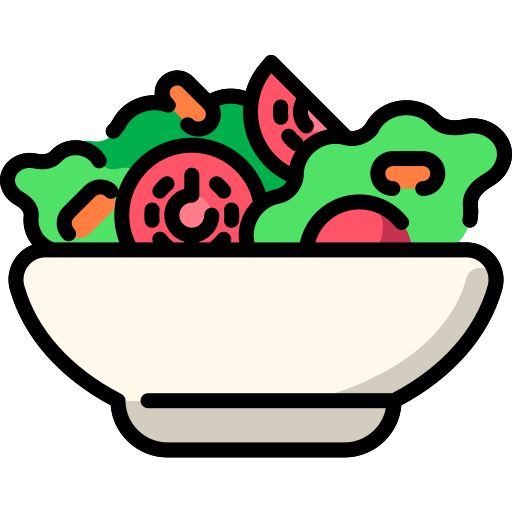} from another nearby store within the mall. This data sequence was selected for a study investigating the daily habits of individuals identified as leading a healthy lifestyle. Therefore, the healthy lifestyle \includegraphics[width=2.5mm]{figures/icon/health.png} acts as a \textbf{selection variable}, of which one of the contributing factors is a balanced diet, exemplified by choices such as steak \includegraphics[width=2.5mm]{figures/icon/steak.png} and salad \includegraphics[width=2.5mm]{figures/icon/salad.png}. Concurrently, the user's financial status \includegraphics[width=2.5mm]{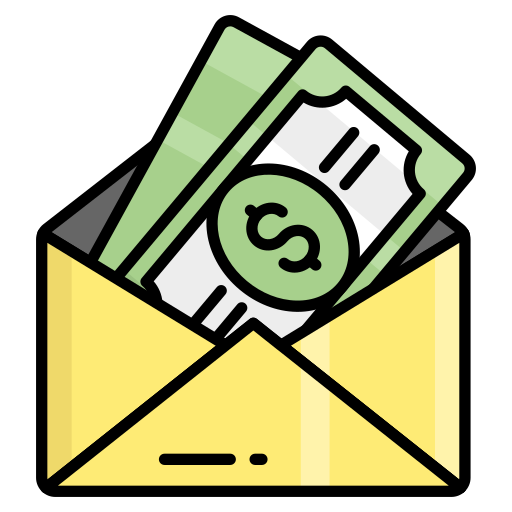} is a latent confounder (a common cause) for the user to buy both the painting \includegraphics[width=2.5mm]{figures/icon/painting.png} and the watch \includegraphics[width=2.5mm]{figures/icon/watch.png}. While there is no direct causal relation between having the steak \includegraphics[width=2.5mm]{figures/icon/steak.png} and the painting \includegraphics[width=2.5mm]{figures/icon/painting.png}, buying salad \includegraphics[width=2.5mm]{figures/icon/salad.png} is directly caused by having the steak \includegraphics[width=2.5mm]{figures/icon/steak.png}.
    \vspace{-0.9em}}
    \label{fig:example}
\end{figure}

\vspace{-0.3em}
\section{Preliminaries}
\label{sec:pre}

The data $\mathbf{X} = \{X_1, X_2, \cdots, X_N\}$ consists of observed variables where each $X_i$ represents an individual one. A crucial aspect of this data is its sequential structure, characterized by a property that if there exists any causal relation $X_i \rightarrow X_j$, where $X_i, X_j \in \mathbf{X}$, it holds that $i < j$. This structure is especially pertinent in temporal data, encapsulating the principle that future events cannot influence the past. The underlying causal structure is defined by a directed acyclic graph $\mathcal{G} = \{\mathbf{V}, \mathbf{E}\}$, where $\mathbf{V} = \{\mathbf{X}, \mathbf{S}, \mathbf{C}\}$ encapsulates all the observed variables $\mathbf{X}$, latent selection variables $\mathbf{S}$, and latent confounders $\mathbf{C}$. Each directed edge represents a direct causal relation.

\looseness=-1
The selection is viewed as follows: Let $S_k \in \mathbf{S}$ denotes a selection variable, which is part of the data-generating process. We can only observe the data points for which the selection criterion is met, i.e., $S_k=1$. As a consequence, the data distribution is actually the conditional one $P(\mathbf{X}|\mathbf{S})$. Moreover, selection variables depend on observed variables and introduce dependence between them. For example, in $X_i \rightarrow S_k \leftarrow X_j$, there is always a dependence between $X_i$ and $X_j$ since $S_k$ is always being conditioned on. To differentiate selection variables from others, we illustrate them using double-circle nodes.
Since our task is to distinguish different types of dependencies in sequential data, we assume that all selection variables depend on more than one observed variable. In addition, we denote a latent confounder as $C_k \in \mathbf{C}$, which is a hidden common cause of two (or more) observed variables. We allow the coexistence of latent confounders and selection variables, provided that there are no direct causal relations between these latent confounders and selection variables. Moreover, we define different types of dependencies between variables as follows, with an illustrative example including all of them in Figure \ref{fig:example}.

\begin{restatable}{definition}{selection_pair}(Selection Pair)
    A pair $(X_i, X_j)$ is a \textit{selection pair}, denoted as ${(X_i, X_j)}_\mathrm{s}$, if there is a selection variable $S_k$ such that $X_i \rightarrow S_k \leftarrow X_j$.
\end{restatable}

\begin{restatable}{definition}{direct_relation}(Direct Relation)
    A pair $(X_i, X_j)$ is a \textit{direct relation}, denoted as ${(X_i, X_j)}_\mathrm{x}$, if there exists a directed edge $X_i \rightarrow X_j$  in the causal graph. It is a higher-order direct relation if $j > i+1$.
\end{restatable}

\begin{restatable}{definition}{confounded_pair}(Confounded Pair)
    A pair $(X_i, X_j)$ is a \textit{confounded pair}, denoted as ${(X_i, X_j)}_\mathrm{c}$, if there is a latent confounder $C_k$ such that $X_i \leftarrow C_k \rightarrow X_j$.
\end{restatable}

Given only the observed variables $\mathbf{X}$, our goal is to identify these different types of dependencies. Specifically, we aim to determine the specific type for all \textit{dependent pairs} among measured variables in the sequential data. We denote a pair $(X_i, X_j)$ as a \textit{dependent pair} $(X_i, X_j)_\mathrm{d}$, if it is a selection pair, direct relation, or confounded pair, i.e., $d \in \{\mathrm{s}, \mathrm{x}, \mathrm{c}\}$.
Note that there could be multiple selection variables or latent confounders affecting a pair of variables, and we represent them as a single entity (e.g., $S_k$ or $C_k$) for simplicity.

\vspace{-0.3em}

\section{Identifiability of Selection Structure}
\label{sec:theory}

\begin{figure}[t]
    \centering
    \includegraphics[height=5.1em]{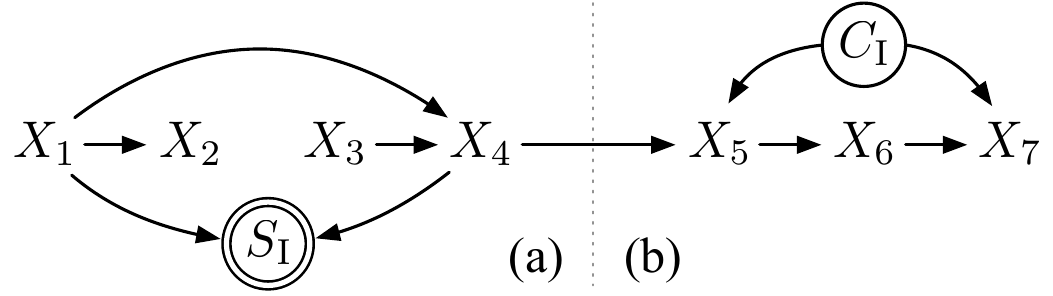}
     \caption{A running example for our identifiability theory.}
    \label{fig:running_example}
    \vspace{-0.6em}
\end{figure}

Is the structure identifiable when selection coexists with other types of dependencies, as illustrated in Figure \ref{fig:running_example}?

\looseness=-1
To answer this, we first establish the identifiability of selection structure without any parametric assumptions or interventional experiments (Section \ref{sec:thm_1}). Then we extend the identifiability result to cases with latent confounders (Section \ref{sec:thm_2}). After that, we illustrate alternative representations of the selection structures (Section \ref{sec:sparsity}). Finally, we propose a provably correct algorithm to identify selection structure as well as other types of dependencies (Section \ref{sec:algorithm}).

\vspace{-0.3em}
\subsection{Identifiability Without Latent Confounders}
\label{sec:thm_1}

We first show that the selection structure is identifiable in scenarios without latent confounders. For instance, in part (a) of Figure \ref{fig:running_example}, we need to distinguish the selection pair ($X_1 \rightarrow S_I \leftarrow X_4$) from the direct relation ($X_1 \rightarrow X_4$). Under some structural conditions, we establish the nonparametric identifiability in sequential data by producing Theorem \ref{thm:1}, which involves the following structural conditions:

\begin{restatable}{condition}{con1} (Structural Conditions).
\label{con:1}
\vspace{-0.5em}
\begin{enumerate}[label=\roman*.,ref=\roman*]
    \item For any ${(X_i, X_j)}_\mathrm{s}$, $(X_{j-1}, X_j)_\mathrm{d}$ exists with $j-1 \neq i$, but not $(X_i, X_{j+1})_\mathrm{d}$, and $X_j$ is not the effect of multiple higher-order direct relations. No $(X_i, X_{i+1})_\mathrm{x}$ if $j=i+2$. 
    \item For any ${(X_i, X_j)}_{\mathrm{d} \neq \mathrm{s}}$, $(X_{j-1}, X_j)_\mathrm{d}$ exists, and if $X_j$ is also the cause of a selection variable, $(X_i, X_{j-1})_\mathrm{d}$ does not exist. No $(X_i, X_{j+1})_\mathrm{d}$ if $j=i+2$.
\end{enumerate}
\end{restatable}

\begin{restatable}{theorem}{thmI}
\label{thm:1}
Let the observed data be a large enough sample generated by a model defined in Section \ref{sec:pre}. In addition to the faithfulness assumption and Markov condition, suppose the following assumptions hold:
\vspace{-0.5em}
\begin{enumerate}[label=\roman*.,ref=\roman*]
  \item (No future influencing the past): For any edge between observed variables, e.g., $X_i \rightarrow X_j$, $i < j$.
  \label{assum:t1a1}
  \item (No latent confounders): $\mathbf{C} = \emptyset$.
  \label{assum:t1a2}
  \item (Exclude degenerate structures): Structural conditions \ref{con:1} are satisfied.
  \label{assum:t1a3}
\end{enumerate}
  \vspace{-0.5em}
Then all selection pairs and direct relations in the causal graph $\mathcal{G}$ are identifiable.
\end{restatable}

\textbf{Proof (Intuition and Sketch).} \ \
\looseness = -1
The complete proof is in Appendix \ref{sec:proof_thm_1}. In particular, we develop an algorithm, Algorithm \ref{alg:short_alg} (with its full version in Appendix \ref{sec:complete_algorithm}), and prove that it can identify all selection pairs and direct relations under the assumptions. We first show that Stage One of the algorithm (lines \ref{ls3}-\ref{ls5}) detects all dependencies but not their specific types. This is achieved by conditional independence tests (CI tests) given a set of variables that blocks all other potential paths between $X_i$ and $X_j$. Subsequently, Stage Two (lines \ref{ls6}-\ref{ls8}) determines the specific type of dependence, whether it be a selection pair, a direct relation, or both. 

Consider the two examples in Figure \ref{fig:thm_1}, since the latent selection variable is always conditioned on, the only structural difference between the direct relation $(X_i, X_j)_\mathrm{x}$ (Figure \ref{fig:thm_1}(b)) and the selection pair $(X_i, X_j)_\mathrm{s}$ (Figure \ref{fig:thm_1}(a)) lies in whether $X_j$ is a collider on the path $\{X_i \cdots X_j \leftarrow \cdots X_k\}$. Thus, by blocking all other paths between $X_i$ and $X_k$, we can identify whether $X_j$ is a collider on the path by the designed CI tests. Specifically, $X_i$ is independent of $X_k$ conditioning on $X_j$ and some other variables if we have $(X_i, X_j)_\mathrm{s}$ but not $(X_i, X_j)_\mathrm{x}$. Thus, we are able to distinguish selection pairs from direct relations. For a pair of variables that constitute both a selection pair and a direct relation (e.g., Figure \ref{fig:running_example}(a)), we can identify them since they are dependent on each other, regardless of whether we condition on $X_j$ or not. At the same time, given the flexibility of the potential structures, there are many different cases that need to be considered, and we include them all in Appendix \ref{sec:proof_thm_1}.

\begin{figure}[t]
\centering
\begin{subfigure}{.5\columnwidth}
    \centering
    \includegraphics[height=3em]{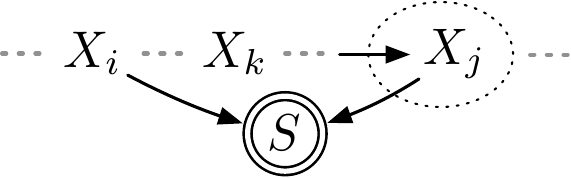}
    \caption{}
\end{subfigure}%
\begin{subfigure}{.5\columnwidth}
    \centering
    \includegraphics[height=3em]{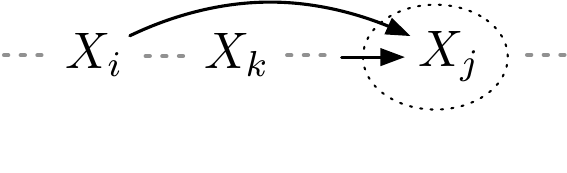}
    \caption{}
\end{subfigure}
\vspace{-2.2em}
\caption{Illustration of the intuition for the proof of Theorem \ref{thm:1}. In general, we distinguish the selection pair (a) from the direct relation (b) based on whether $X_j$ is a collider on the path.}
\label{fig:thm_1}
\end{figure}

\begin{figure}[t]
\centering
\begin{subfigure}{.5\columnwidth}
    \centering
    \includegraphics[height=3.4em]{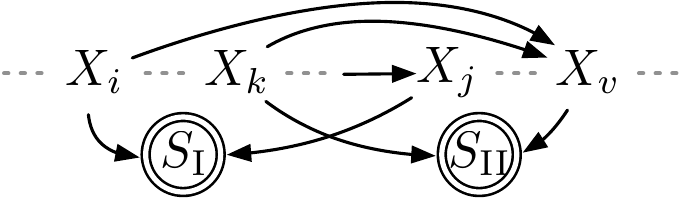}
    \caption{}
\end{subfigure}%
\begin{subfigure}{.5\columnwidth}
    \centering
    \includegraphics[height=3.4em]{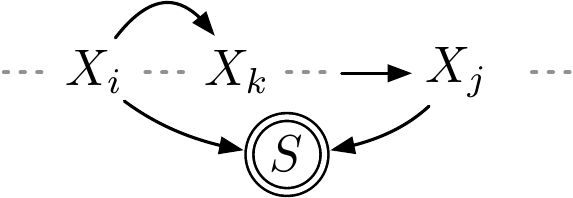}
    \caption{}
\end{subfigure}
\vspace{-2.2em}
\caption{Examples of the degenerate structures in Theorem \ref{thm:1}.}
\label{fig:con_1}
\vspace{-1.2em}
\end{figure}

\textbf{Discussion on Assumptions.} \ \
The first two assumptions are from the considered setting. Assumption \ref{thm:1}-\ref{assum:t1a1} corresponds to the sequential structure of the data that is common in various scenarios, including but not limited to temporal data. Note that we do not assume a simple chain structure between variables as there could be arbitrary disconnected pairs in the sequence. Assumption \ref{thm:1}-\ref{assum:t1a2} ensures that the only latent variables are selection variables.


However, since all selection variables are unobserved and we are only given the selected sub-population, uncovering any information regarding the latent selection is clearly non-trivial. For instance, \citet{kaltenpoth2023identify} employ parametric assumptions solely to detect the presence of selection bias in the entire dataset without any further structural information. In contrast, we aim to not only detect the existence but also identify the structure of selection, and thus some additional assumptions are necessary. Instead of introducing parametric constraints, we consider structural conditions to exclude several degenerate cases. 

\looseness=-1
The intuition of these structural conditions (Condition \ref{con:1} in Assumption \ref{thm:1}-\ref{assum:t1a3}) stems naturally from the difference between different types of dependencies. As outlined in the proof sketch, the main focus is on determining whether $X_j$ acts as a collider on the path $\{X_i \cdots X_j \leftarrow \cdots X_k\}$ (see Figure \ref{fig:thm_1}). For this, it is crucial that $X_{j-1}$ is adjacent to $X_j$, indicated by the existence of $(X_{j-1}, X_j)_\mathrm{d}$. Without this adjacency, $X_j$ cannot function as a collider. Additionally, we also need to block all other paths between $X_i$ and $X_k$. Achieving this requires a carefully chosen conditioning set for d-separations \citep{pearl2000models}, accompanied by assumptions to exclude degenerate cases where the inclusion or exclusion of a variable in the conditioning set invariably leads to a d-connected path. For instance, if an observed variable $X_v$ ($v \neq j$) is caused by both $X_i$ and $X_k$ and is part of a selection pair $(X_i, X_v)_\mathrm{s}$ (e.g., Figure \ref{fig:con_1}(a)), a d-connected path between $X_i$ and $X_k$ will persist regardless of conditioning on $X_v$. Moreover, it is necessary to ensure the existence of such an $X_k$ where $(X_i, X_k)$ does not form a dependent pair (e.g., Figure \ref{fig:con_1}(b)); otherwise, they will always be d-connected. Together with the algorithm (Algorithm \ref{alg:short_alg}), some structural conditions are needed to address cases like these. As a result, we introduce Condition \ref{con:1} for the nonparametric identifiability of the selection structure given only observed data.

\textbf{Implications.} \ \
\looseness=-1
By establishing the nonparametric identifiability of the selection structure, we demonstrate the feasibility of going beyond the conventional generative view to examine specific selection processes in sequential data. This offers additional insights, potentially revealing overlooked aspects of certain inductive biases. For example, if the dependence between two variables is found to be due to selection rather than a direct causal relation, modeling all observed variables with an autoregressive structure could introduce significant bias. Consider a scenario in financial market analysis, where an algorithm examines stock return movements over time. In this sequence, $X_i$ represents the stock return at time $i$, and $X_j$ represents the stock price at a later time $j$ (where $j > i$). A common assumption might be a direct relation, suggesting that $X_i$ influences $X_j$. However, if the dataset only includes days where both $X_i$ and $X_j$ exceed a certain threshold – perhaps due to an internal data selection policy focusing on days with significant market activity – this introduces a selection structure. If selection is overlooked in data analysis, we might incorrectly interpret $X_j$ as a variable directly influenced by $X_i$, and consequently model $X_j$ as being generated by $X_i$ along with other factors. 

\looseness=-1
Furthermore, the knowledge of latent selection structures also has the potential to guide various machine learning tasks by providing more complete causal insights into the underlying mechanisms. For tasks that inherently require a trustworthy modeling process, such as interpretability, including selection structure is naturally beneficial, as it forms a crucial part of the overall causal process. Simultaneously, even for tasks that may not seem directly to benefit from a causal structure, such as pretraining, understanding the selection structure can still be helpful. In most tasks, striking a balance between inductive bias and data is essential. The revealed selection structure contributes to an inductive bias that more accurately reflects the hidden causal processes, thus enhancing the efficiency of data utilization. While a model trained with a less comprehensive inductive bias, such as a universal autoregressive structure, can still be effective, it often requires a substantial amount of data and incurs significant training costs. Incorporating potential selection structures makes it possible to adopt a more appropriate inductive bias, achieving our goals in a more cost-effective manner.

\vspace{-0.3em}
\subsection{Identifiability With Latent Confounders}
\label{sec:thm_2}

In the previous section, we established the identifiability of the selection structure in sequential data (Theorem \ref{thm:1}). While this provides a way to delve deeper into latent selection, the setting of the absence of latent confounders might not hold in certain scenarios, given that latent confounders often exist. For example, we can identify part (a) of Figure \ref{fig:running_example} but not part (b). Motivated by that, we extend our theory and demonstrate the identifiability of the selection structure even in the presence of latent confounders.

\begin{restatable}{condition}{con2}
\label{con:2}
For any $(X_i, X_j)_\mathrm{c}$, there exists no $(X_{i-1}, X_{j-1})_\mathrm{d}$, $(X_{i-1}, X_{j})_\mathrm{s}$, $(X_{i-1}, X_{i+1})_\mathrm{x}$, $(X_i, X_j)_\mathrm{s}$, $(X_{i+1}, X_{j-1})_\mathrm{d}$, and $(X_{i+1}, X_{j+1})_\mathrm{x}$. Neither $X_i$ nor $X_j$ is in another confounded pair or higher-order direct relation.
\end{restatable}

\begin{restatable}{theorem}{thmII}
\label{thm:2}
Let the observed data be a large enough sample generated by a model defined in Section \ref{sec:pre}. In addition to the faithfulness assumption and Markov condition, suppose the following assumptions hold:
\vspace{-0.5em}
\begin{enumerate}[label=\roman*.,ref=\roman*]
  \item (No future influencing the past): For any edge between observed variables, e.g., $X_i \rightarrow X_j$, $i < j$.
  \label{assum:t2a1}
  \item (Exclude degenerate structures): Structural conditions \ref{con:1} and \ref{con:2} are satisfied.
  \label{assum:t2a2}
\end{enumerate}
  \vspace{-0.5em}
Then all selection pairs, direct relations, and confounded pairs in the causal graph $\mathcal{G}$ are identifiable.
\end{restatable}

\textbf{Proof (Intuition and Sketch).} \ \
\looseness=-1
The complete proof is in Appendix \ref{sec:proof_thm_2}. Similar to Theorem \ref{thm:1}, we prove that the algorithm (Algorithm \ref{alg:short_alg}) identifies the dependent pairs at Stage One and distinguishes selection pairs from others after Stage Two. However, due to the possibility of latent confounders, there exist some spurious direct relations corresponding to the unidentified confounded pairs after Stage Two. This is because, when conditioning on $X_i$ for $(X_i, X_j)_\mathrm{c}$ (e.g., Figure \ref{fig:thm_2}), $X_{i-1}$ and $X_j$ will be d-connected by the path $\{X_{i-1} \rightarrow X_i \leftarrow C \rightarrow X_j$. Because both the confounded pair and direct relation have an arrowhead pointing to $X_j$, if there exists $(X_i, X_j)_\mathrm{c}$, both $(X_i, X_j)$ and $(X_{i-1}, X_j)$ will be identified as direct relations after Stage Two according to Theorem \ref{thm:1}. Thus, we introduce Stage Three to identify spurious direct relations, specifically to distinguish Figure \ref{fig:thm_2}(a) from Figure \ref{fig:thm_2}(b), which is achieved by designed CI tests. Meanwhile, since the introduced latent confounders largely broaden the space of possible structures, we elaborate all potential cases in the complete proof (Appendix \ref{sec:proof_thm_2}).

\begin{figure}[t]
\centering
\begin{subfigure}{.5\columnwidth}
    \centering
    \includegraphics[height=2.6em]{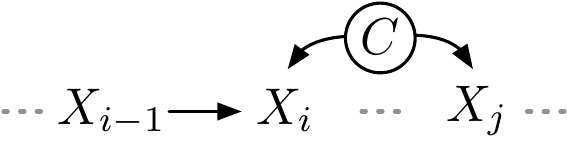}
    \caption{}
\end{subfigure}%
\begin{subfigure}{.5\columnwidth}
    \centering
    \includegraphics[height=2.6em]{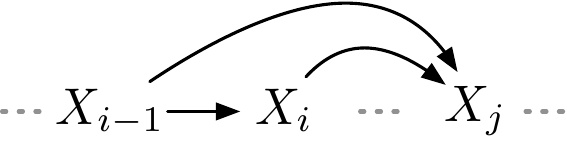}
    \caption{}
\end{subfigure}
\vspace{-2em}
\caption{Intuition for identifying confounded pairs in Theorem \ref{thm:2}. In general, we distinguish the confounded pair (a) from the spurious direct relations (b) by identifying the collider $X_i$.}
\label{fig:thm_2}
\end{figure}

\begin{figure}[t]
\centering
\begin{subfigure}{.5\columnwidth}
    \centering
    \includegraphics[height=4em]{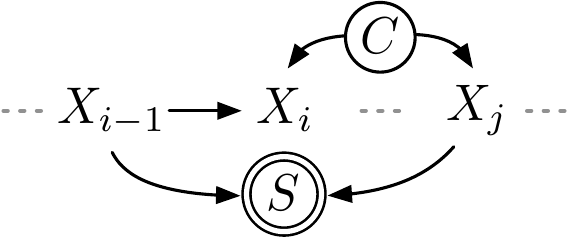}
    \caption{}
\end{subfigure}%
\begin{subfigure}{.5\columnwidth}
    \centering
    \includegraphics[height=4em]{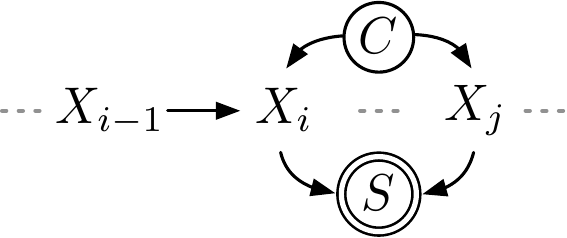}
    \caption{}
\end{subfigure}
\vspace{-2em}
\caption{Examples of the degenerate structures in Theorem \ref{thm:2}.}
\label{fig:con_2}
\vspace{-1.2em}
\end{figure}

\textbf{Discussion on Assumptions.} \ \
\looseness=-1
Compared to Theorem \ref{thm:1}, we remove the assumption of no latent confounders in Theorem \ref{thm:2}, thus increasing the applicability of our theory. However, this extension comes with a trade-off: the inclusion of latent confounders alongside latent selection variables significantly increases the complexity of the problem. To address this, it is necessary to have additional structural conditions (i.e., Conditions \ref{con:2} in Assumption \ref{thm:2}-\ref{assum:t2a2}).

\looseness=-1
As mentioned in the proof sketch, to identify the added latent confounders, we need to distinguish them from the spurious direct relations after Stage Two, i.e., differentiating Figures \ref{fig:thm_2}(a) with \ref{fig:thm_2}(b). Notably, in Figure \ref{fig:thm_2}(a), $X_i$ acts as a collider on the path $\{X_{i-1} \rightarrow X_{i} \leftarrow C \rightarrow X_{j}\}$, a structure absent in Figure \ref{fig:thm_2}(b). This allows us to detect confounded pairs by testing for additional independence created by this collider structure. A crucial requirement is that $X_{i-1}$ and $X_{i}$ do not form dependent pairs with $X_{j}$, as this would invariably lead to dependence between $X_{i-1}$ and $X_j$. For example, if $(X_{i-1}, X_{j})_\mathrm{d}$ exist, $X_{i-1}$ and $X_j$ remain d-connected due to the non-inclusion of $X_{i-1}$ in the conditioning set (Figure \ref{fig:con_2}(a)). This is also the case if $(X_{i}, X_{j})_{\mathrm{d} \neq \mathrm{c}}$ exists (Figure \ref{fig:con_2}(b)), since $X_i$ should not be in the conditioning set for testing the extra independence introduced by $X_i$ acting as a collider on the path. Moreover, we need to make sure that the only d-connected path between $X_{i-1}$ and $X_j$ is $\{X_{i-1} \rightarrow X_{i} \leftarrow C \rightarrow X_{j}\}$ in Figure \ref{fig:thm_2}(a). This necessitates structural conditions, similar to those in Theorem \ref{thm:1}, to ensure that alternative paths can be effectively blocked using the designed CI tests. Additionally, the introduction of latent confounders must not obstruct the identification of other dependency types. For instance, conditions regarding confounders are needed for the algorithm to be able to identify the selection structure, such as the existence of $X_k$ in Stage Two, as outlined in the previous discussion on Theorem \ref{thm:1}. In conjunction with our proposed algorithm (Algorithm \ref{alg:short_alg}), we introduce Condition \ref{con:2} as an added structural constraint to facilitate the extension of our theory to scenarios that include latent confounders.

\textbf{Implications.} \ \
Theorem \ref{thm:2} shows that, even in the case where latent confounders exist, we can still identify selection structures in a manner. This is exciting since both selection variables and confounders are unobserved, yet we still prove that it is actually possible to distinguish them based on only observed data. Notably, this identifiability is established without imposing any parametric constraints or necessitating interventional experiments. 

In real-world scenarios, especially in sequential data, latent confounders often coexist with selection processes. For instance, policy shifts, serving as selection variables, can influence which patients are admitted at different times, thereby impacting health metrics sequentially in longitudinal health studies. At the same time, unobserved health conditions, i.e., latent confounders, may also influence these metrics. Accurately acknowledging the presence of both selection variables and latent confounders is key to developing more nuanced and complete understanding of the data and its underlying dynamics. Compared to Theorem \ref{thm:1}, which establishes identifiability of selection structures in the absence of latent confounders, Theorem \ref{thm:2} broadens the scope by encompassing latent confounders, enhancing both theoretical significance and practical utility.

\vspace{-0.3em}
\subsection{Alternative Representations of Selection Structure}
\label{sec:sparsity}

We have demonstrated that selection structures are identifiable from observed data, both in the presence and absence of latent confounders. At the same time, in line with the principle of simplicity or Occam's razor \citep{zhang2013comparison}, we may consider merging multiple selection variables into a single entity when it simplifies the overall structure.

\begin{restatable}{remark}{remarkI}
\label{remark:1}
For a set of observed variables, where every variable pair is a selection pair, the conditional dependence structure is equivalent to one where all observed variables in this set are direct causes of a single selection variable.
\end{restatable}

\begin{figure}[t]
\centering
\includegraphics[height=7em]{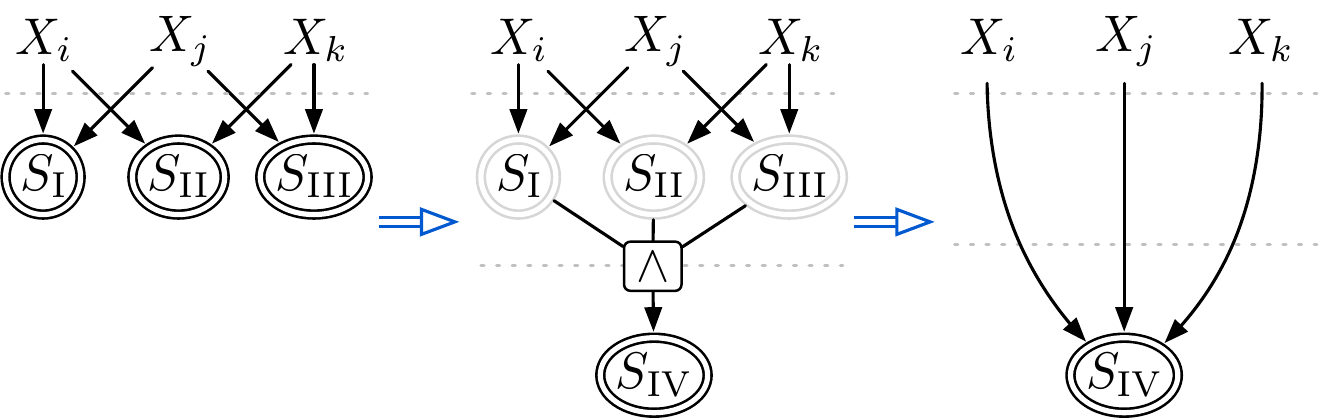}
\caption{Alternative representations of the selection structure.}
\label{fig:remark_1}
\vspace{-1.5em}
\end{figure}

This remark follows directly from the default conditioning of selection variables. If all variable pairs are selection pairs, conditioning on the selection variables d-connects each observed variable with others in the set. This is equivalent to a structure where all observed variables are d-connected by conditioning on a single selection variable acting as the common collider. In a scenario where every pair of variables within a set forms a selection pair, it is possible to construct a higher-level selection variable. This is achieved by merging all individual selection criteria using a logical AND conjunction, thereby consolidating multiple selection variables into one (e.g., Figure \ref{fig:remark_1}). For instance, in a health study assessing variables such as Heart Rate, Blood Pressure, and Cholesterol Levels, where each metric pair is affected by distinct factors like lifestyle, diet, and genetics, these factors can be unified into a singular selection variable. This composite variable, formed through the logical AND of all individual factors, simplifies the structure across the set of variables, reflecting a unified selection structure. Furthermore, if more than two observed variables depend on a single selection variable, we can always represent them in a pair-wise way, since the dependence patterns are exactly the same.


\newlength{\oldtextfloatsep}
\setlength{\oldtextfloatsep}{\textfloatsep}
\setlength{\textfloatsep}{-10em}
\begin{algorithm}[t]
    \caption{Identification of the Selection Structure (\textit{Sketch with full details in Appendix \ref{sec:complete_algorithm}}).}
    \label{alg:short_alg}
    \Input{Data $\mathbf{X}$}
    \Output{A causal graph $\mathcal{G}$}
    
    Initialize lists $\mathcal{L}$ and $\mathcal{R}$ as empty\;
    Initialize an empty edge set $\mathcal{E}$ and a set of vertices $\mathcal{V}$\;
    
    \tcp{\smallcomment{\textbf{Stage One}}}
    \ForEach{variable pair $(X_i, X_j)$}{ \label{ls3}
        Identify potential dependent pairs by conducting designed CI tests on $X_i$ and $X_j$\; \label{ls4}
        Update $\mathcal{L}$ and $\mathcal{R}$ with identified pairs\; \label{ls5}
    }
    
    \tcp{\smallcomment{\textbf{Stage Two}}}
    \ForEach{variable in $\mathcal{L}$ and $\mathcal{R}$}{ \label{ls6}
        Distinguish between selection pairs and direct relations by conducting designed CI tests\; \label{ls7}
        Update the edge set $\mathcal{E}$ accordingly\; \label{ls8}
    }
    
    \tcp{\smallcomment{\textbf{Stage Three} (only for Theorem \ref{thm:2})}}
    \ForEach{direct relation that may be confounded}{ \label{ls9}
        Identify confounded pairs while removing spurious dependencies\; \label{ls10}
        Adjust $\mathcal{E}$ based on confounded pairs\; \label{ls11}
    }
    \Return{$\mathcal{G} = \{\mathcal{V}, \mathcal{E}\}$}
\end{algorithm}
\setlength{\textfloatsep}{\oldtextfloatsep}

Additionally, structuring the selection process in this way may also help uncover a hierarchical nature of selection. This may reflect cognitive processes in decision-making. Take Figure \ref{fig:remark_1} for example, a composer crafting a musical piece might set a specific goal ($S_{\mathrm{IV}}$), such as style or topic, for the whole piece, then divide it into different lower-level goals ($S_{\mathrm{I}}$, $S_{\mathrm{II}}$, and $S_{\mathrm{III}}$) for individual passages.

\vspace{-0.3em}
\subsection{Identification Algorithm}
\label{sec:algorithm}

\looseness=-1
Alongside establishing identifiability, we propose an algorithm for the identification of selection structures, as well as direct relations and confounded pairs, using only observed data. A brief outline of it is presented as Algorithm \ref{alg:short_alg}, while comprehensive details can be found in Appendix \ref{sec:complete_algorithm} with an illustrative example in Appendix \ref{sec:example}. It should be noted that Stage Three of the algorithm is only for scenarios that include latent confounders, as addressed in Theorem \ref{thm:2}. Consequently, this stage is not needed in situations where latent confounders are absent, such as Theorem \ref{thm:1}. Our algorithm achieves a relatively lower complexity of $O(N^2)$, where $N$ denotes the number of observed variables, by searching in a pair-wise manner and focusing solely on essential CI tests required for differentiating various cases. This stands in contrast to the often exponential complexity observed in other constraint-based structure learning algorithms (e.g., the complexity of PC or FCI in the worst case \citep{spirtes2000causation}). However, we do not claim that our method is more computationally efficient, as it targets a new and distinct task.


\begin{figure*}[t]
    \centering
    \captionsetup{format=hang}

    \begin{subfigure}{.24\textwidth}
        \centering
        \includegraphics[height=8em]{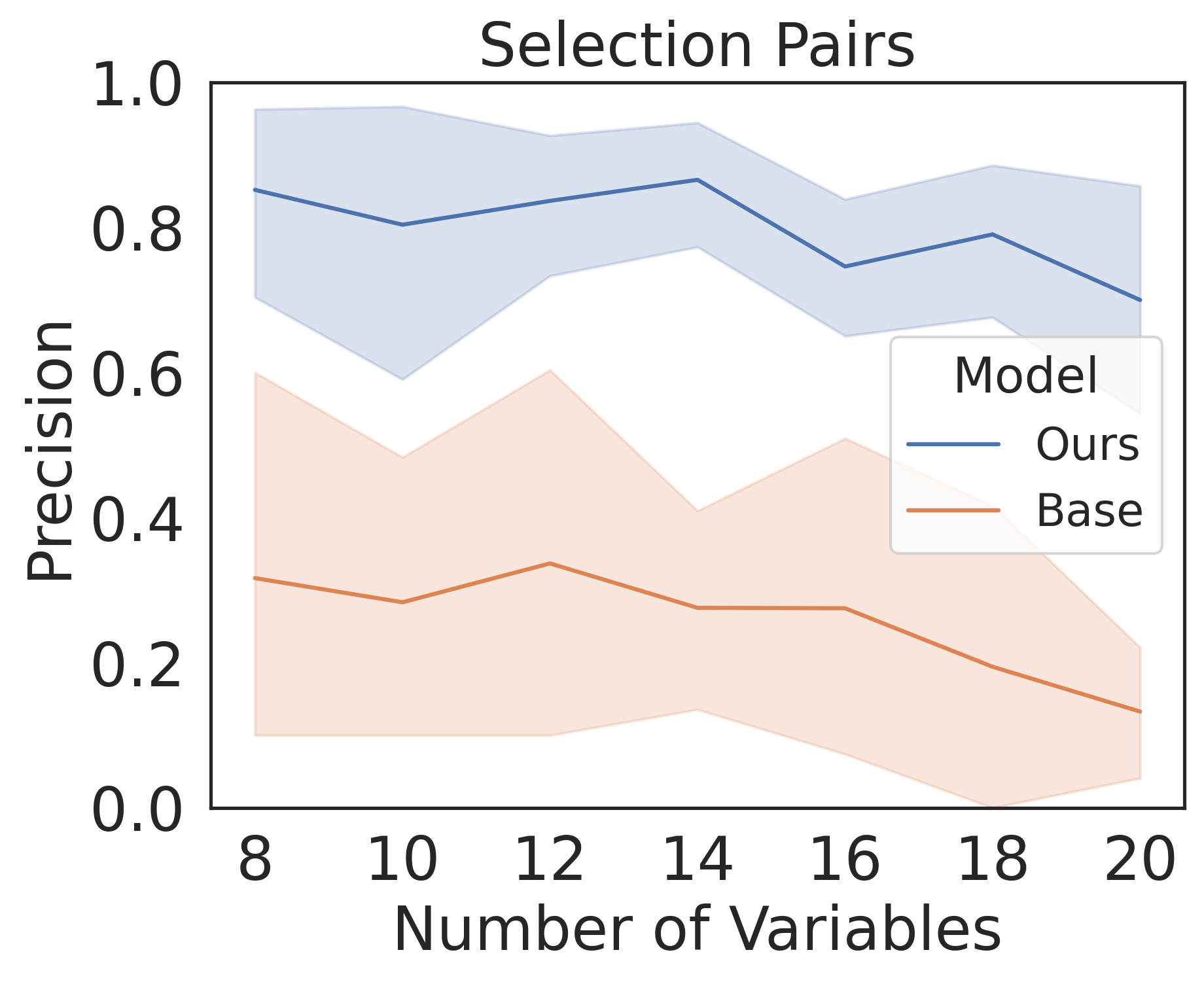}
        \label{fig:1}
    \end{subfigure}%
    \begin{subfigure}{.24\textwidth}
        \centering
        \includegraphics[height=8em]{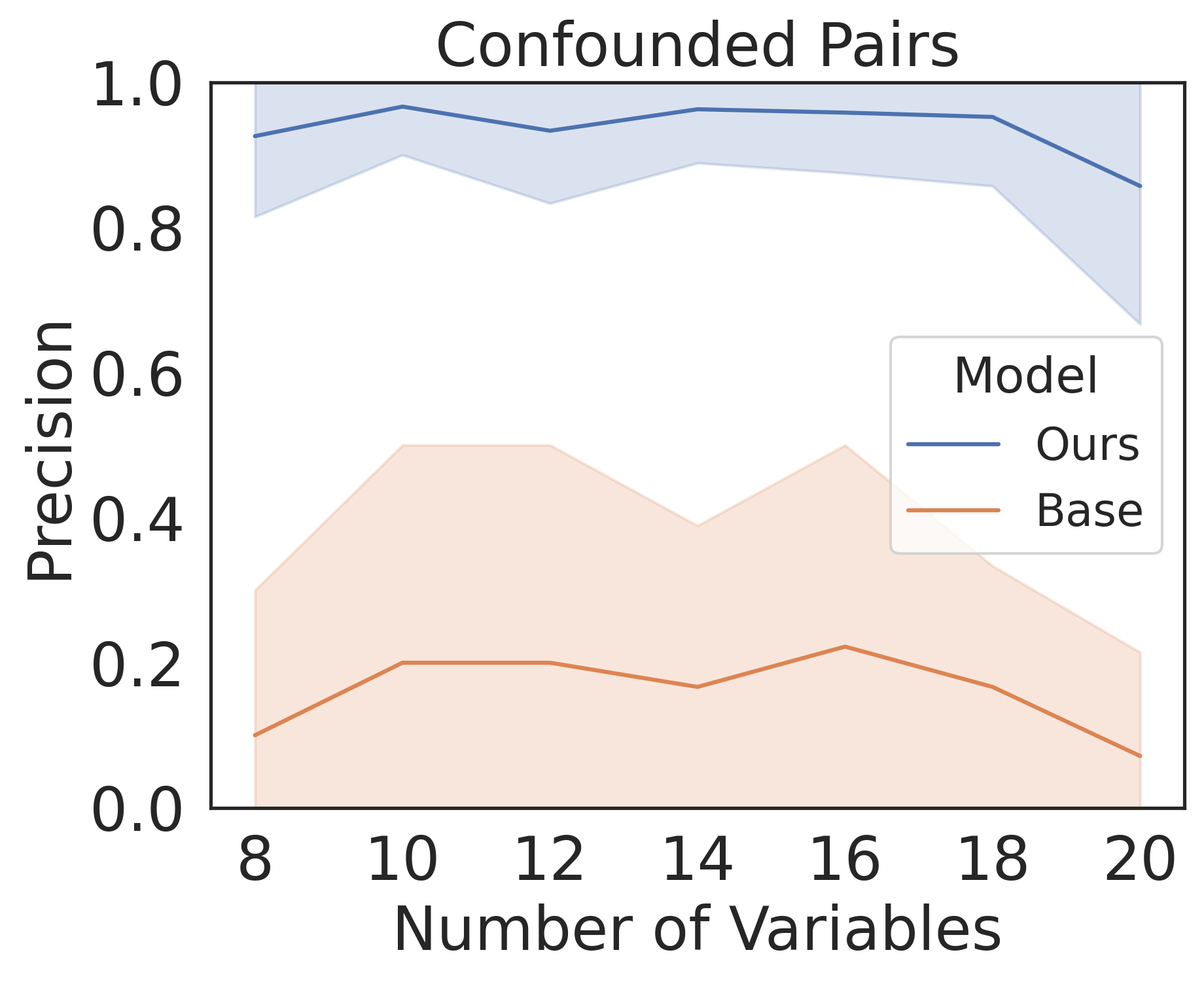}
    \end{subfigure}%
    \begin{subfigure}{.24\textwidth}
        \centering
        \includegraphics[height=8em]{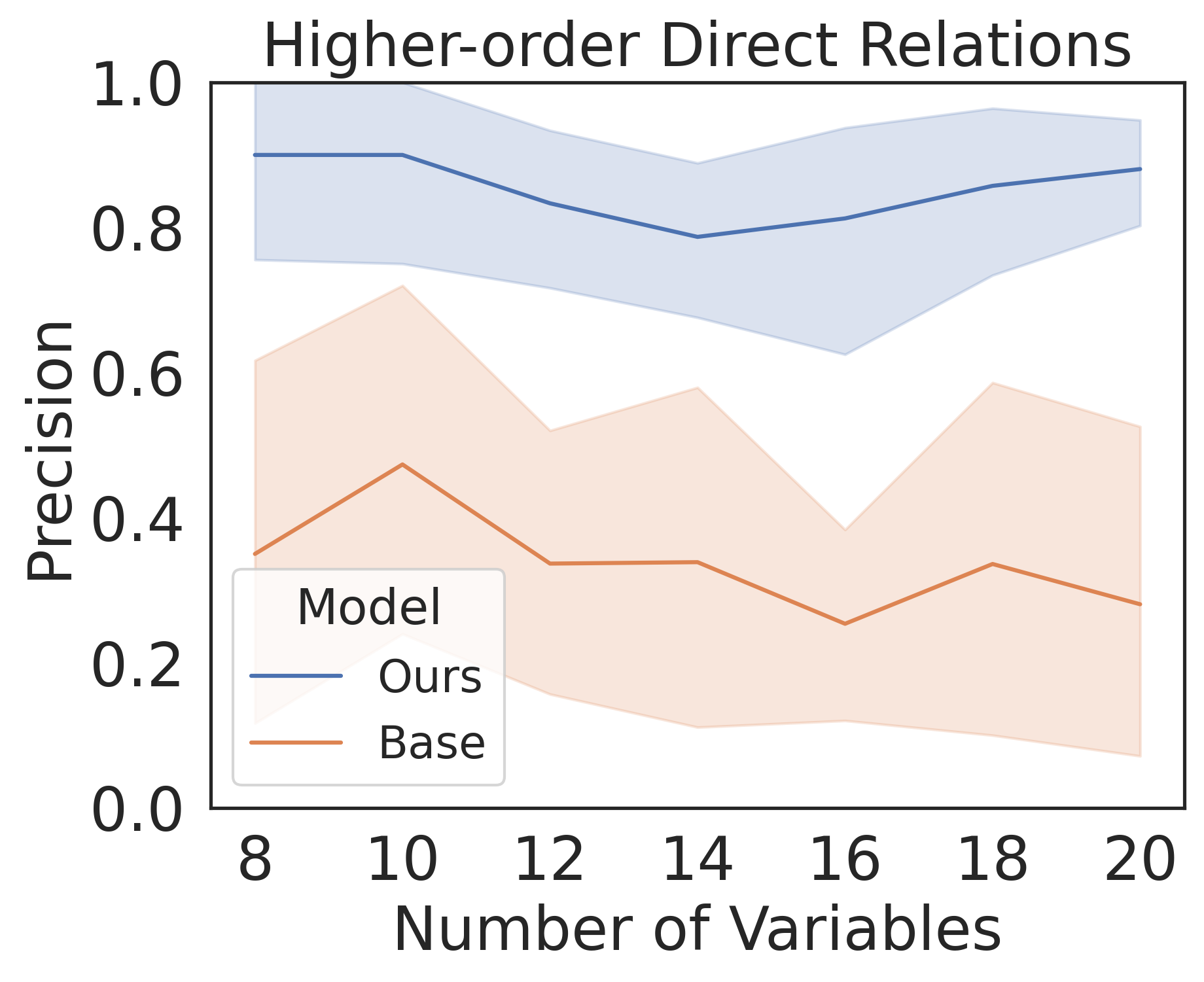}
    \end{subfigure}%
    \begin{subfigure}{.24\textwidth}
        \centering
        \includegraphics[height=8em]{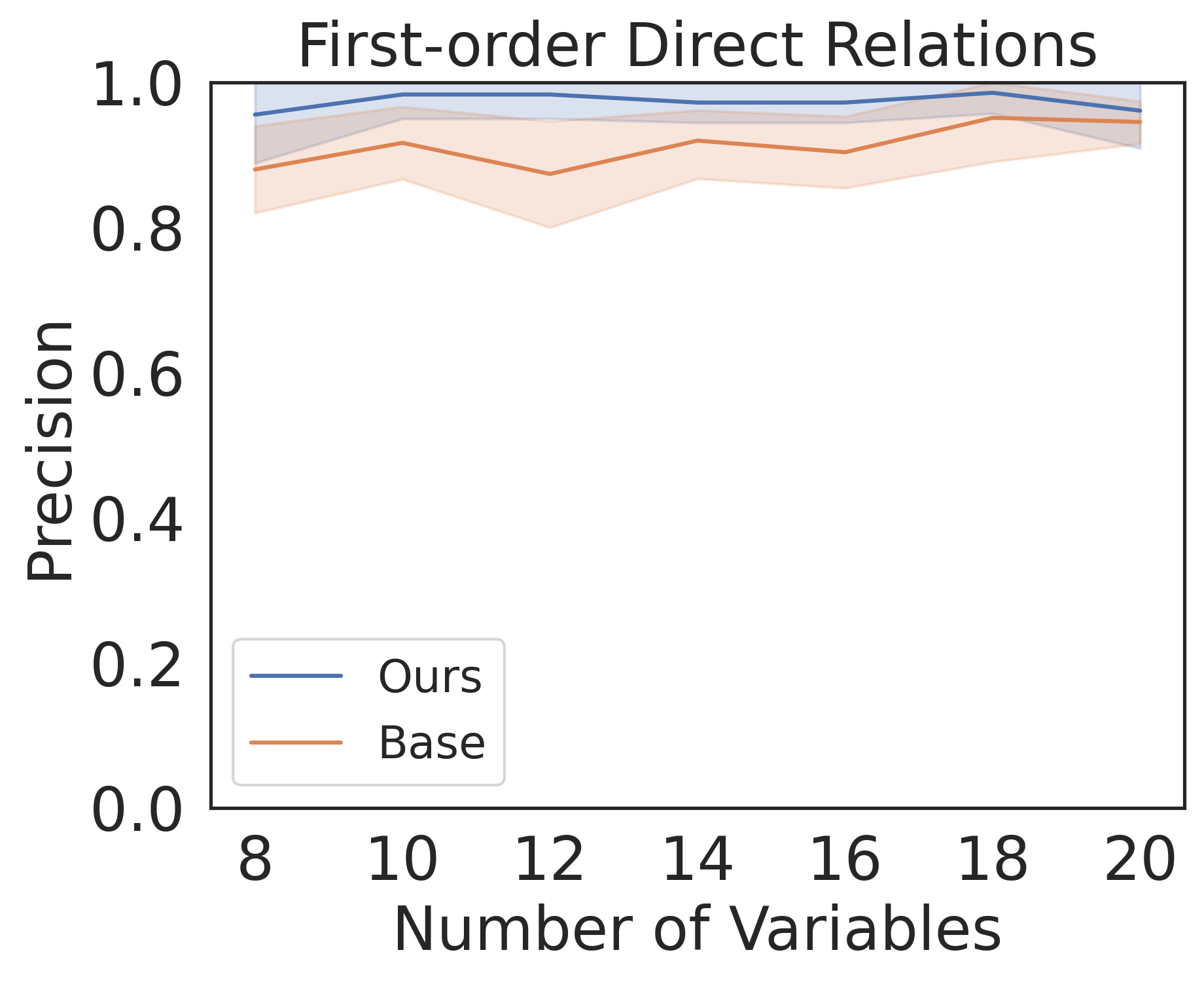}
    \end{subfigure}
    \vspace{0.5em}
    \begin{subfigure}{.24\textwidth}
        \centering
        \includegraphics[height=8em]{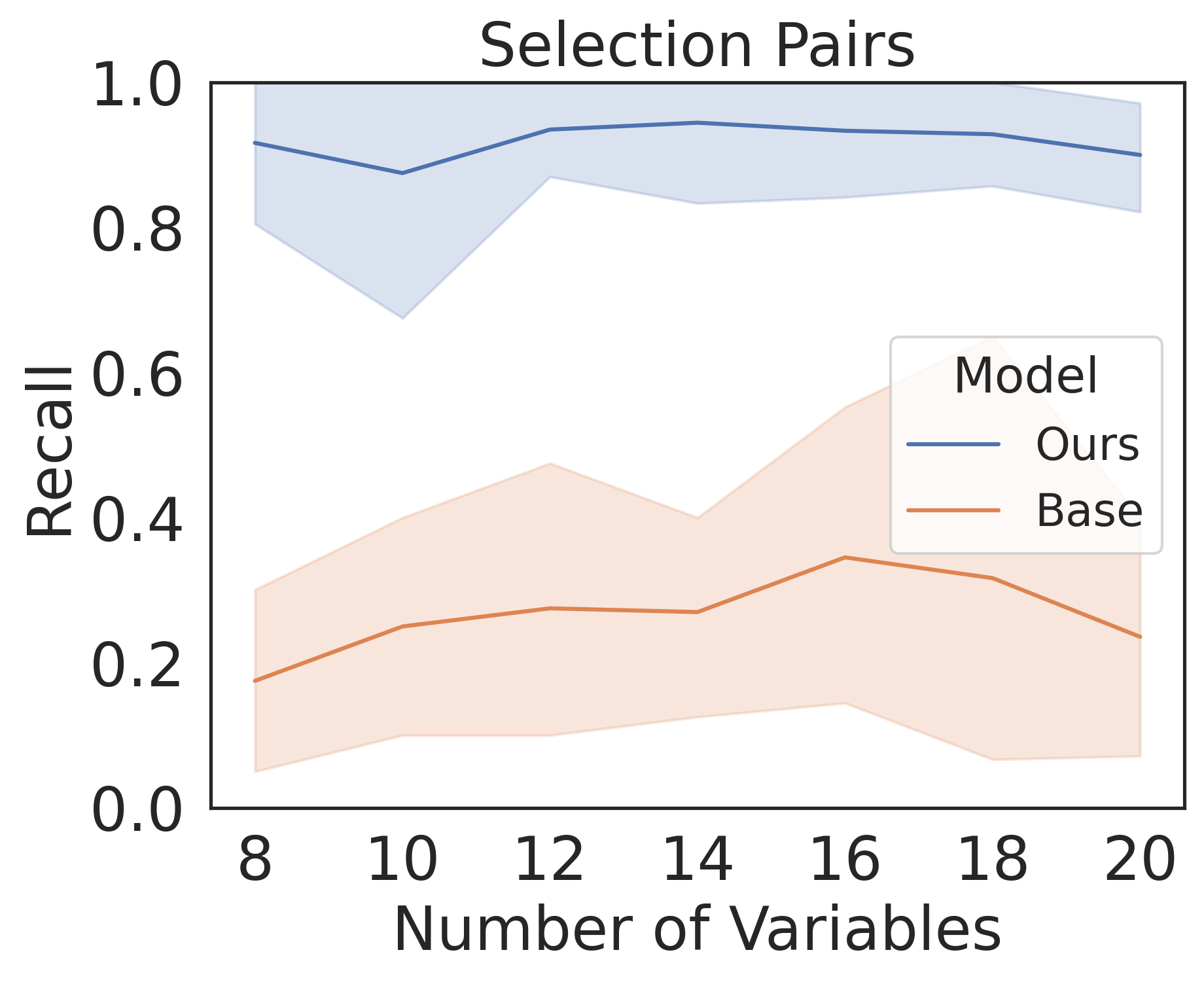}
    \end{subfigure}%
        \begin{subfigure}{.24\textwidth}
        \centering
        \includegraphics[height=8em]{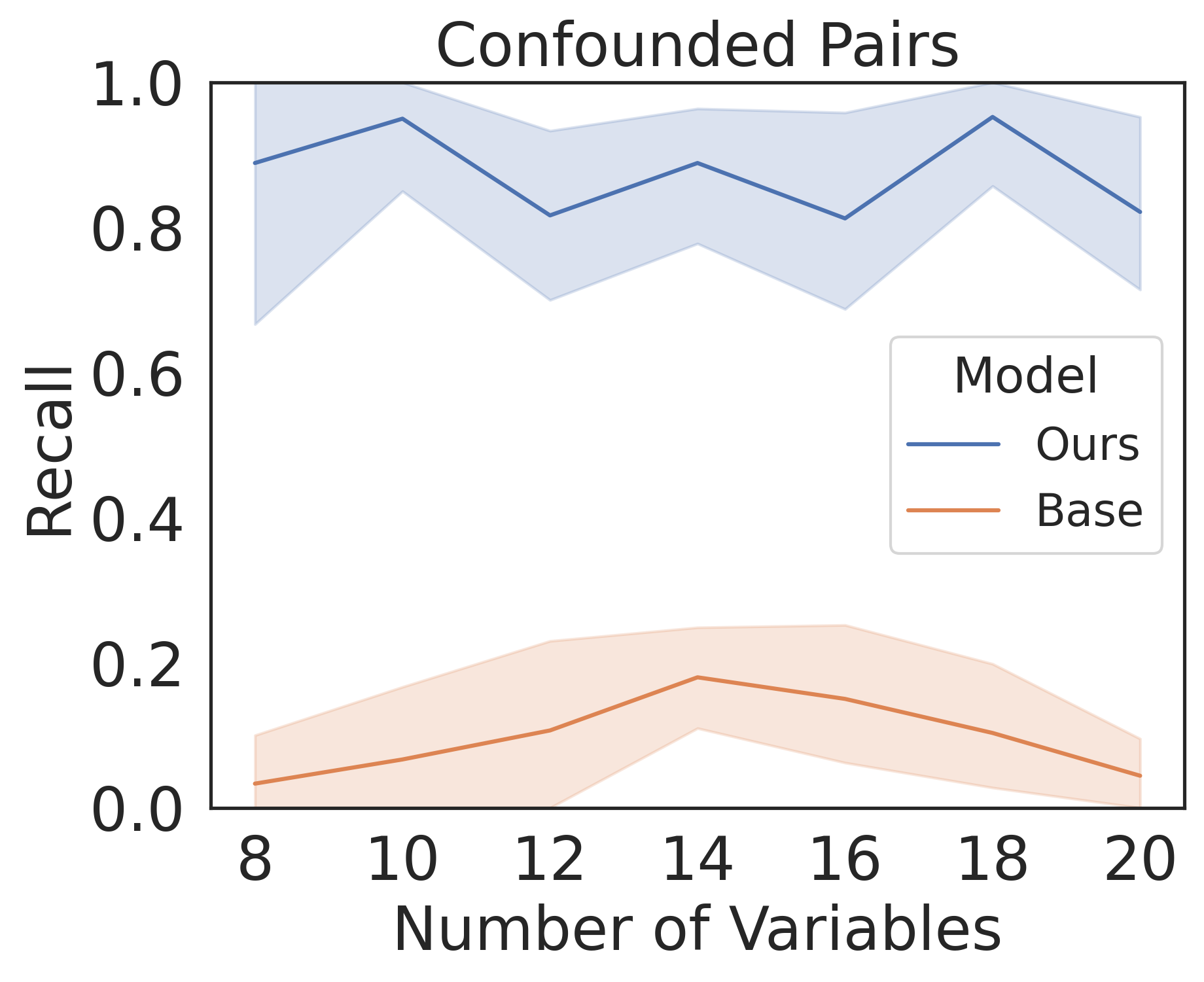}
    \end{subfigure}%
        \begin{subfigure}{.24\textwidth}
        \centering
        \includegraphics[height=8em]{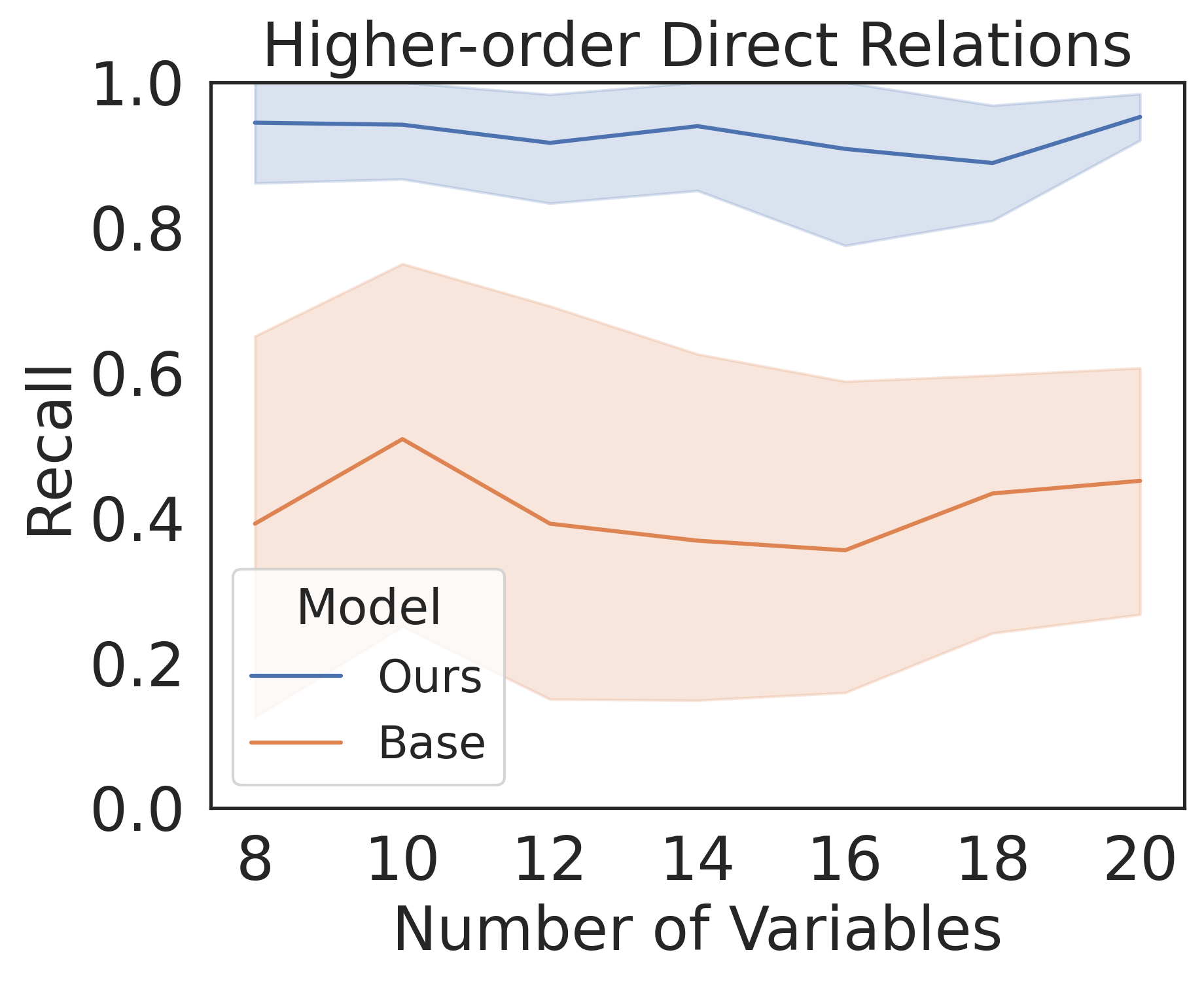}
    \end{subfigure}%
        \begin{subfigure}{.24\textwidth}
        \centering
        \includegraphics[height=8em]{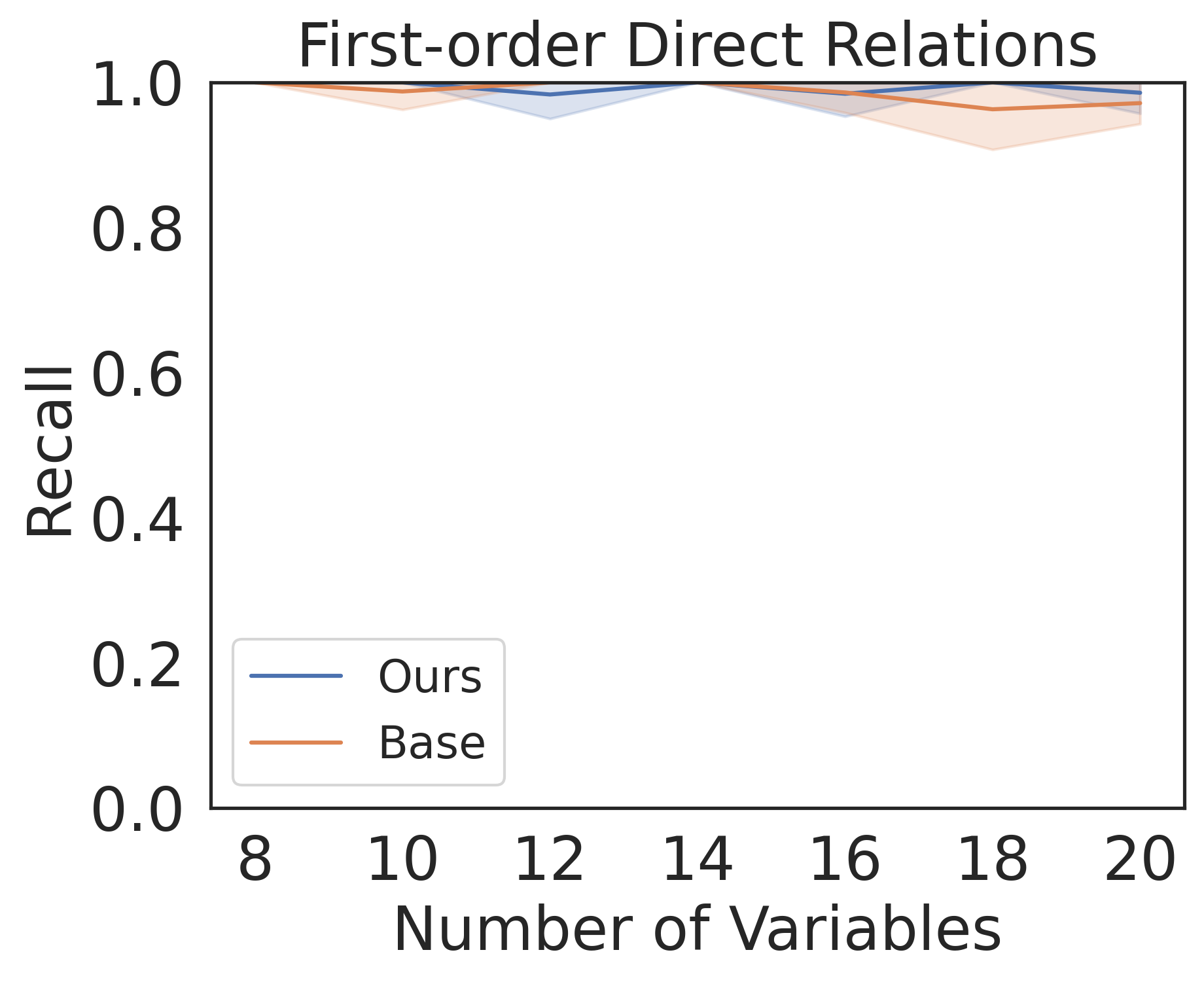}
    \end{subfigure}%

    \vspace{-0.3em}
    \caption{Comparison results on synthetic datasets.}
    \label{fig:comparison_results}
\end{figure*}

\vspace{-0.3em}
\section{Experiments}
In this section, we empirically validate our proposed theory and algorithm on both synthetic and real-world datasets.

\vspace{-0.3em}
\subsection{Synthetic Data}

\textbf{Setup.} \ \
The synthetic datasets are generated according to the selection process on the data generated by linear Gaussian Structural Causal Models (SCMs). We incorporate selection variables as part of the SCMs for preferential sampling, specifically, selecting samples from the unbiased population conditional on whether the values of selection variables exceed their mean. We sample $10,000$ data points from a sufficiently large unbiased population for simulations. For the ground-truth structure, we first randomly generate a basic chain structure, and then randomly drop half of the edges, after which the remaining edges are the first-order direct relations. Then, we randomly add selection pairs, higher-order direct relations, and/or confounded pairs. The number of each type of dependency is randomly sampled from $[0, N/2]$, where $N$ is the number of observed variables. None of the selection variables and latent confounders are included in the observed data. The implementation of the algorithm directly follows the full details elaborated in Appendix \ref{sec:complete_algorithm} with Fisher's z-test \citep{fisher1921014}.

\textbf{Identification.} \ \
To validate our theoretical claims and demonstrate the necessity of the proposed conditions, we conduct a comparative analysis. This involves contrasting our model, in which all assumptions are met, with a base model that does not satisfy structural conditions. We vary the number of observed variables from $\{8,10,\cdots,20\}$, and randomly repeat every experiment $10$ times. From Figure \ref{fig:comparison_results}, it is clear that our algorithm can identify selection pairs as well as other types of dependencies under the proposed conditions, which corroborates our identifiability theory. At the same time, both the base model and ours successfully identify almost all first-order direct relations. This is reasonable since the identification of these first-order direct relations is less complex. However, the base model's precision in identifying first-order direct relations is lower than its recall, particularly when compared to our model. This discrepancy arises because, in the absence of appropriate structural conditions, the algorithm may erroneously classify some confounded pairs as first-order direct relations, leading to an increase in false positives. This observation further underscores the necessity of the proposed conditions.

\begin{figure}[t]
    \centering
    \vspace{0.3em}
    \includegraphics[width=0.8\linewidth]{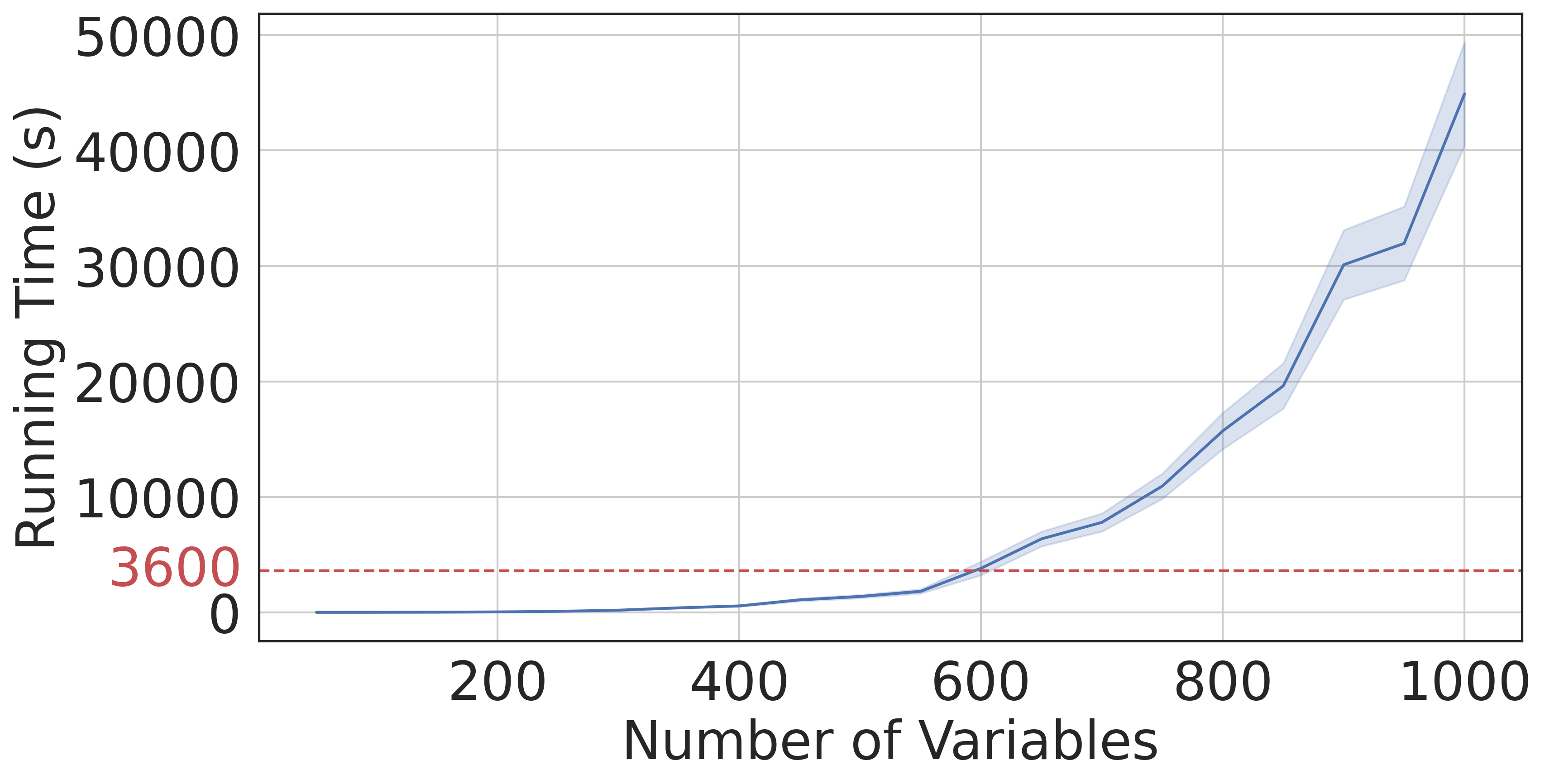}
    \vspace{-0.7em}
    \caption{Running time for large graphs. The algorithm tackles a graph with roughly $600$ variables within one hour (red dotted line).}
    \label{fig:scale}
    \vspace{-1.5em}
\end{figure}

\textbf{Scalability.} \ \
\looseness=-1
We also evaluate the scalability of our algorithm with $\{50, 100, \ldots, 1000\}$ observed variables. Since the number of selection pairs does not influence the number of CI tests needed for the algorithm, we limit it to up to $10$ for sufficient samples after selection. All experiments are from $3$ random trials with only CPUs and $12$ GB of memory. From Figure \ref{fig:scale}, we observe that our algorithm can easily scale to hundreds of observed variables, handling about $600$ variables in an hour (red dotted line), which underscores the scalability of our approach with low complexity of $O(N^2)$.

\subsection{Real-World Data}

\begin{figure*}[t]
    \centering
    \captionsetup{format=hang}
    \begin{subfigure}{.33\textwidth}
        \centering
        \includegraphics[height=6em]{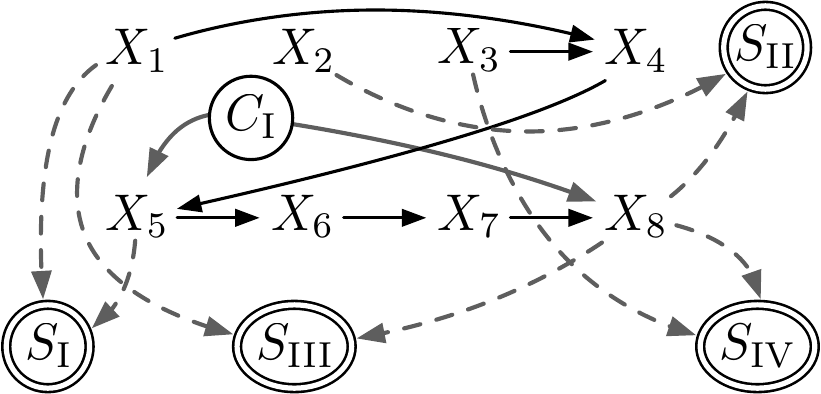}
        \caption{\emph{Hotel California} guitar solo motif}
        \label{fig:hc_guitar_1}
    \end{subfigure}%
    \begin{subfigure}{.33\textwidth}
        \centering
        \includegraphics[height=6em]{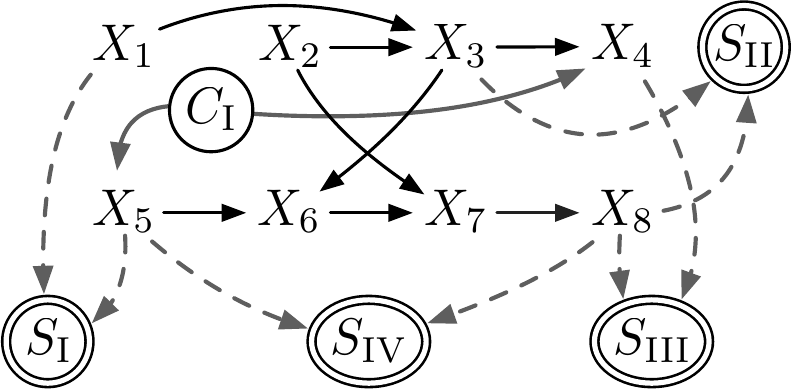}
        \caption{\emph{Hotel California} guitar interplay}
        \label{fig:hc_guitar_2}
    \end{subfigure}%
    \begin{subfigure}{.34\textwidth}
        \centering
        \includegraphics[height=6em]{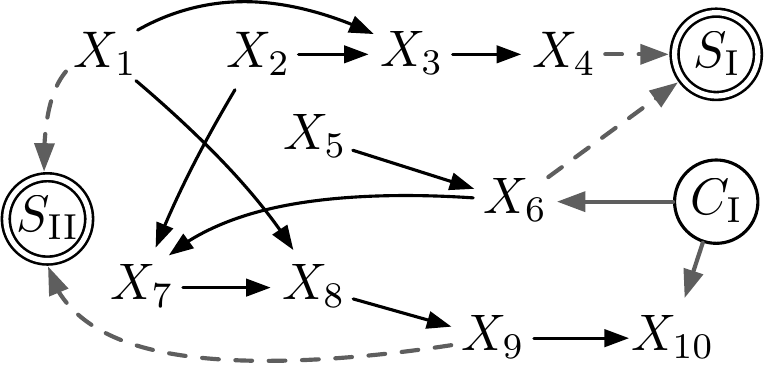}
        \caption{\emph{Panis Angelicus} string section}
        \label{fig:pa_string}
    \end{subfigure}
    \caption{
        Experimental results on real-world music data.
    }
    \label{fig:music_results}
    \vspace{-0.6em}
\end{figure*}


\textbf{Setup.} \ \
\looseness=-1
In this subsection, we present experimental results on real-world music data.
Music can be represented in different formats, and one can consider three main classes of representations: sheet music, symbolic representation, and audio \citep{muller2015fundamentals}.
In our experiments, we consider guitar soundtrack clips from \emph{Hotel California} performed by Eagles, and string section soundtrack clips from \emph{Panis Angelicus}, a verse from the hymn \emph{Sacris solemniis} and set to music by C\'esar Franck.

\textbf{Data Processing.} \ \
\looseness=-1
We use a single node to represent each music segment of a length of one bar on sheet music for \emph{Hotel California} guitar soundtracks (approximately $3.25$ seconds per bar), and a length of two bars for \emph{Panis Angelicus} (approximately $3$ seconds per bar).

The music feature analysis literature has established significant connections between the characteristics in the frequency domain and various musical features, including timbre, pitch and harmony, emotion and mood \citep{kim2010music,yang2011music,muller2015fundamentals}.
Therefore, in order to obtain different samples from the same music piece, we first perform Fast Fourier Transform (FFT) on the audio waveform data (e.g., a \texttt{.wav} or \texttt{.mp3} file).
Then, we apply lower and upper cutoff frequencies, 20Hz and 20kHz, respectively, followed by the Mel scale conversion to better align with the human hearing range \citep{o1999speech}.
We sample one-dimensional features in the frequency domain for the music piece corresponding to each node.
We demonstrate the causal relations among nodes from different music clips in Figure \ref{fig:music_results}.

\textbf{Results.} \ \
Figure \ref{fig:music_results}(a) presents the result corresponding to the soundtrack clip of one repetition of the \emph{Hotel California} guitar solo motif, which is a short pattern of chord movements that are characteristic of the song.
Figure \ref{fig:music_results}(b) presents the result corresponding to the soundtrack clip of another guitar in the interplay.
This clip is synchronized with the guitar solo motif presented in Figure \ref{fig:music_results}(a).
As we can see from the comparison between Figures \ref{fig:music_results}(a) and \ref{fig:music_results}(b), although the two causal graphs are not identical, one can observe similar selection and causal dependence patterns.
Particularly, the selection based on the 1st and 5th bar of the iconic \emph{Hotel California} harmonic motif, denoted by $X_1$ and $X_5$ forming a selected pair $X_1 \rightarrow S_{\mathrm{I}} \leftarrow X_5$ in both Figures \ref{fig:music_results}(a) and \ref{fig:music_results}(b), closely aligns with the overall progression of the Andalusian cadence (a popular musical progression often associated with Flamenco music).
Such alignment is also observed via $X_1 \rightarrow S_{\mathrm{III}} \leftarrow X_8$ in Figure \ref{fig:music_results}(a), and via $X_5 \rightarrow S_{\mathrm{IV}} \leftarrow X_8$ in Figure \ref{fig:music_results}(b).

\looseness=-1
Figure \ref{fig:music_results}(c) presents the result corresponding to the soundtrack clip of the string section in the hymn \emph{Panis Angelicus}.
We do not directly utilize the strophe (in Latin) of the hymn in our experiments.
Here, we reference it to facilitate a more intuitive interpretation and validation of our results.
We observe that the identified selection pattern based on the soundtrack clip exhibits alignment with the strophe to a certain extent.
For instance, for the selected pair $X_4 \rightarrow S_{\mathrm{I}} \leftarrow X_6$, the corresponding lines are ``figuris terminum'' (for $X_4$) and ``Manducat Dominum'' (for $X_6$), holding significant meanings within the context of the hymn's religious and sacramental themes.




\section{Conclusion}
\looseness=-1
We argue that the phenomenon of sample selection may be involved in many types of data. Overlooking the selection process or confusing it either with latent confounding effects or direct causal relations may lead to an incomplete or overcomplicated inductive bias, especially in sequential data. 
Motivated by this, we establish a set of theoretical results demonstrating the identifiability of selection structure in sequential data. Instead of adopting the traditional view that mainly treats selection as a bias to be mitigated, we consider it as an informative mechanism underlying the data, seeking to delve deeper into the causal structure of selection. Specifically, we first prove that the selection structure in sequential data is identifiable without any parametric assumptions or interventional experiments. Moreover, even in the presence of latent confounders, we show that the nonparametric identifiability still holds true under appropriate structural conditions. At the same time, we propose a provably correct algorithm for the identification of selection structures as well as other types of dependencies. The validity of the theoretical claims and the efficacy of the algorithm have been rigorously evaluated under various settings, utilizing both synthetic and real-world data. 

\looseness=-1
Since the structure of selection may provide an often-overlooked insight into the world underlying the data, the discovery results with an identifiability guarantee could act as a new and more veridical inductive bias for machine learning tasks, suggesting exciting directions on both efficiently utilizing the data and reliably uncovering the underlying truth. Currently, the lack of exploration beyond understanding is a limitation; accordingly, future work could be leveraging the uncovered structure to solve open problems such as building a trustworthy foundation model.

\section*{Impact Statement}
This paper presents work whose goal is to advance the field of Machine Learning. There are many potential societal consequences of our work, none of which we feel must be specifically highlighted here.

\section*{Acknowledgement}
We thank Peter Spirtes and anonymous reviewers for insightful comments and discussion. We would also like to acknowledge the support from NSF Grant 2229881, the National Institutes of Health (NIH) under Contract R01HL159805, and grants from Apple Inc., KDDI Research Inc., Quris AI, and Florin Court Capital.
\nocite{langley00}

\bibliography{bibliography}

\begin{thebibliography}{22}
\providecommand{\natexlab}[1]{#1}
\providecommand{\url}[1]{\texttt{#1}}
\expandafter\ifx\csname urlstyle\endcsname\relax
  \providecommand{\doi}[1]{doi: #1}\else
  \providecommand{\doi}{doi: \begingroup \urlstyle{rm}\Url}\fi

\bibitem[Bareinboim et~al.(2014)Bareinboim, Tian, and Pearl]{bareinboim2014recovering}
Bareinboim, E., Tian, J., and Pearl, J.
\newblock Recovering from selection bias in causal and statistical inference.
\newblock In \emph{Proceedings of the AAAI Conference on Artificial Intelligence}, volume~28, 2014.

\bibitem[Chen et~al.(2024)Chen, Zoeter, and Mooij]{chen2024modeling}
Chen, L., Zoeter, O., and Mooij, J.~M.
\newblock Modeling latent selection with structural causal models.
\newblock \emph{arXiv preprint arXiv:2401.06925}, 2024.

\bibitem[Correa et~al.(2019)Correa, Tian, and Bareinboim]{CorreaTianBareinboim2019}
Correa, J.~D., Tian, J., and Bareinboim, E.
\newblock Identification of causal effects in the presence of selection bias.
\newblock In \emph{Proceedings of the AAAI Conference on Artificial Intelligence}, volume~33, pp.\  2744--2751, 2019.

\bibitem[Fisher et~al.(1921)]{fisher1921014}
Fisher, R.~A. et~al.
\newblock 014: On the" probable error" of a coefficient of correlation deduced from a small sample.
\newblock 1921.

\bibitem[Forr{\'e} \& Mooij(2020)Forr{\'e} and Mooij]{forre2020causal}
Forr{\'e}, P. and Mooij, J.~M.
\newblock Causal calculus in the presence of cycles, latent confounders and selection bias.
\newblock In \emph{Uncertainty in Artificial Intelligence}, pp.\  71--80. PMLR, 2020.

\bibitem[Heckman(1979)]{heckman1979sample}
Heckman, J.~J.
\newblock Sample selection bias as a specification error.
\newblock \emph{Econometrica: Journal of the econometric society}, pp.\  153--161, 1979.

\bibitem[Hern{\'a}n et~al.(2004)Hern{\'a}n, Hern{\'a}ndez-D{\'\i}az, and Robins]{hernan2004structural}
Hern{\'a}n, M.~A., Hern{\'a}ndez-D{\'\i}az, S., and Robins, J.~M.
\newblock A structural approach to selection bias.
\newblock \emph{Epidemiology}, pp.\  615--625, 2004.

\bibitem[Kaltenpoth \& Vreeken(2023)Kaltenpoth and Vreeken]{kaltenpoth2023identify}
Kaltenpoth, D. and Vreeken, J.
\newblock Identifying selection bias from observational data.
\newblock In \emph{Proceedings of the AAAI Conference on Artificial Intelligence}, 2023.

\bibitem[Kim et~al.(2010)Kim, Schmidt, Migneco, Morton, Richardson, Scott, Speck, and Turnbull]{kim2010music}
Kim, Y.~E., Schmidt, E.~M., Migneco, R., Morton, B.~G., Richardson, P., Scott, J., Speck, J.~A., and Turnbull, D.
\newblock Music emotion recognition: A state of the art review.
\newblock In \emph{Proc. ismir}, volume~86, pp.\  937--952, 2010.

\bibitem[M{\"u}ller(2015)]{muller2015fundamentals}
M{\"u}ller, M.
\newblock \emph{Fundamentals of music processing: Audio, analysis, algorithms, applications}, volume~5.
\newblock Springer, 2015.

\bibitem[O'Shaughnessy \& Douglas(1999)O'Shaughnessy and Douglas]{o1999speech}
O'Shaughnessy and Douglas.
\newblock \emph{Speech Communications: Human and Machine}.
\newblock 1999.

\bibitem[Pearl et~al.(2000)]{pearl2000models}
Pearl, J. et~al.
\newblock Models, reasoning and inference.
\newblock \emph{Cambridge, UK: CambridgeUniversityPress}, 19\penalty0 (2):\penalty0 3, 2000.

\bibitem[Radford et~al.(2018)Radford, Narasimhan, Salimans, Sutskever, et~al.]{radford2018improving}
Radford, A., Narasimhan, K., Salimans, T., Sutskever, I., et~al.
\newblock Improving language understanding by generative pre-training.
\newblock 2018.

\bibitem[Schoenberg et~al.(1967)Schoenberg, Strang, and Stein]{schoenberg1967musiccomp}
Schoenberg, A., Strang, G., and Stein, L.
\newblock \emph{Fundamentals of musical composition}.
\newblock Faber and Faber, 1967.
\newblock URL \url{https://cir.nii.ac.jp/crid/1130282269094586112}.

\bibitem[Spirtes et~al.(1995)Spirtes, Meek, and Richardson]{spirtes1995causal}
Spirtes, P., Meek, C., and Richardson, T.
\newblock Causal inference in the presence of latent variables and selection bias.
\newblock In \emph{Proceedings of the Eleventh conference on Uncertainty in artificial intelligence}, pp.\  499--506, 1995.

\bibitem[Spirtes et~al.(2000)Spirtes, Glymour, and Scheines]{spirtes2000causation}
Spirtes, P., Glymour, C.~N., and Scheines, R.
\newblock \emph{Causation, prediction, and search}.
\newblock 2000.

\bibitem[Versteeg et~al.(2022)Versteeg, Mooij, and Zhang]{versteeg2022local}
Versteeg, P., Mooij, J., and Zhang, C.
\newblock Local constraint-based causal discovery under selection bias.
\newblock In \emph{Conference on Causal Learning and Reasoning}, pp.\  840--860. PMLR, 2022.

\bibitem[Yang \& Chen(2011)Yang and Chen]{yang2011music}
Yang, Y.-H. and Chen, H.~H.
\newblock \emph{Music emotion recognition}.
\newblock CRC Press, 2011.

\bibitem[Zhang(2008)]{zhang2008completeness}
Zhang, J.
\newblock On the completeness of orientation rules for causal discovery in the presence of latent confounders and selection bias.
\newblock \emph{Artificial Intelligence}, 172\penalty0 (16-17):\penalty0 1873--1896, 2008.

\bibitem[Zhang(2013)]{zhang2013comparison}
Zhang, J.
\newblock A comparison of three {Occam}'s razors for markovian causal models.
\newblock \emph{The British journal for the philosophy of science}, 2013.

\bibitem[Zhang et~al.(2016)Zhang, Zhang, Huang, Sch{\"o}lkopf, and Glymour]{zhang2016identifiability}
Zhang, K., Zhang, J., Huang, B., Sch{\"o}lkopf, B., and Glymour, C.
\newblock On the identifiability and estimation of functional causal models in the presence of outcome-dependent selection.
\newblock In \emph{UAI}, 2016.

\bibitem[Zheng et~al.(2024)Zheng, Huang, Chen, Ramsey, Gong, Cai, Shimizu, Spirtes, and Zhang]{zheng2024causal}
Zheng, Y., Huang, B., Chen, W., Ramsey, J., Gong, M., Cai, R., Shimizu, S., Spirtes, P., and Zhang, K.
\newblock Causal-learn: Causal discovery in python.
\newblock \emph{Journal of Machine Learning Research}, 25\penalty0 (60):\penalty0 1--8, 2024.

\end{thebibliography}
\bibliographystyle{icml2024}

\newpage
\appendix
\onecolumn


\section{Complete Algorithm}
\label{sec:complete_algorithm}

\begin{algorithm}[H]
	\caption{Identification of the Selection Structure}
	\label{alg:global}
	\Input{Data $\mathbf{X}$}
	\Output{A causal graph $\mathcal{G}$}
    \vspace{0.3em}
    \begin{mdframed}[innerleftmargin=20pt,innertopmargin=3pt,linewidth=0.1pt,pstrickssetting={linestyle=dashed,middlelinewidth=10pt}]
    \vspace{-0.3em}
    \Comment{\textbf{Initialization}}
    Set $N$ as the number of observed variables in $\X$\; \label{l1}
    Initialize two empty lists as $\mathcal{L} \coloneqq [\cdot]$ and $\mathcal{R} \coloneqq [\cdot]$, an empty edge set $\mathcal{E}$, and a set of vertices $\mathcal{V} \coloneqq \{1,2,\cdots,N\}$\;
    \end{mdframed}

    \begin{mdframed}[innerleftmargin=20pt,skipabove=-12pt,linewidth=0.1pt]
    \vspace{-0.3em}
    \Comment{\textbf{Stage One}}
    \ForEach{
        $l$ in $\{N-1, N-2, \cdots, 1\}$
        \label{l3}
    }{
        Set $i \coloneqq 1$, $j \coloneqq i+l$\;
        \While{$j \leq N$}{
            \uIf{$X_i \notindependent X_j | \mathbf{X}_{\text{Pre}(j) \setminus i}$\label{ls1_test1}}{
                \uIf{$\forall k \in \mathcal{R}[i], X_i \notindependent X_j | \{\mathbf{X}_{\text{Pre}(j) \setminus i}, X_{j+1}, X_k\}$ $(X_{N+1} = \emptyset)$\label{ls1_test2}} 
                {
                    $\mathcal{L}$.append$(i)$\;
                    $\mathcal{R}[i]$.append$(j)$\;
                }
            }
            Set $i \coloneqq i+1$, $j \coloneqq i+l$\; \label{l10}
        }
    }
    \end{mdframed}

    \begin{mdframed}[innerleftmargin=20pt,skipabove=-12pt,linewidth=0.1pt]
    \vspace{-0.3em}
    \Comment{\textbf{Stage Two}}
    Set $u \coloneqq 0$\; \label{l11}
    \ForEach{$i \in ReverseSorted(\mathcal{L})$}{ \label{l12}
        \ForEach{$j \in ReverseSorted(\mathcal{R}[i])$}{ \label{l13}
            Set $k = j-1$\;
            \While{$k \in \mathcal{R}[i]$ and $k > i+1$}{ \label{l15}
                $k \coloneqq k-1$\;
            } \label{l16}
            Let $\mathbf{T} = \{X_v$, where $(X_p,X_v)_s \in \mathcal{E}$\}\; \label{l17}
            \eIf{$X_i \notindependent X_k | \{\mathbf{X}_{\text{Pre}(k) \setminus i}, X_j, \mathbf{T}\}$ or $k=i$ \label{ls2_test1}}{
                \eIf{$X_i \indep X_k | \{\mathbf{X}_{\text{Pre}(k) \setminus i}, X_{j+1}, \mathbf{T}\}$ $(X_{N+1} = \emptyset)$ or $k=i$ \label{ls2_test2}}{
                    Add to $\mathcal{E}$ an edge $i \rightarrow j$\;
                }{ \label{l19}
                    Add to $\mathcal{E}$ an edge $i \rightarrow j$\;
                    Add to $\mathcal{E}$ edges $i \rightarrow s_u$ and $j \rightarrow s_u$\;
                    Set $u \coloneqq u+1$\; \label{l22}
                    \uIf{$j-1 \in \mathcal{R}[i]$}{
                        $\mathcal{R}[i]$.remove$(j-1)$\; \label{l24}
                    }
                }
            }{
                Add to $\mathcal{E}$ edges $i \rightarrow s_u$ and $j \rightarrow s_u$\;
                Set $u \coloneqq u+1$\; \label{l29}
            }
        }
    }
    \end{mdframed}
    
    \begin{mdframed}[innerleftmargin=20pt,skipabove=-12pt,linewidth=0.1pt]
    \vspace{-0.3em}
    \Comment{\textbf{Stage Three} (for latent confounders)}
    Set $u \coloneqq 0$\;
    \ForEach{$j$ in $\{N, N-1 \cdots, 3\}$}{ \label{l31}
        \ForEach{$i$ in $\{N-1, N-2 \cdots, 2\}$}{ \label{l32}
            Let $\mathbf{T} = \{X_v$, where $(X_{i-1},X_v)_\mathrm{s}$, $(X_{i},X_v)_\mathrm{s}$, or $(X_{i+1},X_v)_\mathrm{s}$, $v>j$\}\label{l33}\;
            \uIf{$i \rightarrow j$ and $i-1 \rightarrow j$ \label{ls3_if}}{
                \eIf{$X_{i-1} \notindependent X_j | \{\mathbf{X}_{\text{Pre}(j) \setminus \{i-1,i\}}, X_{j+1},\mathbf{T}\}$ $(X_{N+1} = \emptyset)$}{
                    Add to $\mathcal{E}$ edges $i-1 \rightarrow j$ and $i \rightarrow j$\;
                   
                }{
                    Add to $\mathcal{E}$ edges $c_u \rightarrow i$ and $c_u \rightarrow j$\;\label{ls3_38}
                    Remove $i-1 \rightarrow j$ and $i \rightarrow j$ from $\mathcal{E}$\;
                    Set $u \coloneqq u+1$\;\label{ls3_40}
                }
            }
        }
    }
    \end{mdframed}

	\Return{$\mathcal{G} = \{\mathcal{V}, \mathcal{E}\}$.}
\end{algorithm}
\clearpage

\section{Illustration of the Algorithm}
\label{sec:example}

\begin{figure}[t]
    \centering
    \includegraphics[width=0.7\linewidth]{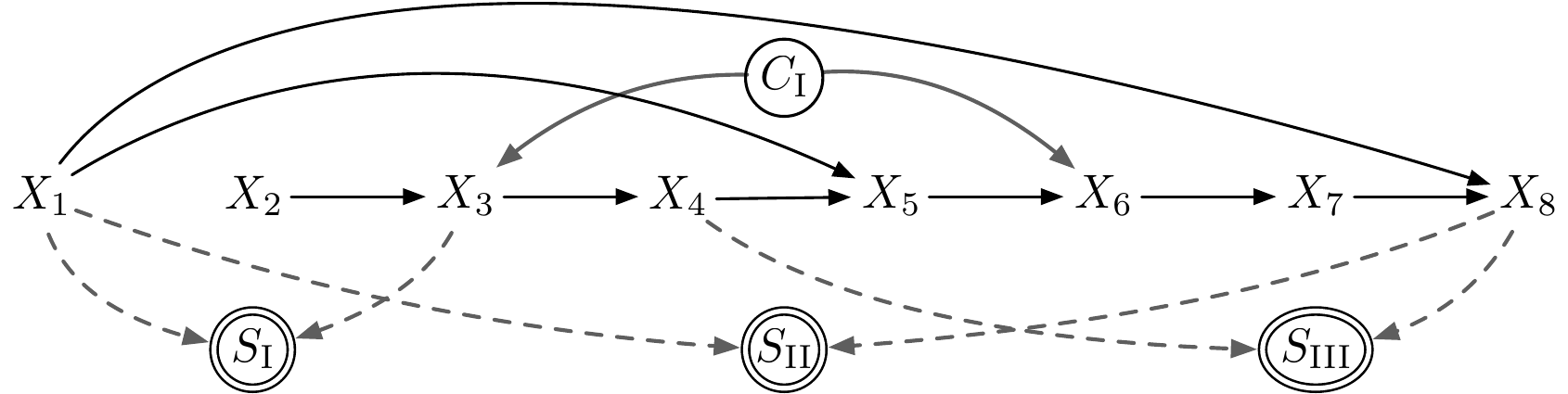}
    \caption{Example for the illustration of the algorithm.}
    \label{fig:alg_example}
\end{figure}

In this section, we illustrate the complete algorithm (Algorithm \ref{alg:global}) with a step-by-step example. Suppose our data $\mathbf{X} = \{X_1, X_2, \cdots, X_8\}$ is generated from the causal structure shown in Figure \ref{fig:alg_example}. Note that the complete proof of the correctness of the algorithm is included in Appendices \ref{sec:proof_thm_1} and \ref{sec:proof_thm_2}. Here we just explain the procedure of the three stages based on the example in Figure \ref{fig:alg_example}:
\begin{itemize}
    \item \textbf{Stage One:} For each pair of variables, the algorithm first tests if $X_i \notindependent X_j | \mathbf{X}_{\text{Pre}(j) \setminus i}$ (line \ref{ls1_test1}). According to the order of tests (lines \ref{l3} and \ref{l10}), we first test $(X_1, X_8)$. For line \ref{ls1_test1}, we get the results $X_1 \notindependent X_8 | \mathbf{X}_{\text{Pre}(8) \setminus 1}$. Since this is the first pair, $\mathcal{R}[1]$ is empty, and we also have $X_{j+1} = X_9 = \emptyset$. Thus, for line \ref{ls1_test2}, we still have  $X_1 \notindependent X_8 | \mathbf{X}_{\text{Pre}(8) \setminus 1}$, and then we update $\mathcal{R}[1]$ as $\mathcal{R}[1] = \{8\}$.

    However, for $i=1$ and $j=7$, there is $X_1 \notindependent X_7 | \mathbf{X}_{\text{Pre}(7) \setminus 1}$ even though $(X_1, X_7)$ is not a dependent pair. This is because there exists a selection pair $(X_1, X_8)_\mathrm{s}$, and if $X_8$ is not conditioned, the path $\{X_1 \rightarrow S_\mathrm{II} \leftarrow X_8 \leftarrow X_7\}$ is d-connected, leading to the conditional dependence $X_1 \notindependent X_7 | \mathbf{X}_{\text{Pre}(7) \setminus 1}$. Moreover, when we move to line \ref{ls1_test2}, we still have $X_1 \indep X_7 | \{\mathbf{X}_{\text{Pre}(7) \setminus 1}, X_8\}$ since $(X_1, X_8)$ is also a direct relation on the path. The conditioning on $X_8$ although blocks the path $\{X_1 \rightarrow S_\mathrm{II} \leftarrow X_8 \leftarrow X_7\}$, but also opens the path $\{X_1 \rightarrow X_8 \leftarrow X_7\}$. Therefore, we have $\mathcal{R}[1] = \{8,7\}$, and include $(X_1, X_7)$ as a spurious dependent pair.

    Then we move to $(X_2, X_8)$, which is similar to $(X_1, X_8)$ and will be included as a dependent pair. However, it gets complicated again when it comes to $(X_1, X_6)$. Similarly, because of $\{X_1 \rightarrow S_\mathrm{II} \leftarrow X_8 \leftarrow X_7 \leftarrow X_6\}$, there is $X_1 \notindependent X_6 | \mathbf{X}_{\text{Pre}(6) \setminus 1}$ even though $(X_1, X_6)$ is not a dependent pair. Nevertheless, $(X_1, X_6)$ will not be included this time. This is because, for $(X_1, X_6)$, the test result in line \ref{ls1_test2} is $X_1 \indep X_6 | \{\mathbf{X}_{\text{Pre}(6) \setminus 1}, X_7, X_8\}$ ($X_{j+1} = 7$ and $8 \in \mathcal{R}[1]$), which is different from the case for $(X_1, X_7)$. The reason is that, even if $X_8$ is both a collider and a non-collider, we have $X_6$ in the conditioning set, which blocks the path  $\{X_1 \rightarrow X_8 \leftarrow X_7 \leftarrow X_6\}$. The path $\{X_1 \rightarrow S_\mathrm{II} \leftarrow X_8 \leftarrow X_7 \leftarrow X_6\}$ is blocked by conditioning on $X_8$.

    The preceding pairs of variables are all similar to the cases discussed above until we meet $(X_2, X_6)$. This pair is special since we have $(X_3, X_6)$ as a confounded pair in the ground-truth structure. Thus, when conditioning on $X_3$, $X_2$ and $X_6$ are always d-connected. Therefore, we first have $X_2 \notindependent X_6 | \mathbf{X}_{\text{Pre}(6) \setminus 1}$. Then, we also have $X_2 \notindependent X_6 | \{\mathbf{X}_{\text{Pre}(6) \setminus 1}, X_7\}$, which indicates that $(X_2, X_6)$ is included as a spurious dependent pair.

    The other pairs are all similar to the cases we have discussed. At the end of Stage One, we have the following results: $\mathcal{R}[1] = \{8,7,5,3\}$, $\mathcal{R}[2] = \{6,3\}$, $\mathcal{R}[3] = \{6,4\}$, $\mathcal{R}[4] = \{8,5\}$, $\mathcal{R}[5] = \{6\}$. $\mathcal{R}[6] = \{7\}$, and $\mathcal{R}[7] = \{8\}$.

    \item \textbf{Stage Two:} Then we move to the second stage with $\mathcal{L}$ and $\mathcal{R}$ from the first stage. We start from a descending order for both $\mathcal{L}$ and $\mathcal{R}$. The first pair is $(X_7, X_8)$, which is identified as a direct relation since $k=i$. This is similar for $(X_6, X_7)$ and $(X_5, X_6)$, and both of them are identified as (first-order) direct relations.

    Now we move to $(X_4, X_8)$, for which we have $k=7$. For line \ref{ls2_test1}, we have $X_4 \indep X_7 | \{\mathbf{X}_{\text{Pre}(7) \setminus 4}, X_8\}$, where $\mathbf{T}=\emptyset$. Thus, $(X_4, X_8)$ is identified as a selection pair. Then we identify $(X_4, X_5)$ as a direct relation similar to the case discussed above. 

    Next we move to $(X_3, X_6)$, for which we have $k=5$ and $\mathbf{T} = \{X_8\}$. Because $X_3 \notindependent X_5 | \{\mathbf{X}_{\text{Pre}(5) \setminus 3}, X_6, X_8\}$ (line \ref{ls2_test1}, $\mathbf{T}=\{X_8\}$) and $X_3 \indep X_5 | \{\mathbf{X}_{\text{Pre}(5) \setminus 3}, X_7\}$ (line \ref{ls2_test2}), we identify $(X_3, X_6)$ as a direct relation. However, since $(X_3, X_6)$ is a confounded pair, we now include a spurious direct relation in $\mathcal{E}$, which will be identified in Stage Three.

    After that, we identify $(X_3, X_4)$ as a first-order direct relation, then we consider $(X_2, X_6)$, for which we have $k=5$ and $\mathbf{T} = \{X_8\}$. Note that, $(X_2, X_6)$ is actually not a dependent pair. This is included in Stage Two only because $(X_3, X_6)$ is a confounded pair, and the collider $X_3$ on the path was always included in the conditioning set in Stage One. For Stage Two, because $X_2 \notindependent X_5 | \{\mathbf{X}_{\text{Pre}(5) \setminus 2}, X_6, X_8\}$ (line \ref{ls2_test1}) and $X_2 \indep X_5 | \{\mathbf{X}_{\text{Pre}(5) \setminus 2}, X_7, X_8\}$ (line \ref{ls2_test2}), $(X_2, X_6)$ is still identified as a direct relation as a spurious one again.

    Then after identifying $(X_2, X_3)$ as a first-order direct relation, we move to $(X_1, X_8)$. For this pair, we have $K=6$ since $7 \in \mathcal{R}[1]$. Because $X_1 \notindependent X_6 | \{\mathbf{X}_{\text{Pre}(6) \setminus 1}, X_8\}$ (line \ref{ls2_test1}) and $X_1 \notindependent X_6 | \{\mathbf{X}_{\text{Pre}(6) \setminus 1}, X_8\}$ (line \ref{ls2_test2}), we identify $(X_1, X_8)$ as both a direct relation and a selection pair. As a result, we also remove $7$ from $\mathcal{R}[1]$ (line \ref{l24}), so now we have $\mathcal{R}[1] = \{8, 5, 3\}$, and we remove to $(X_1, X_5)$.

    For $(X_1, X_5)$, we have $k=4$ and $\mathbf{T} = \{X_8\}$. Consequentially, we have $X_1 \notindependent X_4 | \{\mathbf{X}_{\text{Pre}(4) \setminus 1}, X_5, X_8\}$ (line \ref{ls2_test1}) and $X_1 \indep X_4 | \{\mathbf{X}_{\text{Pre}(4) \setminus 1}, X_6, X_8\}$ (line \ref{ls2_test2}). Thus, $(X_1, X_5)$ is identified as a direct relation. Then we move to  $(X_1, X_3)$ with $k=2$ and $\mathbf{T} = \{X_8\}$. For line \ref{ls1_test1}, we have $X_1 \indep X_2 | \{\mathbf{X}_{\text{Pre}(2) \setminus 1}, X_3, X_8\}$, which implies that we identify it as a selection pair.

    Finally, we finished all tests in Stage Two. In summary, we have identified all selection pairs and direct relations, while misidentifying a confounded pair $(X_3, X_6)_\mathrm{c}$ as a direct relation and introducing a spurious direct relation for $(X_2, X_6)$ because of the confounded pair $(X_3, X_6)_\mathrm{c}$.

    \item \textbf{Stage Three:} At this stage, we need to identify all confounded pairs while removing all spurious structures. It is worth noting that, according to the previous discussion, all spurious structures after Stage Two are caused by confounded pairs. Thus, if our setting does not allow the existence of latent confounders (Theorem \ref{thm:1}), Stage Three will be removed from the algorithm.

    According to line \ref{ls3_if}, we only need to test $(X_3, X_6)$ since we have both $(X_3, X_6)$ and $(X_2, X_6)$ are identified as direct relations in the previous stage. At this point, we have $\mathbf{T}= \{X_8\}$ (line \ref{l33}) and $X_{2} \indep X_6 | \{\mathbf{X}_{\text{Pre}(6) \setminus \{2,3\}}, X_{7}, X_8\}$. Note that both  $X_{7}$ and $X_8$ are essential in the conditioning set. If we do not condition on both, there will be a d-connected path $\{X_2 \rightarrow X_3 \rightarrow S_\mathrm{I} \leftarrow X_1 \rightarrow S_\mathrm{II} \leftarrow X_8 \leftarrow X_7 \leftarrow X_6 \}$ that makes $X_2$ and $X_6$ conditionally dependent. This is because, if we only add $X_8$ to the conditioning set together with $\mathbf{X}_{\text{Pre}(6) \setminus \{2,3\}}$, there will be another d-connected path $\{X_2 \rightarrow X_3 \rightarrow S_\mathrm{I} \leftarrow X_1 \rightarrow X_8 \leftarrow X_7 \leftarrow X_6 \}$ between $X_2$ and $X_6$ since $X_8$ is a collider on the path. Only when we also add $X_7$ in the condition set, we will have $X_{2} \indep X_6 | \{\mathbf{X}_{\text{Pre}(6) \setminus \{2,3\}}, X_{7}, X_8\}$. Thus, according to lines \ref{ls3_38}-\ref{ls3_40}, we identify $(X_3, X_6)$ as a confounded pair, as well as remove $X_3 \rightarrow X_6$ and $X_2 \rightarrow X_6$ from the edge set $\mathcal{E}$, which are the only remaining spurious direct relations. After that, we have identified all confounded pairs while removing all spurious structures. 
\end{itemize}

Now, we have identified all selection pairs, direct relations, and confounded pairs in Figure \ref{fig:alg_example}. The implementation of tests is based on \textit{causal-learn} \citep{zheng2024causal}.

\subsection{Additional Discussion on the Algorithm}

\textbf{Beyond Structural Conditions.} \ \ In the above example, we have shown how to identify different types of dependencies based on purely observational data, even when there may exist hidden confounders and selection variables. In real-world scenarios, it is important to acknowledge that some structural conditions may not consistently apply universally. The current algorithm is proposed to uncover the complete structure under the premise of identifiability, where all conditions are met. However, for practical applications, we can integrate condition checkers into the algorithm to allow potential ambiguities in the output when certain conditions are not satisfied. Meanwhile, the algorithm could be designed to raise a warning when it is unsure about certain parts of the output structure, and thus users could still make use of the remaining parts and deal with the uncertainty with their desired strategies (e.g., additional experiments).

\textbf{Beyond Sequential Data.} \ \ 
By discovering hidden selection structures from observational data, a more veridical inductive bias can be unveiled to apply in various tasks. In particular, the modeling of sequential data benefits a lot from that selection structure, since the commonly assumed autoregressive structure is usually overcomplicated and dense, and there could be latent confounders and selection variables that introduce different types of dependencies. However, the idea behind our proposed algorithm is also applicable beyond the setting where data has a sequential nature. The sequential structure is leveraged to distinguish selection from other types of dependencies in certain structures and does not need to be satisfied universally. For instance, if there is no selection among a subset of observed variables, their structure can be identified using previous constraint-based methods without assuming the sequential nature. Moreover, even when selection exists in the presence of potential latent confounders in certain structures, we can still identify it without the sequentiality assumption. Consider a non-sequential case with $A \rightarrow B$, $C \leftarrow D$, and an unknown dependence between $B$ and $C$. If this dependence arises from $B$ causing $C$, $C$ causing $B$, or latent confounders for them, there must be at least one edge along the path between $B$ and $C$ that has an arrowhead into $B$ or $C$. This contrasts with scenarios where latent selection variables influence $B$ and $C$, as these introduce additional colliders, distinguishable by testing for conditional dependence given the colliders under the faithfulness assumption. For certain dependencies that cannot be distinguished without the sequential structure, the algorithm may just remain and alert the indeterminacies in the output, or other conditions could be introduced if needed.

\section{Proofs}

\subsection{Proof of Theorem \ref{thm:1}}
\label{sec:proof_thm_1}

\thmI*

\begin{proof}
    We prove the identifiability of the causal graph by showing that the algorithm \ref{alg:global} (only lines \ref{l1}-\ref{l29} since there are no latent confounders according to assumption \ref{assum:t1a2}) identifies all selection pairs and direct relations. Consistent with the two stages of the algorithm, the proof is divided into the following three stages:
    \begin{itemize}
        \item \textbf{Stage One}: Identifying dependent pairs together with several spurious structures.
        \item \textbf{Stage Two}: Identifying selection pairs, dependent pairs that are both selection pairs and direct relations, and direct relations while removing spurious structures.
    \end{itemize}

    \textbf{Stage One.} \ \
    We first show that, after Stage One (lines \ref{l3}-\ref{l10}), the identified sets of variable pairs, denoted as $\mathbf{E}_1$, only consist of the following types of cases:
    \begin{itemize}
        \item All selection pairs $(X_i, X_j)_\mathrm{s}$ in the ground-truth causal graph. 
        \item All direct relations $(X_i, X_j)_\mathrm{x}$ in the ground-truth causal graph.
        \item All dependent pairs that are both selection pairs and direct relations, i.e., $(X_i, X_j)_{\{\mathrm{s,x}\}}$, in the ground-truth causal graph.
        \item All pairs of variables $(X_i, X_j)$ that there exists $(X_i, X_{j+1})$ that is both a selection pair and a direct relation in the ground-truth causal graph.
    \end{itemize}
    For a variable pair $(X_i, X_j)$, we discuss all potential cases for the paths between $X_i$ and $X_j$ as follows. We denote $X_k$ as the observed variable closest to $X_i$.  
    
    \begin{enumerate}[label=\Roman*:]
        \item $(X_i, X_j)$ is a dependent pair. In this case, $X_i$ is d-connected with $X_j$ given any conditioning set, thus  $X_i \notindependent X_j | \mathbf{X}_{\text{Pre}(j) \setminus i}$ and $(X_i, X_j)_\mathrm{d} \in \mathbf{E}_1$. This covers the cases where $(X_i, X_j)$ is a selection pair, a direct relation, and a dependent pair which is both a selection pair and a direct relation. Note that there do not exist any confounded pairs according to assumption \ref{assum:t1a2}.
        
        \item $\{X_i \cdots \leftarrow X_k \cdots X_j\}$, where $k<i$. Since $X_k \in \mathbf{X}_{\text{Pre}(j) \setminus i}$, this path is d-connected only if $X_k$ is a collider on the path. However, $X_k$ cannot be a collider because of $\{X_i \leftarrow X_k\}$. Thus, $X_i \indep X_j | \mathbf{X}_{\text{Pre}(j) \setminus i}$ and $(X_i, X_j)_\mathrm{d} \notin \mathbf{E}_1$.
        
        \item $\{X_i \cdots \rightarrow X_k \cdots X_j\}$, where $k<i$. Because of assumption \ref{assum:t2a1}, $\{X_i \rightarrow X_k\}$ is impossible. Thus, there must be a confounded pair $\{X_k, X_i\}_\mathrm{c}$. Because of assumption \ref{assum:t1a2}, this is impossible and thus $(X_i, X_j)_\mathrm{d} \notin \mathbf{E}_1$.
        
        \item $\{X_i \cdots \leftrightarrow X_k \cdots X_j\}$, where $k<i$ and $\leftrightarrow$ means both tails and arrow heads exist. Because of assumption \ref{assum:t1a1}, $\{X_i \rightarrow X_k\}$ is impossible. Thus, there must be a confounded pair $\{X_k, X_i\}_\mathrm{c}$. Because of assumption \ref{assum:t1a2}, this is impossible and thus $(X_i, X_j)_\mathrm{d} \notin \mathbf{E}_1$.
        
        \item $\{X_i \cdots \leftarrow X_k \cdots X_j\}$, where $k>i$. Based on this, we discuss the following cases: 
        \begin{enumerate}[label=V.\roman*:]
            \item The case where $i < k < j$. Because of assumption \ref{assum:t1a1}, $\{X_i \leftarrow X_k\}$ is impossible. Thus, there must be a selection pair $\{X_i, X_k\}_\mathrm{s}$, and $X_k$ must not be a collider on the path. Since $\mathbf{X}_{\text{Pre}(j) \setminus i}$ is the condition set and $X_k \in \mathbf{X}_{\text{Pre}(j) \setminus i}$, we have $X_i \indep X_j | \mathbf{X}_{\text{Pre}(j) \setminus i}$ ($X_k \in \mathbf{X}_{\text{Pre}(j) \setminus i})$ and thus $(X_i, X_j)_\mathrm{d} \notin \mathbf{E}_1$.
            \item The case where $k > j$. Because of assumption \ref{assum:t1a1}, $\{X_i \leftarrow X_k\}$ is impossible. Thus, there must be a selection pair $\{X_i, X_k\}_\mathrm{s}$, and $X_k$ must not be a collider. Since $k > j$, $X_k$ is not in the conditioning set $\mathbf{X}_{\text{Pre}(j) \setminus i}$. However, the index $k$ for any selection pair $\{X_i, X_k\}_\mathrm{s}$, where $k>j$, has already been included in $\mathcal{R}[i]$ according to the order of tests in the algorithm (lines \ref{l3}-\ref{l10}). Therefore, there is the conditional independence $X_i \indep X_j | \{\mathbf{X}_{\text{Pre}(j) \setminus i}, X_{j+1}, X_k\}$, which implies $(X_i, X_j)_\mathrm{d} \notin \mathbf{E}_1$.
        \end{enumerate}

        \item $\{X_i \cdots \rightarrow X_k \cdots X_j\}$, where $k>i$. Based on this, we discuss the following cases:
        \begin{enumerate}[label=VI.\roman*:]
            \item The case where $i < k < j$. Since $\mathbf{X}_{\text{Pre}(j) \setminus i}$ is the condition set and $X_k \in \mathbf{X}_{\text{Pre}(j) \setminus i}$, $X_k$ must be a collider for the path to be d-connected. On the path, let us denote the $X_p$ as the observed variable closest to $X_k$ other than $X_i$. Thus, we must have $\{X_k \leftarrow \cdots X_p\}$ on the path. We further have the following cases:
                \begin{itemize}
                    \item The case where $p < k$. If $p < k$, we have $p < i$ as $p \neq i$. Since $\{X_k \leftarrow \cdots X_p\}$, $(X_k, X_p)$ is either a confounded pair or direct relation. Because of assumption \ref{assum:t1a2}, $(X_k, X_p)$ can only be a direct relation. Therefore, $X_{p}$ must not be a collider. So we have $X_i \indep X_j | \mathbf{X}_{\text{Pre}(j) \setminus i}$, which implies $(X_i, X_j)_\mathrm{d} \notin \mathbf{E}_1$.
                    \item The case where $p > k$. In this case, $(X_k, X_p)$ can only be a confounded pair. Because of assumption \ref{assum:t1a2}, this is impossible and thus $(X_i, X_j)_\mathrm{d} \notin \mathbf{E}_1$.                  
                \end{itemize}
                
            \item The case where $k > j$. If $k > j$ ($X_k \notin \mathbf{X}_{\text{Pre}(j) \setminus i}$), the only case where the path between $X_i$ and $X_j$ is not d-separated by $X_k$ is the case where $X_k$ is the cause of a selection variable. However, since we have $\{X_i \cdots \rightarrow X_k\}$, it is not possible for $(X_i, X_k)$ to be a selection pair. For any other selection pairs, in order to have $X_i \notindependent X_j | \{\mathbf{X}_{\text{Pre}(j) \setminus i}\}$, $(X_j, X_k)$ must be a selection pair. According to assumption \ref{assum:t1a3}, we have $k \neq j+1$. Therefore, the path is d-separated by either $k$, where $k \in \mathcal{R}[i]$, or $j+1$. Consequently, there must be the conditional independence $X_i \notindependent X_j | \{\mathbf{X}_{\text{Pre}(j) \setminus i}, X_{j+1}, X_k\}$ and thus $(X_i, X_j)_\mathrm{d} \notin \mathbf{E}_1$.
        \end{enumerate}

        \item $\{X_i \cdots \leftrightarrow X_k \cdots X_j\}$, where $k>i$. Based on this, we discuss the following cases:
        \begin{enumerate}[label=VII.\roman*:]
            \item The case where $i < k < j$. Since $\mathbf{X}_{\text{Pre}(j) \setminus i}$ is the condition set and $X_k \in \mathbf{X}_{\text{Pre}(j) \setminus i}$, $X_k$ must be a collider for the path to be d-connected. At the same time, $(X_i, X_k)$ must be a selection pair. We further discuss the following cases:
                \begin{itemize}
                    \item The case where $k = i+1$. According to assumption \ref{assum:t1a3}, $(X_i, X_k)$ cannot be a selection pair if $k = i+1$, which is a contradiction.
                    \item The case where $k \neq i+1$. On the path, let us denote the $X_p$ as the observed variable closest to $X_k$ other than $X_i$. Thus, we must have $\{X_k \leftarrow \cdots X_p\}$ on the path. We further have the following cases:
                     \begin{itemize}
                        \item The case where $p < k$. If $p < k$, we have $p < i$ as $p \neq i$. Since $\{X_k \leftarrow \cdots X_p\}$, $(X_p, X_k)$ is either a confounded pair or a direct relation. Because of assumption \ref{assum:t1a2}, $(X_p, X_k)$ can only be a direct relation. Therefore, $X_{p}$ must not be a collider, and $X_i \indep X_j | \{\mathbf{X}_{\text{Pre}(j) \setminus i}\}$, which implies $(X_i, X_j)_\mathrm{d} \notin \mathbf{E}_1$.          
                        \item The case where $p > k$. In this case, $(X_k, X_p)$ can only be a confounded pair. Because of assumption \ref{assum:t1a2}, this is impossible and thus $(X_i, X_j)_\mathrm{d} \notin \mathbf{E}_1$.
                    \end{itemize}
                \end{itemize}
            \item The case where $k > j$. If $k > j$ ($X_k \notin \mathbf{X}_{\text{Pre}(j) \setminus i}$), the only case where the path between $X_i$ and $X_j$ is not d-separated by $X_k$ is the case where $X_k$ is the cause of a selection variable. However, the index $k$ for any selection pair $\{X_{i}, X_k\}_\mathrm{s}$, where $k>j$, has already been included in $R[i]$ according to the order of tests in the algorithm (lines \ref{l3}-\ref{l10}). According to assumption \ref{assum:t1a3}, $X_k$ cannot be caused by more than two observed variables via higher-order direct relations. Therefore, the only case for $X_i$ to be conditionally dependent of $X_j$ given $\{\mathbf{X}_{\text{Pre}(j) \setminus i}, X_{j+1}, X_k\}$ for all $k \in \mathcal{R}[i]$ is the structure that $(X_i, X_{j+1})_{\{s,x\}}$. At the same time, we also have $X_i \notindependent X_j | \{\mathbf{X}_{\text{Pre}(j) \setminus i}\}$. As a result, we have $(X_i, X_j)_\mathrm{d} \in \mathbf{E}_1$ where $(X_i, X_{j+1})$ is both a selection pair and a direct relation.
        \end{enumerate}
    \end{enumerate}

    \textbf{Stage Two.} \ \ 
    Next, we prove that we can identify selection pairs, direct relations, and dependent pairs that are both selection pairs and direct relations in $\mathcal{E}_1$ after lines \ref{l11}-\ref{l29} in the algorithm.
    
    We discuss the following cases for all types in $\mathcal{E}_1$:
        \begin{enumerate}[label=\Roman*:]
            \item $(X_i, X_j)$ is a selection pair but not a direct relation. According to the algorithm, it will be identified as a selection pair but not a direct relation if and only if $X_i \indep X_k | \{\mathbf{X}_{\text{Pre}(k) \setminus i}, X_j, \mathbf{T}\}$ (the definition of $\mathbf{T}$ is in line \ref{l17} in the algorithm).
            

            We first show that all these selection pairs will be identified. If $(X_i, X_j)$ is a selection pair but not a direct relation, $X_j$ must be a non-collider on the path since there is no latent confounder according to assumption \ref{assum:t1a2}, which may lead to the d-separation and thus the conditional independence. Thus, we need to show that there exists no other path that leads to conditional dependence. We prove it by contradiction. 
            
            Suppose that there exist a path other than $\{X_i \rightarrow S \leftarrow X_j \cdots X_k\}$ that makes $X_i \notindependent X_k | \{\mathbf{X}_{\text{Pre}(k) \setminus i}, X_j, \mathbf{T}\}$. We discuss all potential cases for that path:

            \begin{enumerate}[label=I.\roman*:]
                \item $\{X_i \cdots \leftarrow X_p \cdots X_k\}$, where $p<i$ and $X_p$ denotes the observed variable that is the closest to $X_i$ on the path. Since $X_p \in \mathbf{X}_{\text{Pre}(k) \setminus i}$, this path is d-connected if and only if $X_p$ is a collider on the path. However, $X_p$ cannot be a collider because of $\{X_i \leftarrow X_p\}$. Thus, this case is impossible.
                \item $\{X_i \cdots \rightarrow X_p \cdots X_k\}$, where $p<i$ and $X_p$ denotes the observed variable that is the closest to $X_i$ on the path. Because of assumption \ref{assum:t1a1}, $\{X_i \rightarrow X_p\}$ is impossible. Thus, there must be a confounded pair $\{X_p, X_i\}_\mathrm{c}$. Because of assumption \ref{assum:t1a2}, there are no latent confounders and thus this case is impossible.
                \item $\{X_i \cdots \leftrightarrow X_p \cdots X_k\}$, where $p<i$ and $X_p$ denotes the observed variable that is the closest to $X_i$ on the path. Because of assumption \ref{assum:t1a1}, $\{X_p \rightarrow X_i\}$ is impossible. Thus, there must be a confounded pair $\{X_p, X_i\}_\mathrm{c}$. Because of assumption \ref{assum:t1a2}, there are no latent confounders and thus this case is impossible.
           
                \item $\{X_i \cdots \leftarrow X_p \cdots X_k\}$, where $i < p < k$ and $X_p$ denotes the observed variable that is the closest to $X_i$ on the path. Because of assumption \ref{assum:t1a1}, $\{X_i \leftarrow X_p\}$ is impossible. Thus, there must be a selection pair $(X_i, X_p)_\mathrm{s}$, and $X_p$ must not be a collider. Since $\{\mathbf{X}_{\text{Pre}(k) \setminus i}, X_j, \mathbf{T}\}$ is the condition set and $X_p \in \mathbf{X}_{\text{Pre}(k) \setminus i}$, we have $X_i \indep X_k | \{\mathbf{X}_{\text{Pre}(k) \setminus i}, X_j, \mathbf{T}\}$ ($X_p \in \mathbf{X}_{\text{Pre}(k) \setminus i}$) and thus this case is impossible.

                \item $\{X_i \cdots \rightarrow X_p \cdots X_k\}$, where $i < p < k$. Since $\mathbf{X}_{\text{Pre}(k) \setminus i}$ is in the condition set and $X_p \in \mathbf{X}_{\text{Pre}(k) \setminus i}$, $X_p$ must be a collider for the path to be d-connected. On the path, let us denote the $X_q$ as the observed variable closest to $X_p$ other than $X_i$. Thus, we must have $\{X_p \leftarrow \cdots X_q\}$ on the path. We further have the following cases:
                    \begin{itemize}
                    \item The case where $q < p$. If $q < p$, we have $q < i$ as $q \neq i$. Since $\{X_p \leftarrow \cdots X_q\}$, $(X_q, X_p)$ is either a confounded pair or direct relation. Because of assumption \ref{assum:t1a2}, it can only be a direct relation. If $(X_q, X_p)$ is a direct relation, $X_{q}$ must not be a collider, and $X_i \indep X_k | \{\mathbf{X}_{\text{Pre}(k) \setminus i}, X_j, \mathbf{T}\}$ ($X_q \in \mathbf{X}_{\text{Pre}(k) \setminus i}$) and thus this case is impossible.
                    \item The case where $q > p$. In this case, $(X_p, X_q)$ can only be a confounded pair. Because of assumption \ref{assum:t1a2}, there are no latent confounders and thus this case is impossible.
                \end{itemize}

                \item $\{X_i \cdots \leftrightarrow X_p \cdots X_k\}$, where $i < p < k$. Since $\mathbf{X}_{\text{Pre}(k) \setminus i}$ is in the condition set and $X_p \in \mathbf{X}_{\text{Pre}(k) \setminus i}$, $X_p$ must be a collider for the path to be d-connected. At the same time, $(X_i, X_p)$ must be a selection pair. We further discuss the following cases:
                \begin{itemize}
                    \item The case where $p = i+1$. According to assumption \ref{assum:t1a3}, $(X_i, X_p)$ cannot be a selection pair if $p = i+1$, which is a contradiction.
                    \item The case where $p \neq i+1$. On the path, let us denote the $X_q$ as the observed variable closest to $X_p$ other than $X_i$. Thus, we must have $\{X_p \leftarrow \cdots X_q\}$ on the path. We further have the following cases:
                     \begin{itemize}
                        \item The case where $q < p$. If $q < p$, we have $q < i$ as $q \neq i$. Since $\{X_p \leftarrow \cdots X_q\}$, $(X_q, X_p)$ is either a confounded pair or a direct relation. Because of assumption \ref{assum:t1a2}, $(X_q, X_p)$ must be a direct relation, $X_{q}$ must not be a collider, and the path is, again, not d-connected. Thus, this case is impossible.
                        \item The case where $q > p$. In this case, $(X_p, X_q)$ can only be a confounded pair. Because of assumption \ref{assum:t1a2}, there should not be any latent confounders. Thus, this case is impossible.
                    \end{itemize}               
                \end{itemize}
                
                \item $\{X_i \cdots \leftarrow X_p\}$, where $k=p$. This is impossible because of the algorithm, specifically, how $k$ is selected (lines 20-25).

                \item $\{X_i \cdots \rightarrow X_p\}$, where $k=p$. This is also impossible because of the algorithm.

                \item $\{X_i \cdots \leftrightarrow X_p\}$, where $k=p$. Similarly, this is impossible according to the algorithm.

                \item $\{X_i \cdots \leftarrow X_p \cdots X_k\}$, $k < p < j$. Because of assumption \ref{assum:t1a1}, $\{X_i \leftarrow X_p\}$ is impossible. Thus, there must be a selection pair $(X_i, X_p)_\mathrm{s}$. 
                Together with the search procedure (lines \ref{l11}-\ref{l16}), it is thus impossible for the selected $k$ to be smaller than $p$, which is a contradiction.

                \item $\{X_i \cdots \rightarrow X_p \cdots X_k\}$, $k < p < j$. Because of $\{X_i \cdots \rightarrow X_p\}$, $(X_i, X_p)$ is a higher-order direct relation. Since $X_p$ is not in the condition set, $(X_i, X_p)$ must also be a selection pair for the path to be d-connected. However, according to assumption \ref{assum:t1a3}, there exits no $(X_i, X_{p+1})_\mathrm{d}$. This contradicts the search procedure to select $k$ (lines \ref{l11}-\ref{l16}), and thus this case is impossible.

                \item $\{X_i \cdots \leftrightarrow X_p \cdots X_k\}$, $k < p < j$. Because of assumption \ref{assum:t1a1}, $\{X_i \leftarrow X_p\}$ is impossible. Thus, there must be a selection pair $(X_i, X_p)_\mathrm{s}$. Similarly, it is thus impossible for the selected $k$ to be smaller than $p$, which is a contradiction.

                \item $\{X_i \cdots \leftarrow X_p \cdots X_k\}$, $j < p$. Because of assumption \ref{assum:t1a1}, $\{X_i \leftarrow X_p\}$ is impossible. Thus, there must be a selection pair $(X_i, X_p)_\mathrm{s}$. Moreover, according to assumption \ref{assum:t1a3}, $X_p$ cannot be the effect of multiple higher-order directions. Thus, $X_p$ must not be a collider on the path, so we can d-separate the path by conditioning on it. Since $X_p$ is already in $\mathcal{T}$ according to the algorithm, we have $X_i \indep X_k | \{\mathbf{X}_{\text{Pre}(k) \setminus i}, X_j, \mathbf{T}\}$ and thus this case is impossible.

                \item $\{X_i \cdots \rightarrow X_p \cdots X_k\}$, $j < p$. In this case, $(X_i, X_p)$ is either a higher-order direct relation or a confounded pair. Because of assumption \ref{assum:t1a2}, there are no latent confounders and thus it can only be a higher-order direct relation. We further discuss the following case:
                \begin{itemize}
                    \item $X_p$ is not a cause of a selection variable. In this case, $X_p$ must be a collider on the path, and the path is d-separated when not conditioning on $X_p$. So we have $X_i \indep X_k | \{\mathbf{X}_{\text{Pre}(k) \setminus i}, X_j, \mathbf{T}\}$, which indicates that this case is impossible.
                    \item $X_p$ is a cause of a selection variable. According to assumption \ref{assum:t1a3}, it cannot be the effect of multiple higher-order directions. Therefore, if $X_p$ is the cause of another selection variable together with any observed variable with index $u \in [k, j)$, the path is always d-separated by conditioning on both $X_p$ and $X_j$. At the same time, if $X_p$ is the cause of another selection variable together with any observed variable with an index larger than $j$, that variable will also be in the condition set ($\mathbf{T}$) according to the algorithm (line \ref{l17}). Recursively, all these variables will be in the condition set, and the path will be d-separated, which indicates that this case is impossible.               
                \end{itemize}

                \item $\{X_i \cdots \leftrightarrow X_p \cdots X_k\}$, $j < p$. Because of assumption \ref{assum:t1a1}, $\{X_i \leftarrow X_p\}$ is impossible. Thus, there must be a selection pair $(X_i, X_p)_\mathrm{s}$. At the same time, we have $(X_i, X_p)_{\mathrm{d} \neq \mathrm{c}}$, i.e., it is both a selection pair and a direct relation. According to assumption \ref{assum:t1a3}, it cannot be the effect of multiple higher-order directions. Therefore, if $X_p$ is the cause of another selection variable together with any observed variable with index $u \in [k, j)$, the path is always d-separated by conditioning on both $X_p$ and $X_j$. At the same time, if $X_p$ is the cause of another selection variable together with any observed variable with an index larger than $j$, that variable will also be in the condition set ($\mathbf{T}$) according to the algorithm (line \ref{l17}). Recursively, all these variables will be in the condition set, and the path will be d-separated, which indicates that this case is impossible.               
                
            \end{enumerate}
            Therefore, there does not exist any path other than $\{X_i \rightarrow S \leftarrow X_j \cdots X_k\}$ that makes $X_i \notindependent X_k | \{\mathbf{X}_{\text{Pre}(k) \setminus i}, X_j, \mathbf{T}\}$. So all selection pairs $(X_i, X_j)_\mathrm{s}$ are identified.
            
            We then discuss the other cases in $\mathcal{E}_1$ and show that it is impossible for any other types of variable pairs in $\mathcal{E}_1$ to have the test results that $X_i \indep X_k | \{\mathbf{X}_{\text{Pre}(k) \setminus i}, X_j, \mathbf{T}\}$.

            \begin{itemize}
                \item $(X_i, X_j)$ is a direct relation but not a selection pair. In this case, since $X_j$ is collider on the path $\{X_i \cdots \rightarrow X_j\}$, for any $v \in \mathbf{T}$, there is always $X_i \notindependent X_k | \{\mathbf{X}_{\text{Pre}(k) \setminus i}, X_j, \mathbf{T}\}$. Thus, this case is impossible.
                \item $(X_i, X_j)$ is both a selection pair and a direct relation. In this case, since $X_j$ is still a collider on the path $\{X_i \cdots \rightarrow X_j\}$, for any $v \in \mathbf{T}$, there is always $X_i \notindependent X_k | \{\mathbf{X}_{\text{Pre}(k) \setminus i}, X_j, \mathbf{T}\}$, which indicates that this case is impossible.
                \item $(X_i, X_j)$ is a spurious dependent pair that there exists $(X_i, X_{j+1})$ that is both a selection pair and a direct relation in the ground-truth causal graph. In this case, $(X_i, X_j)$ is not a dependent pair according to assumption \ref{assum:t1a3}, and will be removed by the algorithm (line \ref{l24}). Therefore, this case is impossible.
            \end{itemize}
       
            \item $(X_i, X_j)$ is a direct relation but not a selection pair. We first show that all these variable pairs will be identified. In this case, $X_j$ is a collider on the path. Thus, we have $X_i \notindependent X_k | \{\mathbf{X}_{\text{Pre}(k) \setminus i}, X_j, \mathbf{T}\}$. However, when $X_j$ is not in the condition set, we may have the d-separation and thus the conditional independence $X_i \indep X_k | \{\mathbf{X}_{\text{Pre}(k) \setminus i}, X_{j+1}, \mathbf{T}\}$ (the definition of $\mathbf{T}$ is in line \ref{l17} in the algorithm). Thus, we need to show that there exists no other path that leads to conditional dependence. We prove it by contradiction. 

            Suppose that there exist a path other than $\{X_i \rightarrow  X_j \cdots X_k\}$ that makes  $X_i \notindependent X_k | \{\mathbf{X}_{\text{Pre}(k) \setminus i}, X_{j+1}, \mathbf{T}\}$. Denote $X_p$ as the observed variable that is the closest to $X_i$ on the path, we discuss all potential cases for that path:

            \begin{enumerate}[label=II.\roman*:]
                \item The discussions for the following cases are similar to those in the cases where $(X_i, X_j)$ is a selection pair but not a direct relation:
                    \begin{itemize}
                        \item $\{X_i \cdots \leftarrow X_p \cdots X_k\}$, where $p<i$, $i < p < k$, or $k < p < j$.
                        \item $\{X_i \cdots \rightarrow X_p \cdots X_k\}$, where $p<i$ or $k < p < j$ .
                        \item $\{X_i \cdots \leftrightarrow X_p \cdots X_k\}$, where $p<i$, $i < p < k$, or $k < p < j$.
                        \item $\{X_i \cdot X_p\}$, where $p = k$. 
                    \end{itemize}
                    Thus, we mainly focus on the other cases as follows.

                \item $\{X_i \cdots \rightarrow X_p \cdots X_k\}$, where $i < p < k$. Since $\mathbf{X}_{\text{Pre}(k) \setminus i}$ is in the condition set and $X_p \in \mathbf{X}_{\text{Pre}(k) \setminus i}$, $X_p$ must be a collider for the path to be d-connected. On the path, let us denote the $X_q$ as the observed variable closest to $X_p$ other than $X_i$. Thus, we must have $\{X_p \leftarrow \cdots X_q\}$ on the path. We further have the following cases:
                \begin{itemize}
                    \item The case where $q < p$. If $q < p$, we have $q < i$ as $q \neq i$. Since $\{X_p \leftarrow \cdots X_q\}$, $(X_p, X_q)$ is either a confounded pair or a direct relation. Because of assumption \ref{assum:t1a2}, it can only be a direct relation. Therefore, $X_{q}$ must not be a collider, and $X_i \indep X_k | \{\mathbf{X}_{\text{Pre}(k) \setminus i}, \mathbf{T}\}$ ($X_{q-1} \in \mathbf{X}_{\text{Pre}(k) \setminus i}$) and thus this case is impossible.
                    \item The case where $q > p$. In this case, $(X_p, X_q)$ can only be a confounded pair. Because of assumption \ref{assum:t1a2}, there are no latent confounders and thus this case is impossible.
                \end{itemize}

                \item $\{X_i \cdots \leftarrow X_p \cdots X_k\}$, $j < p$. Because of assumption \ref{assum:t1a3}, $\{X_i \leftarrow X_p\}$ is impossible. Thus, there must be a selection pair $(X_i, X_p)_\mathrm{s}$. Moreover, according to assumption \ref{assum:t1a3}, $X_p$ cannot be the effect of multiple higher-order directions. Thus, $X_p$ must not be a collider on the path, so we can d-separate the path by conditioning on it. Since $X_p$ is already in $\mathbf{T}$ according to the algorithm, we have $X_i \indep X_k | \{\mathbf{X}_{\text{Pre}(k) \setminus i}, X_{j+1}, \mathbf{T}\}$ and thus this case is impossible.

                \item $\{X_i \cdots \rightarrow X_p \cdots X_k\}$, $j < p$. In this case, $(X_i, X_p)$ is either a higher-order direct relation or a confounded pair. Because of assumption \ref{assum:t1a2}, there are no latent confounders, so it can only be a higher-order direct relation. We further discuss the following case:
                \begin{itemize}
                    \item $X_p$ is not a cause of a selection variable. In this case, $X_p$ must be a collider on the path, and the path is d-separated when not conditioning on $X_p$. So we have $X_i \indep X_k | \{\mathbf{X}_{\text{Pre}(k) \setminus i}, X_{j+1}, \mathbf{T}\}$, which indicates that this case is impossible.
                    \item $X_p$ is a cause of a selection variable. According to assumption \ref{assum:t1a3}, it cannot be the effect of multiple higher-order directions. Meanwhile, the index $p$ has already been included in $R[i]$ according to the order of tests in the algorithm (lines \ref{l12} and \ref{l13}). At the same time, it is impossible for $p$ to be $j+1$ according to assumption \ref{assum:t1a3}. This is because if $X_{j+1}$ is a cause of a selection variable while $(X_i, X_{j+1})$ is a higher-order direct relation, it is impossible to have $(X_i, X_j)_\mathrm{d}$. Therefore, the path is always d-separated by conditioning on $\{\mathbf{X}_{\text{Pre}(k) \setminus i}, X_{j+1}, \mathbf{T}\}$ and thus this case is impossible.            
                \end{itemize}

                \item $\{X_i \cdots \leftrightarrow X_p \cdots X_k\}$, $j < p$. Because of assumption \ref{assum:t1a1}, $\{X_i \leftarrow X_p\}$ is impossible. Thus, there must be a selection pair $(X_i, X_p)_\mathrm{s}$. At the same time, we have $(X_i, X_p)_{\mathrm{d} \neq \mathrm{c}}$, i.e., it is both a selection pair and a direct relation. According to assumption \ref{assum:t1a3}, it cannot be the effect of multiple higher-order directions. Meanwhile, it is impossible for $p$ to be $j+1$ according to assumption \ref{assum:t1a3}. This is because if $X_{j+1}$ is a cause of a selection variable while $(X_i, X_{j+1})$ is a higher-order direct relation, it is impossible to have $(X_i, X_j)_\mathrm{d}$. Therefore, the path is always d-separated by conditioning on $\{\mathbf{X}_{\text{Pre}(k) \setminus i}, X_{j+1}, \mathbf{T}\}$ and thus this case is impossible.            
            \end{enumerate}


            
        We then discuss the other cases in $\mathcal{E}_1$ and show that it is impossible for any other types of variable pairs in $\mathcal{E}_1$ to be identified as a direct relation.

        \begin{itemize}
            \item $(X_i, X_j)$ is a selection pair but not a directed relation. As already shown before, it will always be identified as a selection pair but not a direct relation. Thus, this case is impossible
            
            \item $(X_i, X_j)$ is both a selection pair and a direct relation. In this case, since $X_j$ is still a collider on the path $\{X_i \cdots \rightarrow X_k\}$, there is always $X_i \notindependent X_k | \{\mathbf{X}_{\text{Pre}(k) \setminus i}, X_j, \mathbf{T}\}$. At the same time, since $(X_i, X_j)$ is also a selection pair, not conditioning on $X_j$ will also open the path. Thus, we also have $X_i \notindependent X_k | \{\mathbf{X}_{\text{Pre}(k) \setminus i}, X_{j+1}, \mathbf{T}\}$, which implies that this variable pair will not be identified as just a direct relation according to the algorithm (lines \ref{l19}-\ref{l22}).

            \item $(X_i, X_j)$ is a spurious dependent pair that there exists $(X_i, X_{j+1})$ that is both a selection pair and a direct relation in the ground-truth causal graph. In this case, $(X_i, X_j)$ is not a dependent pair according to assumption \ref{assum:t1a3}, and will be removed by the algorithm (line \ref{l24}). Therefore, this case is impossible.

        \end{itemize}

        \item $(X_i, X_j)$ is both a selection pair and a direct relation. In this case, since $X_j$ is still a collider on the path $\{X_i \cdots \rightarrow X_k\}$, there is always $X_i \notindependent X_k | \{\mathbf{X}_{\text{Pre}(k) \setminus i}, X_j, \mathbf{T}\}$. At the same time, since $(X_i, X_j)$ is also a selection pair, not conditioning on $X_j$ will also open the path. Thus, we also have $X_i \notindependent X_k | \{\mathbf{X}_{\text{Pre}(k) \setminus i}, X_{j+1}, \mathbf{T}\}$. We then discuss the other cases in $\mathcal{E}_1$ and show that, except for those variable pairs that are both a selection pair and a direct relation, it is impossible for any other types of variable pairs in $\mathcal{E}_1$ to be identified as it.

        \begin{itemize}
            \item $(X_i, X_j)$ is a selection pair but not a directed relation. As already shown before, it will always be identified as a selection pair but not a direct relation. Thus, this case is impossible
            
            \item $(X_i, X_j)$ is a direct relation but not a selection pair. In this case, according to the discussion before, there is always the conditional independence $X_i \notindependent X_k | \{\mathbf{X}_{\text{Pre}(k) \setminus i}, X_{j+1}, \mathbf{T}\}$. Thus, it is impossible for this case to be identified as a variable pair i.e. both a selection pair and a direct relation.

             \item $(X_i, X_j)$ is a spurious dependent pair that there exists $(X_i, X_{j+1})$ that is both a selection pair and a direct relation in the ground-truth causal graph. In this case, $(X_i, X_j)$ is not a dependent pair according to assumption \ref{assum:t1a3}, and will be removed by the algorithm (line \ref{l24}). Therefore, this case is impossible.
        \end{itemize}

        \item $(X_i, X_j)$ is a spurious dependent pair that there exists $(X_i, X_{j+1})$ that is both a selection pair and a direct relation in the ground-truth causal graph. In this case, $(X_i, X_j)$ is not a dependent pair according to assumption \ref{assum:t1a3}, and will be removed by the algorithm (line \ref{l24}).
        \end{enumerate}
    Therefore, after Stage Two, we identify all selection pairs and direct relations. As a result, we have shown that all selection pairs and direct relations in the ground-truth causal graph are identifiable.
\end{proof}

\subsection{Proof of Theorem \ref{thm:2}}
\label{sec:proof_thm_2}

\thmII*

\begin{proof}
    We prove the identifiability of the causal graph by showing that the algorithm \ref{alg:global} identifies all selection pairs, direct relations, and confounded pairs. Consistent with the three stages of the algorithm, the proof is divided into the following three stages:
    \begin{itemize}
        \item \textbf{Stage One}: Identifying dependent pairs (selection pairs, direct relations, and confounded pairs) together with several spurious structures.
        \item \textbf{Stage Two}: Identifying selection pairs, dependent pairs that are both selection pairs and direct relations, and direct relations, while remaining some spurious direct relations corresponding to unidentified confounded pairs.
        \item \textbf{Stage Three}: Identifying confounded pairs while removing all spurious structures.
    \end{itemize}

    \textbf{Stage One.} \ \
    We first show that, after lines \ref{l1}-\ref{l10}, the identified sets of variable pairs, denoted as $\mathbf{E}_1$, only consists of the following types of cases:
    \begin{itemize}
        \item All selection pairs $(X_i, X_j)_\mathrm{s}$ in the ground-truth causal graph. 
        \item All direct relations $(X_i, X_j)_\mathrm{x}$ in the ground-truth causal graph.
        \item All confounded pairs $(X_i, X_j)_\mathrm{c}$ in the ground-truth causal graph.
        \item All pairs of variables $(X_i, X_j)$ that there exist confounded pairs $(X_{i+1}, X_j)_\mathrm{c}$ in the ground-truth causal graph.
        \item All dependent pairs that are both selection pairs and direct relations, i.e., $(X_i, X_j)_{\{s,x\}}$, in the ground-truth causal graph.
        \item All pairs of variables $(X_i, X_j)$ that there exists $(X_i, X_{j+1})$ that is both a selection pair and a direct relation in the ground-truth causal graph.
    \end{itemize}
    For a variable pair $(X_i, X_j)$, we discuss all potential cases for the paths between $X_i$ and $X_j$ as follows. We denote $X_k$ as the observed variable closest to $X_i$.  
    
    \begin{enumerate}[label=\Roman*:]
        \item $(X_i, X_j)$ is a dependent pair. In this case, $X_i$ is d-connected with $X_j$ given any conditioning set, thus  $X_i \notindependent X_j | \mathbf{X}_{\text{Pre}(j) \setminus i}$ and $(X_i, X_j)_\mathrm{d} \in \mathbf{E}_1$. This covers the cases where $(X_i, X_j)$ is a selection pair, a direct relation, a confounded pair, and a dependent pair which is both a selection pair and a direct relation.
        
        \item $\{X_i \cdots \leftarrow X_k \cdots X_j\}$, where $k<i$. Since $X_k \in \mathbf{X}_{\text{Pre}(j) \setminus i}$, this path is d-connected if and only if $X_k$ is a collider on the path. However, $X_k$ cannot be a collider because of $\{X_i \leftarrow X_k\}$. Thus, $X_i \indep X_j | \mathbf{X}_{\text{Pre}(j) \setminus i}$ and $(X_i, X_j)_\mathrm{d} \notin \mathbf{E}_1$.
        
        \item $\{X_i \cdots \rightarrow X_k \cdots X_j\}$, where $k<i$. Because of assumption \ref{assum:t2a1}, $\{X_i \rightarrow X_k\}$ is impossible. Thus, there must be a confounded pair $\{X_k, X_i\}_\mathrm{c}$. Because of assumption \ref{con:2}, $X_k$ cannot be caused by another confounder or observed variable via higher-order direct relation. Therefore, $X_k$ is not a collider on the path, and we have $X_i \indep X_j | \mathbf{X}_{\text{Pre}(j) \setminus i}$ ($X_k \in \mathbf{X}_{\text{Pre}(j) \setminus i}$) and thus $(X_i, X_j)_\mathrm{d} \notin \mathbf{E}_1$.
        
        \item $\{X_i \cdots \leftrightarrow X_k \cdots X_j\}$, where $k<i$ and $\leftrightarrow$ means both tails and arrow heads exist. Because of assumption \ref{assum:t2a1}, $\{X_i \rightarrow X_k\}$ is impossible. Thus, there must be a confounded pair $\{X_k, X_i\}_\mathrm{c}$. Because of assumption \ref{con:2}, $(X_k, X_i)$ cannot be a selection pair or direct relation, and $X_k$ cannot be caused by another confounder or observed variable via higher-order direct relation. We further discuss the following cases:
        \begin{enumerate}[label=IV.\roman*:]
            \item $X_k$ is a collider. For $X_k$ to be a collider, the only case is $\{X_{k-1} \rightarrow X_k \leftarrow C \rightarrow X_i\}$, where $C$ as the corresponding latent confounder in $\mathbf{C}$. In that case, $X_{k-1} \in \mathbf{X}_{\text{Pre}(j) \setminus i}$ cannot be a collider on the path. Thus, we have $X_i \indep X_j | \mathbf{X}_{\text{Pre}(j) \setminus i}$, which implies $(X_i, X_j)_\mathrm{d} \notin \mathbf{E}_1$.
            \item $X_k$ is not a collider. Then we have $X_i \indep X_j | \mathbf{X}_{\text{Pre}(j) \setminus i}$ ($X_k \in \mathbf{X}_{\text{Pre}(j) \setminus i}$) and thus $(X_i, X_j)_\mathrm{d} \notin \mathbf{E}_1$.
        \end{enumerate}
        
        \item $\{X_i \cdots \leftarrow X_k \cdots X_j\}$, where $k>i$. Based on this, we discuss the following cases: 
        \begin{enumerate}[label=V.\roman*:]
            \item The case where $i < k < j$. Because of assumption \ref{assum:t2a1}, $\{X_i \leftarrow X_k\}$ is impossible. Thus, there must be a selection pair $\{X_i, X_k\}_\mathrm{s}$, and $X_k$ must not be a collider on the path. Since $\mathbf{X}_{\text{Pre}(j) \setminus i}$ is the condition set and $X_k \in \mathbf{X}_{\text{Pre}(j) \setminus i}$, we have $X_i \indep X_j | \mathbf{X}_{\text{Pre}(j) \setminus i}$ ($X_k \in \mathbf{X}_{\text{Pre}(j) \setminus i})$ and thus $(X_i, X_j)_\mathrm{d} \notin \mathbf{E}_1$.
            \item The case where $k > j$. Because of assumption \ref{assum:t2a1}, $\{X_i \leftarrow X_k\}$ is impossible. Thus, there must be a selection pair $\{X_i, X_k\}_\mathrm{s}$, and $X_k$ must not be a collider. Since $k > j$, $X_k$ is not in the conditioning set $\mathbf{X}_{\text{Pre}(j) \setminus i}$. However, the index $k$ for any selection pair $\{X_i, X_k\}_\mathrm{s}$, where $k>j$, has already been included in $\mathcal{R}[i]$ according to the order of tests in the algorithm (lines \ref{l3}-\ref{l10}). Therefore, there is the conditional independence $X_i \indep X_j | \{\mathbf{X}_{\text{Pre}(j) \setminus i}, X_{j+1}, X_k\}$, which implies $(X_i, X_j)_\mathrm{d} \notin \mathbf{E}_1$.
        \end{enumerate}

        \item $\{X_i \cdots \rightarrow X_k \cdots X_j\}$, where $k>i$. Based on this, we discuss the following cases:
        \begin{enumerate}[label=VI.\roman*:]
            \item The case where $i < k < j$. Since $\mathbf{X}_{\text{Pre}(j) \setminus i}$ is the condition set and $X_k \in \mathbf{X}_{\text{Pre}(j) \setminus i}$, $X_k$ must be a collider for the path to be d-connected. On the path, let us denote the $X_p$ as the observed variable closest to $X_k$ other than $X_i$. Thus, we must have $\{X_k \leftarrow \cdots X_p\}$ on the path. We further have the following cases:
                \begin{itemize}
                    \item The case where $p < k$. If $p < k$, we have $p < i$ as $p \neq i$. Since $\{X_k \leftarrow \cdots X_p\}$, $(X_k, X_p)$ is either a confounded pair or direct relation. We further discuss the following cases:
                    \begin{itemize}
                        \item If $(X_p, X_k)$ is a confounded pair, because $X_p$ cannot be caused by another confounder (assumption \ref{con:2}), the only case for $X_p$ to also be a collider is $\{X_k \leftarrow C \rightarrow X_p \leftarrow X_{p-1}$. Thus, $X_{p-1}$ must not be a collider on the path, and $X_i \indep X_j | \{\mathbf{X}_{\text{Pre}(j) \setminus i}\}$, which implies $(X_i, X_j)_\mathrm{d} \notin \mathbf{E}_1$.
                        \item If $(X_p, X_k)$ is a direct relation, $X_{p}$ must not be a collider. So we have $X_i \indep X_j | \mathbf{X}_{\text{Pre}(j) \setminus i}$, which implies $(X_i, X_j)_\mathrm{d} \notin \mathbf{E}_1$.
                    \end{itemize}
                
                    \item The case where $p > k$. In this case, $(X_k, X_p)$ can only be a confounded pair. Because of assumption \ref{con:2}, $X_k$ and $X_p$ cannot be caused by another confounder or higher-order direct relation. Therefore, there must be $\{X_i \rightarrow X_k\}$ and $k = i+1$ for the path to be d-connected. We further discuss the following cases:
                    \begin{itemize}
                        \item If $k< p < j$, because of assumption \ref{con:2}, $X_p$ cannot be caused by another confounder or higher-order direct relation. If $X_p$ is a collider, $X_{p-1}$ must not be a collider on the path as $\{X_{p-1} \rightarrow X_p\}$. Thus, we have $X_i \indep X_j | \mathbf{X}_{\text{Pre}(j) \setminus i}$, which implies $(X_i, X_j)_\mathrm{d} \notin \mathbf{E}_1$.
                        \item If $p = j$, then $X_i \notindependent X_j | \{\mathbf{X}_{\text{Pre}(j) \setminus i}\}$ because of $(X_{i+1}, X_{j})_\mathrm{c}$, which implies $(X_i, X_j)_\mathrm{d} \in \mathbf{E}_1$.
                        \item If $p > j$ ($X_p \notin \mathbf{X}_{\text{Pre}(j) \setminus i}$), the only case where the path between $X_i$ and $X_j$ is not d-separated by $X_p$ is the case where $X_p$ is the cause of a selection variable. However, the index $p$ has already been included in $R[i]$ according to the order of tests in the algorithm (lines \ref{l3}-\ref{l10}). According to assumption \ref{con:1}, $X_p$ cannot be caused by more than two observed variables via higher-order direct relations. Therefore, the only case for $X_i$ to be conditionally dependent of $X_j$ given $\{\mathbf{X}_{\text{Pre}(j) \setminus i}, X_{j+1}, X_k\}$ for all $k \in \mathcal{R}[i]$ is the structure that $\{X_i \rightarrow S \leftarrow X_p\}$ and $p = j+1$. However, this contradicts with assumption \ref{con:2} since $\{X_{i+1}, X_p\}$ is a confounded pair.
                        
                    \end{itemize}
                    
                \end{itemize}
                
            \item The case where $k > j$. If $k > j$ ($X_k \notin \mathbf{X}_{\text{Pre}(j) \setminus i}$), the only case where the path between $X_i$ and $X_j$ is not d-separated by $X_k$ is the case where $X_k$ is the cause of a selection variable. However, since we have $\{X_i \cdots \rightarrow X_k\}$, it is not possible for $(X_i, X_k)$ to be a selection pair. For any other selection pairs, in order to have $X_i \notindependent X_j | \{\mathbf{X}_{\text{Pre}(j) \setminus i}\}$, $(X_j, X_k)$ must be a selection pair. According to assumption \ref{con:1}, we have $k \neq j+1$. Therefore, the path is d-separated by either $k$, where $k \in \mathcal{R}[i]$, or $j+1$. Consequently, there must be the conditional independence $X_i \notindependent X_j | \{\mathbf{X}_{\text{Pre}(j) \setminus i}, X_{j+1}, X_k\}$ and thus $(X_i, X_j)_\mathrm{d} \notin \mathbf{E}_1$.
        \end{enumerate}

        \item $\{X_i \cdots \leftrightarrow X_k \cdots X_j\}$, where $k>i$. Based on this, we discuss the following cases:
        \begin{enumerate}[label=VII.\roman*:]
            \item The case where $i < k < j$. Since $\mathbf{X}_{\text{Pre}(j) \setminus i}$ is the condition set and $X_k \in \mathbf{X}_{\text{Pre}(j) \setminus i}$, $X_k$ must be a collider for the path to be d-connected. At the same time, $(X_i, X_k)$ must be a selection pair. We further discuss the following cases:
                \begin{itemize}
                    \item The case where $k = i+1$. According to assumption \ref{con:1}, $(X_i, X_k)$ cannot be a selection pair if $k = i+1$, which is a contradiction.
                    \item The case where $k \neq i+1$. On the path, let us denote the $X_p$ as the observed variable closest to $X_k$ other than $X_i$. Thus, we must have $\{X_k \leftarrow \cdots X_p\}$ on the path. We further have the following cases:
                     \begin{itemize}
                        \item The case where $p < k$. If $p < k$, we have $p < i$ as $p \neq i$. Since $\{X_k \leftarrow \cdots X_p\}$, $(X_p, X_k)$ is either a confounded pair or a direct relation. If $(X_p, X_k)$ is a confounded pair, because $X_p$ cannot be caused by another confounder (assumption \ref{con:2}), the only case for $X_p$ to also be a collider is $\{X_k \leftarrow C \rightarrow X_p \leftarrow X_{p-1}\}$. Thus, $X_{p-1}$ must not be a collider, and $X_i \indep X_j |\{\mathbf{X}_{\text{Pre}(j) \setminus i}\}$, which implies $(X_i, X_j)_\mathrm{d} \notin \mathbf{E}_1$; If $(X_p, X_k)$ is a direct relation, $X_{p}$ must not be a collider, and $X_i \indep X_j | \{\mathbf{X}_{\text{Pre}(j) \setminus i}\}$, which implies $(X_i, X_j)_\mathrm{d} \notin \mathbf{E}_1$.
                        \item The case where $p > k$. In this case, $(X_k, X_p)$ can only be a confounded pair. However, because $k \neq i+1$ and $\{X_i \cdots \leftrightarrow X_k \}$, $(X_i, X_k)$ is either a confounded pair or higher-order direct relation, both of which are impossible according to assumption \ref{con:2}. Thus, $(X_i, X_j)_\mathrm{d} \notin \mathbf{E}_1$.
                    \end{itemize}               
                \end{itemize}
            \item The case where $k > j$. If $k > j$ ($X_k \notin \mathbf{X}_{\text{Pre}(j) \setminus i}$), the only case where the path between $X_i$ and $X_j$ is not d-separated by $X_k$ is the case where $X_k$ is the cause of a selection variable. However, the index $k$ for any selection pair $\{X_{i}, X_k\}_\mathrm{s}$, where $k>j$, has already been included in $R[i]$ according to the order of tests in the algorithm (lines \ref{l3}-\ref{l10}). According to assumption \ref{con:1}, $X_k$ cannot be caused by more than two observed variables via higher-order direct relations. Therefore, the only case for $X_i$ to be conditionally dependent of $X_j$ given $\{\mathbf{X}_{\text{Pre}(j) \setminus i}, X_{j+1}, X_k\}$ for all $k \in \mathcal{R}[i]$ is the structure that $(X_i, X_{j+1})_{\{s,x\}}$. At the same time, we also have $X_i \notindependent X_j | \{\mathbf{X}_{\text{Pre}(j) \setminus i}\}$. As a result, we have $(X_i, X_j)_\mathrm{d} \in \mathbf{E}_1$ where $(X_i, X_{j+1})$ is both a selection pair and a direct relation.
        \end{enumerate}
    \end{enumerate}

    \textbf{Stage Two.} \ \ 
    Next, we prove that we can identify selection pairs, direct relations, and dependent pairs that are both selection pairs and direct relations in $\mathcal{E}_1$ after lines \ref{l11}-\ref{l29} in the algorithm. After this stage, there are still some spurious dependent pairs because of unidentified confounded pairs. Note that, according to the algorithm, for any $X_i$, the tests start from variables in $\mathcal{R}[i]$ with the largest index and continue in a descending order. In addition, $X_k$ for $k$ after lines \ref{l15} and \ref{l16} is a variable where $(X_i, X_k)$ is not a pair in $\mathcal{E}_1$. 
    
    We discuss the following cases for all types in $\mathcal{E}_1$.
        \begin{enumerate}[label=\Roman*:]
            \item $(X_i, X_j)$ is a selection pair but not a direct relation. According to the algorithm, it will be identified as a selection pair but not a direct relation if and only if $X_i \indep X_k | \{\mathbf{X}_{\text{Pre}(k) \setminus i}, X_j, \mathbf{T}\}$ (the definition of $\mathbf{T}$ is in line \ref{l17} in the algorithm).
            

            We first show that all these selection pairs will be identified. If $(X_i, X_j)$ is a selection pair but not a direct relation, according to assumption \ref{con:2}, if $(X_i, X_j)$ is a selection pair, it cannot be a confounded pair. Thus, $X_j$ must be a non-collider on the path, which may lead to the d-separation and thus the conditional independence. Thus, we need to show that there exists no other path that leads to conditional dependence. We prove it by contradiction. 
            
            Suppose that there exist a path other than $\{X_i \rightarrow S \leftarrow X_j \cdots X_k\}$ that makes $X_i \notindependent X_k | \{\mathbf{X}_{\text{Pre}(k) \setminus i}, X_j, \mathbf{T}\}$. We discuss all potential cases for that path:

            \begin{enumerate}[label=I.\roman*:]
                \item $\{X_i \cdots \leftarrow X_p \cdots X_k\}$, where $p<i$ and $X_p$ denotes the observed variable that is the closest to $X_i$ on the path. Since $X_p \in \mathbf{X}_{\text{Pre}(k) \setminus i}$, this path is d-connected if and only if $X_p$ is a collider on the path. However, $X_p$ cannot be a collider because of $\{X_i \leftarrow X_p\}$. Thus, this case is impossible.
                \item $\{X_i \cdots \rightarrow X_p \cdots X_k\}$, where $p<i$ and $X_p$ denotes the observed variable that is the closest to $X_i$ on the path. Because of assumption \ref{assum:t2a1}, $\{X_i \rightarrow X_p\}$ is impossible. Thus, there must be a confounded pair $\{X_p, X_i\}_\mathrm{c}$. Because of assumption \ref{con:2}, $X_p$ cannot be caused by another confounder or observed variable via higher-order direct relation. Therefore, $X_p$ is not a collider and this case is impossible
                \item $\{X_i \cdots \leftrightarrow X_p \cdots X_k\}$, where $p<i$ and $X_p$ denotes the observed variable that is the closest to $X_i$ on the path. Because of assumption \ref{assum:t2a1}, $\{X_p \rightarrow X_i\}$ is impossible. Thus, there must be a confounded pair $\{X_p, X_i\}_\mathrm{c}$. Because of assumption \ref{con:2}, $(X_p, X_i)$ cannot be a selection pair or direct relation. Thus, it is impossible to have $\{X_i \cdots \leftrightarrow X_p\}$ on the path. Consequently, this case is impossible.
                
           
                \item $\{X_i \cdots \leftarrow X_p \cdots X_k\}$, where $i < p < k$ and $X_p$ denotes the observed variable that is the closest to $X_i$ on the path. Because of assumption \ref{assum:t2a1}, $\{X_i \leftarrow X_p\}$ is impossible. Thus, there must be a selection pair $(X_i, X_p)_\mathrm{s}$, and $X_p$ must not be a collider. Since $\{\mathbf{X}_{\text{Pre}(k) \setminus i}, X_j, \mathbf{T}\}$ is the condition set and $X_p \in \mathbf{X}_{\text{Pre}(k) \setminus i}$, we have $X_i \indep X_k | \{\mathbf{X}_{\text{Pre}(k) \setminus i}, X_j, \mathbf{T}\}$ ($X_p \in \mathbf{X}_{\text{Pre}(k) \setminus i}$) and thus this case is impossible.

                \item $\{X_i \cdots \rightarrow X_p \cdots X_k\}$, where $i < p < k$. Since $\mathbf{X}_{\text{Pre}(k) \setminus i}$ is in the condition set and $X_p \in \mathbf{X}_{\text{Pre}(k) \setminus i}$, $X_p$ must be a collider for the path to be d-connected. On the path, let us denote the $X_q$ as the observed variable closest to $X_p$ other than $X_i$. Thus, we must have $\{X_p \leftarrow \cdots X_q\}$ on the path. We further have the following cases:
                    \begin{itemize}
                    \item The case where $q < p$. If $q < p$, we have $q < i$ as $q \neq i$. Since $\{X_p \leftarrow \cdots X_q\}$, $(X_q, X_p)$ is either a confounded pair or direct relation. We further discuss the following cases:
                    \begin{itemize}
                        \item If $(X_q, X_p)$ is a confounded pair, because $X_q$ cannot be caused by another confounder (assumption \ref{con:2}), the only case for $X_q$ to also be a collider is $\{X_p \leftarrow C \rightarrow X_q \leftarrow X_{q-1}$. Thus, $X_{q-1}$ must not be a collider, and $X_i \indep X_k | \{\mathbf{X}_{\text{Pre}(k) \setminus i}, X_j, \mathbf{T}\}$ ($X_{q-1} \in \mathbf{X}_{\text{Pre}(k) \setminus i}$) and thus this case is impossible.
                        \item If $(X_q, X_p)$ is a direct relation, $X_{q}$ must not be a collider, and $X_i \indep X_k | \{\mathbf{X}_{\text{Pre}(k) \setminus i}, X_j, \mathbf{T}\}$ ($X_q \in \mathbf{X}_{\text{Pre}(k) \setminus i}$) and thus this case is impossible.
                    \end{itemize}
                
                    \item The case where $q > p$. In this case, $(X_p, X_q)$ can only be a confounded pair. Because of assumption \ref{con:2}, $X_p$ and $X_q$ cannot be caused by another confounder or higher-order direct relation. Therefore, there must be $\{X_i \rightarrow X_p\}$ and $p = i+1$ for the path to be d-connected. We further discuss the following cases:
                    \begin{itemize}
                        \item If $p < q < k$, because of assumption \ref{con:2}, $X_q$ cannot be caused by another confounder or higher-order direct relation. If $X_q$ is a collider, $X_{q-1}$ must not be a collider on the path as $\{X_{q-1} \rightarrow X_q\}$, and $X_i \indep X_k | \{\mathbf{X}_{\text{Pre}(k) \setminus i}, X_j, \mathbf{T}\}$ ($X_{q-1} \in \mathbf{X}_{\text{Pre}(k) \setminus i}$) and thus this case is impossible.
                        \item If $q = k$, then we have $(X_{i+1}, X_{k})_\mathrm{c}$, which is a contradiction to the definition of $k$ ($k \in \mathcal{R}[i]$ according to the algorithm)
                        \item If $k < q < j$, because of assumption \ref{con:2}, $(X_{p-1}, X_{q+1})_{d\neq c}$ should not exist. However, since $p-1 = i$, this contradicts the definition of $k$ ($k \in \mathcal{R}[i]$ according to the algorithm) or the existing selection pair $(X_i, X_j)_\mathrm{s}$.
                        \item If $q = j$, this contradicts assumption \ref{con:2} that there exists no $(X_i, X_j)_\mathrm{s}$ ($i = p-1$).
                        \item If $q > j$, if not conditioning on $X_q$, the only case where the path between $X_i$ and $X_k$ is not d-separated by $X_q$ is the case where $X_q$ is the cause of a selection variable. However, the index $q$ has already been included in $R[i]$ according to the order of tests in the algorithm (Lines \ref{l12} and \ref{l13}). According to assumption \ref{con:1}, $X_q$ cannot be caused by more than two observed variables via higher-order direct relations. As a result, we have $X_i \indep X_k | \{\mathbf{X}_{\text{Pre}(k) \setminus i}, X_j, \mathbf{T}\}$ and thus this case is impossible.
                    \end{itemize}   
                \end{itemize}

                \item $\{X_i \cdots \leftrightarrow X_p \cdots X_k\}$, where $i < p < k$. Since $\mathbf{X}_{\text{Pre}(k) \setminus i}$ is in the condition set and $X_p \in \mathbf{X}_{\text{Pre}(k) \setminus i}$, $X_p$ must be a collider for the path to be d-connected. At the same time, $(X_i, X_p)$ must be a selection pair. We further discuss the following cases:
                \begin{itemize}
                    \item The case where $p = i+1$. According to assumption \ref{con:1}, $(X_i, X_p)$ cannot be a selection pair if $p = i+1$, which is a contradiction.
                    \item The case where $p \neq i+1$. On the path, let us denote the $X_q$ as the observed variable closest to $X_p$ other than $X_i$. Thus, we must have $\{X_p \leftarrow \cdots X_q\}$ on the path. We further have the following cases:
                     \begin{itemize}
                        \item The case where $q < p$. If $q < p$, we have $q < i$ as $q \neq i$. Since $\{X_p \leftarrow \cdots X_q\}$, $(X_q, X_p)$ is either a confounded pair or direct relation. If $(X_q, X_p)$ is a confounded pair, because $X_q$ cannot be caused by another confounder (assumption \ref{con:2}), the only case for $X_q$ to also be a collider is $\{X_p \leftarrow C \rightarrow X_q \leftarrow X_{q-1}\}$. Thus, $X_{q-1}$ must not be a collider, and the path is not d-connected; If $(X_q, X_p)$ is a direct relation, $X_{q}$ must not be a collider, and the path is, again, not d-connected. Thus, this case is impossible.
                        \item The case where $q > p$. In this case, $(X_p, X_q)$ can only be a confounded pair. However, because $p \neq i+1$ and $\{X_i \cdots \leftrightarrow X_p \}$, $(X_i, X_p)$ is either a confounded pair or higher-order direct relation, both of which are impossible according to assumption \ref{con:2}. Thus, this case is impossible.
                    \end{itemize}               
                \end{itemize}
                
                \item $\{X_i \cdots \leftarrow X_p\}$, where $k=p$. This is impossible because of the algorithm, specifically, how $k$ is selected (lines \ref{l15} and \ref{l16}).

                \item $\{X_i \cdots \rightarrow X_p\}$, where $k=p$. This is also impossible because of the algorithm.

                \item $\{X_i \cdots \leftrightarrow X_p\}$, where $k=p$. Similarly, this is impossible according to the algorithm.

                \item $\{X_i \cdots \leftarrow X_p \cdots X_k\}$, $k < p < j$. Because of assumption \ref{assum:t2a1}, $\{X_i \leftarrow X_p\}$ is impossible. Thus, there must be a selection pair $(X_i, X_p)_\mathrm{s}$. 
                Together with the search procedure (lines 20-25), it is thus impossible for the selected $k$ to be smaller than $p$, which is a contradiction.

                \item $\{X_i \cdots \rightarrow X_p \cdots X_k\}$, $k < p < j$. Because of $\{X_i \cdots \rightarrow X_p\}$, $(X_i, X_p)$ is either a higher-order direct relation or a confounded pair. Since $X_p$ is not in the condition set, $(X_i, X_p)$ must also be a selection pair for the path to be d-connected. However, according to assumption \ref{con:1}, there exits no $(X_i, X_{p+1})_\mathrm{d}$. This contradicts the search procedure to select $k$ (lines \ref{l15} and \ref{l16}), and thus this case is impossible.

                \item $\{X_i \cdots \leftrightarrow X_p \cdots X_k\}$, $k < p < j$. Because of assumption \ref{assum:t2a1}, $\{X_i \leftarrow X_p\}$ is impossible. Thus, there must be a selection pair $(X_i, X_p)_\mathrm{s}$. Similarly, it is thus impossible for the selected $k$ to be smaller than $p$, which is a contradiction.

                \item $\{X_i \cdots \leftarrow X_p \cdots X_k\}$, $j < p$. Because of assumption \ref{assum:t2a1}, $\{X_i \leftarrow X_p\}$ is impossible. Thus, there must be a selection pair $(X_i, X_p)_\mathrm{s}$. Moreover, according to assumption \ref{con:1}, $X_p$ cannot be the effect of multiple higher-order directions. Thus, $X_p$ must not be a collider on the path, so we can d-separate the path by conditioning on it. Since $X_p$ is already in $\mathcal{T}$ according to the algorithm, we have $X_i \indep X_k | \{\mathbf{X}_{\text{Pre}(k) \setminus i}, X_j, \mathbf{T}\}$ and thus this case is impossible.

                \item $\{X_i \cdots \rightarrow X_p \cdots X_k\}$, $j < p$. In this case, $(X_i, X_p)$ is either a higher-order direct relation or a confounded pair. We further discuss the following case:
                \begin{itemize}
                    \item $X_p$ is not a cause of a selection variable. In this case, $X_p$ must be a collider on the path, and the path is d-separated when not conditioning on $X_p$. So we have $X_i \indep X_k | \{\mathbf{X}_{\text{Pre}(k) \setminus i}, X_j, \mathbf{T}\}$, which indicates that this case is impossible.
                    \item $X_p$ is a cause of a selection variable. According to assumption \ref{con:1}, it cannot be the effect of multiple higher-order directions. Therefore, if $X_p$ is the cause of another selection variable together with any observed variable with index $u \in [k, j)$, the path is always d-separated by conditioning on both $X_p$ and $X_j$. At the same time, if $X_p$ is the cause of another selection variable together with any observed variable with an index larger than $j$, that variable will also be in the condition set ($\mathbf{T}$) according to the algorithm (line \ref{l17}). Recursively, all these variables will be in the condition set, and the path will be d-separated, which indicates that this case is impossible.               
                \end{itemize}

                \item $\{X_i \cdots \leftrightarrow X_p \cdots X_k\}$, $j < p$. Because of assumption \ref{assum:t2a1}, $\{X_i \leftarrow X_p\}$ is impossible. Thus, there must be a selection pair $(X_i, X_p)_\mathrm{s}$. At the same time, since $(X_i, X_p)$ cannot be both a selection pair and confounded pair, we have $(X_i, X_p)_{\mathrm{d} \neq \mathrm{c}}$, i.e., it is both a selection pair and a direct relation. According to assumption \ref{con:1}, it cannot be the effect of multiple higher-order directions. Therefore, if $X_p$ is the cause of another selection variable together with any observed variable with index $u \in [k, j)$, the path is always d-separated by conditioning on both $X_p$ and $X_j$. At the same time, if $X_p$ is the cause of another selection variable together with any observed variable with an index larger than $j$, that variable will also be in the condition set ($\mathbf{T}$) according to the algorithm (line \ref{l17}). Recursively, all these variables will be in the condition set, and the path will be d-separated, which indicates that this case is impossible.               
                
            \end{enumerate}
            Therefore, there does not exist any path other than $\{X_i \rightarrow S \leftarrow X_j \cdots X_k\}$ that makes $X_i \notindependent X_k | \{\mathbf{X}_{\text{Pre}(k) \setminus i}, X_j, \mathbf{T}\}$. So all selection pairs $(X_i, X_j)_\mathrm{s}$ are identified.
            
            We then discuss the other cases in $\mathcal{E}_1$ and show that it is impossible for any other types of variable pairs in $\mathcal{E}_1$ to have the test results that $X_i \indep X_k | \{\mathbf{X}_{\text{Pre}(k) \setminus i}, X_j, \mathbf{T}\}$.

            \begin{itemize}
                \item $(X_i, X_j)$ is a direct relation but not a selection pair. In this case, since $X_j$ is collider on the path $\{X_i \cdots \rightarrow X_j\}$, for any $v \in \mathbf{T}$, there is always $X_i \notindependent X_k | \{\mathbf{X}_{\text{Pre}(k) \setminus i}, X_j, \mathbf{T}\}$. Thus, this case is impossible.
                \item $(X_i, X_j)$ is both a selection pair and a direct relation. In this case, since $X_j$ is still a collider on the path $\{X_i \cdots \rightarrow X_j\}$, for any $v \in \mathbf{T}$, there is always $X_i \notindependent X_k | \{\mathbf{X}_{\text{Pre}(k) \setminus i}, X_j, \mathbf{T}\}$, which indicates that this case is impossible.
                \item $(X_i, X_j)$ is a confounded pair. In this case, $X_j$ is a collider on the path (all confounders are latent and thus will never be conditioned). Similar to the previous two cases, for any $v \in \mathbf{T}$, there is always $X_i \notindependent X_k | \{\mathbf{X}_{\text{Pre}(k) \setminus i}, X_j, \mathbf{T}\}$, and this case is impossible.
                \item $(X_i, X_j)$ is a spurious dependent pair that there exists a confounded pair $(X_{i+1}, X_j)_\mathrm{c}$ in the ground-truth causal graph. In this case, $X_{i+1}$ is a collider on the path, and thus conditioning on it does not d-separate the path. At the same time, $X_j$ is also a collider. Since both $X_{i+1}$ and $X_j$ are in the condition set, there is always $X_i \notindependent X_k | \{\mathbf{X}_{\text{Pre}(k) \setminus i}, X_j, \mathbf{T}\}$, which indicates that this case is impossible.
                \item $(X_i, X_j)$ is a spurious dependent pair that there exists $(X_i, X_{j+1})$ that is both a selection pair and a direct relation in the ground-truth causal graph. In this case, $(X_i, X_j)$ is not a dependent pair according to assumption \ref{con:1}, and will be removed by the algorithm (line \ref{l24}). Therefore, this case is impossible.
            \end{itemize}
       
            \item $(X_i, X_j)$ is a direct relation but not a selection pair. We first show that all these variable pairs will be identified. In this case, $X_j$ is a collider on the path. Thus, we have $X_i \notindependent X_k | \{\mathbf{X}_{\text{Pre}(k) \setminus i}, X_j, \mathbf{T}\}$. However, when $X_j$ is not in the condition set, we may have the d-separation and thus the conditional independence $X_i \indep X_k | \{\mathbf{X}_{\text{Pre}(k) \setminus i}, X_{j+1}, \mathbf{T}\}$ (the definition of $\mathbf{T}$ is in line \ref{l17} in the algorithm). Thus, we need to show that there exists no other path that leads to conditional dependence. We prove it by contradiction. 

            Suppose that there exist a path other than $\{X_i \rightarrow  X_j \cdots X_k\}$ that makes  $X_i \notindependent X_k | \{\mathbf{X}_{\text{Pre}(k) \setminus i}, X_{j+1}, \mathbf{T}\}$. Denote $X_p$ as the observed variable that is the closest to $X_i$ on the path, we discuss all potential cases for that path:

            \begin{enumerate}[label=II.\roman*:]
                \item The discussions for the following cases are similar to those in the cases where $(X_i, X_j)$ is a selection pair but not a direct relation:
                    \begin{itemize}
                        \item $\{X_i \cdots \leftarrow X_p \cdots X_k\}$, where $p<i$, $i < p < k$, or $k < p < j$.
                        \item $\{X_i \cdots \rightarrow X_p \cdots X_k\}$, where $p<i$ or $k < p < j$ .
                        \item $\{X_i \cdots \leftrightarrow X_p \cdots X_k\}$, where $p<i$, $i < p < k$, or $k < p < j$.
                        \item $\{X_i \cdot X_p\}$, where $p = k$. 
                    \end{itemize}
                    Thus, we mainly focus on the other cases as follows.

                \item $\{X_i \cdots \rightarrow X_p \cdots X_k\}$, where $i < p < k$. Since $\mathbf{X}_{\text{Pre}(k) \setminus i}$ is in the condition set and $X_p \in \mathbf{X}_{\text{Pre}(k) \setminus i}$, $X_p$ must be a collider for the path to be d-connected. On the path, let us denote the $X_q$ as the observed variable closest to $X_p$ other than $X_i$. Thus, we must have $\{X_p \leftarrow \cdots X_q\}$ on the path. We further have the following cases:
                    \begin{itemize}
                    \item The case where $q < p$. If $q < p$, we have $q < i$ as $q \neq i$. Since $\{X_p \leftarrow \cdots X_q\}$, $(X_p, X_q)$ is either a confounded pair or direct relation. We further discuss the following cases:
                    \begin{itemize}
                        \item If $(X_p, X_q)$ is a confounded pair, because $X_q$ cannot be caused by another confounder (assumption \ref{con:2}), the only case for $X_q$ to also be a collider is $\{X_p \leftarrow C \rightarrow X_q \leftarrow X_{q-1}$. Thus, $X_{q-1}$ must not be a collider, and $X_i \indep X_k | \{\mathbf{X}_{\text{Pre}(k) \setminus i}, \mathbf{T}\}$ ($X_{q-1} \in \mathbf{X}_{\text{Pre}(k) \setminus i}$) and thus this case is impossible.
                        \item If $(X_p, X_q)$ is a direct relation, $X_{q}$ must not be a collider, and $X_i \indep X_k | \{\mathbf{X}_{\text{Pre}(k) \setminus i}, \mathbf{T}\}$ ($X_{q-1} \in \mathbf{X}_{\text{Pre}(k) \setminus i}$) and thus this case is impossible.
                    \end{itemize}
                
                    \item The case where $q > p$. In this case, $(X_p, X_q)$ can only be a confounded pair. Because of assumption \ref{con:2}, $X_p$ and $X_q$ cannot be caused by another confounder or higher-order direct relation. Therefore, there must be $\{X_i \rightarrow X_p\}$ and $p = i+1$ for the path to be d-connected. We further discuss the following cases:
                    \begin{itemize}
                        \item If $p < q < k$, because of assumption \ref{con:2}, $X_q$ cannot be caused by another confounder or higher-order direct relation. If $X_q$ is a collider, $X_{q-1}$ must not be a collider on the path as $\{X_{q-1} \rightarrow X_q\}$, and $X_i \indep X_k | \{\mathbf{X}_{\text{Pre}(k) \setminus i}, \mathbf{T}\}$ ($X_{q-1} \in \mathbf{X}_{\text{Pre}(k) \setminus i}$) and thus this case is impossible.
                        \item If $q = k$, then we have $(X_{i+1}, X_{k})_\mathrm{c}$, which is a contradiction to the definition of $k$ ($k \in \mathcal{R}[i]$ according to the algorithm)
                        \item If $k < q < j$, because of assumption \ref{con:2}, $(X_{p-1}, X_{q+1})_{d\neq c}$ should not exist. However, since $p-1 = i$, this contradicts the definition of $k$ ($k \in \mathcal{R}[i]$ according to the algorithm) or the existing selection pair $(X_i, X_j)_\mathrm{s}$.
                        \item If $q = j$, this contradicts assumption \ref{con:2} that there exists no $(X_i, X_j)_\mathrm{s}$ ($i = p-1$).
                        \item If $q > j$, if not conditioning on $X_q$, the only case where the path between $X_i$ and $X_k$ is not d-separated by $X_q$ is the case where $X_q$ is the cause of a selection variable. However, the index $q$ has already been included in $R[i]$ according to the order of tests in the algorithm (lines \ref{l12} and \ref{l13}). According to assumption \ref{con:1}, $X_q$ cannot be caused by more than two observed variables via higher-order direct relations. At the same time, it is impossible to have $q=j+1$, since there should not exist $(X_i, X_j)_\mathrm{d}$ if $(X_{i+1}, X_{j+1})$ is a confounded pair according to assumption \ref{assum:t2a2}. Therefore, the path is d-separated by conditioning on $\{\mathbf{X}_{\text{Pre}(k) \setminus i}, X_j, \mathbf{T}\}$. As a result, we have $X_i \indep X_k | \{\mathbf{X}_{\text{Pre}(k) \setminus i}, \mathbf{T}\}$ and thus this case is impossible.
                    \end{itemize}   
                \end{itemize}

                \item $\{X_i \cdots \leftarrow X_p \cdots X_k\}$, $j < p$. Because of assumption \ref{con:1}, $\{X_i \leftarrow X_p\}$ is impossible. Thus, there must be a selection pair $(X_i, X_p)_\mathrm{s}$. Moreover, according to assumption \ref{con:1}, $X_p$ cannot be the effect of multiple higher-order directions. Thus, $X_p$ must not be a collider on the path, so we can d-separate the path by conditioning on it. Since $X_p$ is already in $\mathbf{T}$ according to the algorithm, we have $X_i \indep X_k | \{\mathbf{X}_{\text{Pre}(k) \setminus i}, X_{j+1}, \mathbf{T}\}$ and thus this case is impossible.

                \item $\{X_i \cdots \rightarrow X_p \cdots X_k\}$, $j < p$. In this case, $(X_i, X_p)$ is either a higher-order direct relation or a confounded pair. We further discuss the following case:
                \begin{itemize}
                    \item $X_p$ is not a cause of a selection variable. In this case, $X_p$ must be a collider on the path, and the path is d-separated when not conditioning on $X_p$. So we have $X_i \indep X_k | \{\mathbf{X}_{\text{Pre}(k) \setminus i}, X_{j+1}, \mathbf{T}\}$, which indicates that this case is impossible.
                    \item $X_p$ is a cause of a selection variable. According to assumption \ref{con:1}, it cannot be the effect of multiple higher-order directions. Meanwhile, the index $p$ has already been included in $R[i]$ according to the order of tests in the algorithm (lines \ref{l12} and \ref{l13}). At the same time, it is impossible for $p$ to be $j+1$ according to assumption \ref{con:1}. This is because if $X_{j+1}$ is a cause of a selection variable while $(X_i, X_{j+1})$ is either a higher-order direct relation or a confounded pair, it is impossible to have $(X_i, X_j)_\mathrm{d}$. Therefore, the path is always d-separated by conditioning on $\{\mathbf{X}_{\text{Pre}(k) \setminus i}, X_{j+1}, \mathbf{T}\}$ and thus this case is impossible.            
                \end{itemize}

                \item $\{X_i \cdots \leftrightarrow X_p \cdots X_k\}$, $j < p$. Because of assumption \ref{assum:t2a1}, $\{X_i \leftarrow X_p\}$ is impossible. Thus, there must be a selection pair $(X_i, X_p)_\mathrm{s}$. At the same time, since $(X_i, X_p)$ cannot be both a selection pair and confounded pair, we have $(X_i, X_p)_{\mathrm{d} \neq \mathrm{c}}$, i.e., it is both a selection pair and a direct relation. According to assumption \ref{con:1}, it cannot be the effect of multiple higher-order directions. Meanwhile, it is impossible for $p$ to be $j+1$ according to assumption \ref{con:1}. This is because if $X_{j+1}$ is a cause of a selection variable while $(X_i, X_{j+1})$ is either a higher-order direct relation or a confounded pair, it is impossible to have $(X_i, X_j)_\mathrm{d}$. Therefore, the path is always d-separated by conditioning on $\{\mathbf{X}_{\text{Pre}(k) \setminus i}, X_{j+1}, \mathbf{T}\}$ and thus this case is impossible.            
            \end{enumerate}


            
        We then discuss the other cases in $\mathcal{E}_1$ and show that, except for those spurious direct relations corresponding to unidentified confounded pairs, it is impossible for any other types of variable pairs in $\mathcal{E}_1$ to be identified as a direct relation.

        \begin{itemize}
            \item $(X_i, X_j)$ is a selection pair but not a directed relation. As already shown before, it will always be identified as a selection pair but not a direct relation. Thus, this case is impossible
            
            \item $(X_i, X_j)$ is both a selection pair and a direct relation. In this case, since $X_j$ is still a collider on the path $\{X_i \cdots \rightarrow X_k\}$, there is always $X_i \notindependent X_k | \{\mathbf{X}_{\text{Pre}(k) \setminus i}, X_j, \mathbf{T}\}$. At the same time, since $(X_i, X_j)$ is also a selection pair, not conditioning on $X_j$ will also open the path. Thus, we also have $X_i \notindependent X_k | \{\mathbf{X}_{\text{Pre}(k) \setminus i}, X_{j+1}, \mathbf{T}\}$, which implies that this variable pair will not just be identified as a direct relation according to the algorithm (lines \ref{l19}-\ref{l22}).
                            
            \item $(X_i, X_j)$ is a confounded pair. In this case, $X_j$ is always a collider on the path, and thus it is similar to the case where $(X_i, X_j)$ is a direct relation. According to the discussion above, it will also be identified as a direct relation, and thus there will be spurious direct relations corresponding to unidentified confounded pairs.
            
            \item $(X_i, X_j)$ is a spurious dependent pair that there exists a confounded pair $(X_{i+1}, X_j)_\mathrm{c}$ in the ground-truth causal graph. In this case, $X_{i+1}$ is a collider on the path, and thus conditioning on it does not d-separate the path. At the same time, $X_j$ is also a collider. Since $X_{i+1}$ is always in the condition set, $X_j$ is always a collider on the path. Thus, similar to the previous discussion, $(X_i, X_j)$ will be identified as a spurious direct relation corresponding to the unidentified confounded pair.

            \item $(X_i, X_j)$ is a spurious dependent pair that there exists $(X_i, X_{j+1})$ that is both a selection pair and a direct relation in the ground-truth causal graph. In this case, $(X_i, X_j)$ is not a dependent pair according to assumption \ref{con:1}, and will be removed by the algorithm (line \ref{l24}). Therefore, this case is impossible.
        \end{itemize}

        \item $(X_i, X_j)$ is both a selection pair and a direct relation. In this case, since $X_j$ is still a collider on the path $\{X_i \cdots \rightarrow X_k\}$, there is always $X_i \notindependent X_k | \{\mathbf{X}_{\text{Pre}(k) \setminus i}, X_j, \mathbf{T}\}$. At the same time, since $(X_i, X_j)$ is also a selection pair, not conditioning on $X_j$ will also open the path. Thus, we also have $X_i \notindependent X_k | \{\mathbf{X}_{\text{Pre}(k) \setminus i}, X_{j+1}, \mathbf{T}\}$. We then discuss the other cases in $\mathcal{E}_1$ and show that, except for those variable pairs that are both a selection pair and a direct relation, it is impossible for any other types of variable pairs in $\mathcal{E}_1$ to be identified as it.

        \begin{itemize}
            \item $(X_i, X_j)$ is a selection pair but not a directed relation. As already shown before, it will always be identified as a selection pair but not a direct relation. Thus, this case is impossible
            
            \item $(X_i, X_j)$ is a direct relation but not a selection pair. In this case, according to the discussion before, there is always the conditional independence $X_i \notindependent X_k | \{\mathbf{X}_{\text{Pre}(k) \setminus i}, X_{j+1}, \mathbf{T}\}$. Thus, it is impossible for this case to be identified as a variable pair i.e. both a selection pair and a direct relation.
                            
            \item $(X_i, X_j)$ is a confounded pair. In this case, $X_j$ is always a collider on the path, and thus it is similar to the case where $(X_i, X_j)$ is a direct relation. According to the discussion above, there is always the conditional independence $X_i \notindependent X_k | \{\mathbf{X}_{\text{Pre}(k) \setminus i}, X_{j+1}, \mathbf{T}\}$. Thus, this case is impossible
            
            \item $(X_i, X_j)$ is a spurious dependent pair that there exists a confounded pair $(X_{i+1}, X_j)_\mathrm{c}$ in the ground-truth causal graph. In this case, $X_{i+1}$ is a collider on the path, and thus conditioning on it does not d-separate the path. At the same time, $X_j$ is also a collider. Since $X_{i+1}$ is always in the condition set, $X_j$ is always a collider on the path. Thus, similar to the previous discussion, $(X_i, X_j)$ will be identified as a spurious direct relation corresponding to the unidentified confounded pair.

            \item $(X_i, X_j)$ is a spurious dependent pair that there exists $(X_i, X_{j+1})$ that is both a selection pair and a direct relation in the ground-truth causal graph. In this case, $(X_i, X_j)$ is not a dependent pair according to assumption \ref{con:1}, and will be removed by the algorithm (line \ref{l24}). Therefore, this case is impossible.
        \end{itemize}

        \item $(X_i, X_j)$ is a confounded pair. In this case, $X_j$ is always a collider on the path, and $X_i$ is always a non-collider on the path i.e in the condition set for all tests in the algorithm, Thus, it is similar to the case where $(X_i, X_j)$ is a direct relation. According to the discussion above, it will also be identified as a direct relation, and thus there will be spurious direct relations corresponding to unidentified confounded pairs.
        
        \item $(X_i, X_j)$ is a spurious dependent pair that there exists a confounded pair $(X_{i+1}, X_j)_\mathrm{c}$ in the ground-truth causal graph. In this case, $X_{i+1}$ is a collider on the path, and thus conditioning on it does not d-separate the path. At the same time, $X_j$ is also a collider. Since $X_{i+1}$ is always in the condition set, $X_j$ is always a collider on the path. Thus, similar to the previous discussion, $(X_i, X_j)$ will be identified as a spurious direct relation corresponding to the unidentified confounded pair.

        \item $(X_i, X_j)$ is a spurious dependent pair that there exists $(X_i, X_{j+1})$ that is both a selection pair and a direct relation in the ground-truth causal graph. In this case, $(X_i, X_j)$ is not a dependent pair according to assumption \ref{con:1}, and will be removed by the algorithm (line \ref{l24}).
        \end{enumerate}

    Therefore, after Stage Two, we can identify selection pairs, dependent pairs that are both selection pairs and direct relations, and direct relations, while maintaining some spurious direct relations corresponding to unidentified confounded pairs.
    
    \textbf{Stage Three.} \ \ 
    Finally, we need to identify confounded pairs and remove the spurious direct relations caused by the previously unidentified confounded pairs. After the second stage, the only cases where $(X_i, X_j)$ can be a confounded pair is where we have identified both $(X_{i-1}, X_j)$ and $(X_i, X_j)$ as direct relations. Based on the test order in the algorithm (lines \ref{l31} and \ref{l32}), if $(X_i, X_j)$ is a spurious dependent pair caused by the confounded pair $(X_i, X_{j+1}$), it would be removed before the test for itself. Thus, we discuss the following potential cases:
    \begin{enumerate}[label=\Roman*:]
        \item $(X_{i-1}, X_j)$ is a direct relation in the ground-truth, and $(X_i, X_j)$ is a confounded pair in the ground-truth. This case is impossible according to assumption \ref{con:2}.
        \item $(X_{i-1}, X_j)$ is a confounded pair in the ground-truth, and $(X_i, X_j)$ is a direct relation in the ground-truth. This case is impossible according to assumption \ref{con:2}.
        \item Both $(X_{i-1}, X_j)$ and $(X_i, X_j)$ are direct relations in the ground-truth. In this case, we always have $X_{i-1}$ dependent of $X_j$ given any condition set. Thus, these two direct relations will be identified by the algorithm, 
        \item Both $(X_{i-1}, X_j)$ and $(X_i, X_j)$ are confounded pairs in the ground-truth. This case is impossible according to assumption \ref{con:2}.
        \item $(X_{i-1}, X_j)$ is not a dependent pair in the ground-truth, and $(X_i, X_j)$ is a confounded pair in the ground-truth. In this case, $X_i$ is a collider on the path between $X_{i-1}$ and $X_j$. So if there are no other d-connected paths between them, we will have $X_{i-1} \indep X_j | \{\mathbf{X}_{\text{Pre}(j) \setminus \{i-1,i\}}, X_{j+1}, \mathbf{T}\}$. 
        
        Suppose there exists another path other than $\{X_{i-1} \rightarrow X_i \leftarrow C \rightarrow X_j\}$ that makes $X_{i-1} \notindependent X_j | \{\mathbf{X}_{\text{Pre}(j) \setminus \{i-1,i\}}\}$ or $X_{i-1} \notindependent X_j | \{\mathbf{X}_{\text{Pre}(j) \setminus \{i-1,i\}}, \mathbf{T}\}$. Denote $X_p$ as the observed variable that is the closest to $X_{i-1}$ on the path, we discuss all potential cases for that path:
            \begin{enumerate}[label=V.\roman*:]
                \item $\{X_{i-1} \cdots \leftarrow X_p \cdots X_j\}$, where $p<i-1$. In this case, it is impossible for $X_p$ to be a collider on the path. Since $X_p \in \mathbf{X}_{\text{Pre}(j) \setminus \{i-1,i\}}$, there is $X_{i-1} \indep X_j | \{\mathbf{X}_{\text{Pre}(j) \setminus \{i-1,i\}}, X_{j+1}, \mathbf{T}\}$. Thus, this case is impossible.
                
                \item $\{X_{i-1} \cdots \rightarrow X_p \cdots X_j\}$, where $p<i-1$. In this case, according to assumption \ref{assum:t2a1}, it is impossible to have $\{X_{i-1} \rightarrow X_p\}$. Thus, there must be a confounded pair $(X_p, X_{i-1})_\mathrm{c}$. According to assumption \ref{con:2}, $X_p$ cannot be caused by another latent confounder or observed variable via higher-order direct relation. Therefore, if this path is d-connected, we must have $X_{p-1} \rightarrow X_p$ on the path. However, if this is the case, $X_{p-1}$ cannot be the collider on the path. Because both $X_{p}$ and $X_{p-1}$ are in the condition set, we have $X_{i-1} \indep X_j | \{\mathbf{X}_{\text{Pre}(j) \setminus \{i-1,i\}}, X_{j+1}, \mathbf{T}\}$. Thus, this case is impossible.

                \item $\{X_{i-1} \cdots \leftrightarrow X_p \cdots X_j\}$, where $p<i-1$. In this case, according to assumption \ref{assum:t2a1}, it is impossible to have $\{X_{i-1} \rightarrow X_p\}$. Thus, there must be a confounded pair $(X_p, X_{i-1})_\mathrm{c}$. Then, because of assumption \ref{con:2}, $(X_p, X_{i-1})$ cannot be a selection pair or direct relation. Thus, it is impossible to have $\{X_{i-1} \cdots \leftrightarrow X_p\}$ on the path. Consequently, this case is impossible.

                \item $\{X_{i-1} \cdots \leftarrow X_p \cdots X_j\}$, where $i<p<j$. Because of assumption \ref{assum:t2a1}, $\{X_{i-1} \leftarrow X_p\}$ is impossible. Thus, there must be a selection pair $(X_{i-1}, X_p)_\mathrm{s}$, and $X_p$ must not be a collider. Since $\{\mathbf{X}_{\text{Pre}(j) \setminus \{i-1,i\}}\}$ is the condition set and $X_p \in \mathbf{X}_{\text{Pre}(k) \setminus \{i-1,i\}}$, we have $X_{i-1} \indep X_j | \{\mathbf{X}_{\text{Pre}(j) \setminus \{i-1,i\}}, X_{j+1}, \mathbf{T}\}$ and thus this case is impossible.
                
                \item $\{X_{i-1} \cdots \rightarrow X_p \cdots X_j\}$, where $i<p<j$. Since $\mathbf{X}_{\text{Pre}(j) \setminus \{i-1,i\}}$ is in the condition set and $X_p \in \mathbf{X}_{\text{Pre}(j) \setminus \{i-1,i\}}$, $X_p$ must be a collider on the path for the path to be d-connected. In addition, because of assumption \ref{con:2}, there must be $p \neq i+1$. On the path, let us denote the $X_q$ as the observed variable closest to $X_p$ other than $X_i$. Thus, we must have $\{X_p \leftarrow \cdots X_q\}$ on the path. We further have the following cases:
                    \begin{itemize}
                    \item The case where $q < p$. Since $\{X_p \leftarrow \cdots X_q\}$, $(X_q, X_p)$ is either a confounded pair or direct relation. Since $X_i$ is already involved in a confounded pair, we have $q \neq i$ according to assumption \ref{con:2}. Since there is also $X_q \neq i-1$, we have $q < i-1$. We further discuss the following cases:
                    \begin{itemize}
                        \item If $(X_q, X_p)$ is a confounded pair, because $X_q$ cannot be caused by another confounder (assumption \ref{con:2}), the only case for $X_q$ to also be a collider is $\{X_p \leftarrow C \rightarrow X_q \leftarrow X_{q-1}$. Thus, $X_{q-1}$ must not be a collider on the path, and $X_{i-1} \indep X_j | \{\mathbf{X}_{\text{Pre}(j) \setminus \{i-1,i\}}, X_{j+1}, \mathbf{T}\}$ ($X_{q-1} \in \mathbf{X}_{\text{Pre}(j) \setminus \{i-1,i\}}$) and thus this case is impossible.
                        \item If $(X_q, X_p)$ is a direct relation, $X_{q}$ must not be a collider on the path, and $X_{i-1} \indep X_j | \{\mathbf{X}_{\text{Pre}(j) \setminus \{i-1,i\}}, X_{j+1}, \mathbf{T}\}$ ($X_{q} \in \mathbf{X}_{\text{Pre}(j) \setminus \{i-1,i\}}$) and thus this case is impossible.
                    \end{itemize}
                
                    \item The case where $q > p$. In this case, $(X_p, X_q)$ can only be a confounded pair. Because of assumption \ref{con:2}, $X_p$ and $X_q$ cannot be in another confounded pair or higher-order direct relation. Since $p > i+1$, it is impossible to have $(X_{i-1}, X_p)_\mathrm{d}$ or $(X_{i}, X_p)_\mathrm{d}$. Thus $X_p$ cannot be a collider on the path, which implies $X_{i-1} \indep X_j | \{\mathbf{X}_{\text{Pre}(j) \setminus \{i-1,i\}}, X_{j+1}, \mathbf{T}\}$ ($X_{p} \in \mathbf{X}_{\text{Pre}(j) \setminus \{i-1,i\}}$). As a result, this case is impossible.
                \end{itemize}

                \item $\{X_{i-1} \cdots \leftrightarrow X_p \cdots X_j\}$, where $i<p<j$. Since $\mathbf{X}_{\text{Pre}(j) \setminus \{i-1,i\}}$ is in the condition set and $X_p \in \mathbf{X}_{\text{Pre}(j) \setminus \{i-1,i\}}$, $X_p$ must be a collider on the path for the path to be d-connected. In addition, because of assumption \ref{con:2}, there must be $p \neq i+1$. So we have $p > i+1$. Because there is $\{X_{i-1} \cdots \leftrightarrow X_p\}$, $(X_{i-1}, X_p)$ must be a selection pair. As a result, $(X_{i-1}, X_p)$ cannot be a confounded pair and thus must also be a direct relation. According to assumption \ref{con:1}, $X_p$ cannot be the effect of more than one higher-order direct relation. Therefore, another direct relation can only be $(X_{p-1}, X_{p})_\mathrm{d}$. Then $X_{p-1}$ must not be a collider on the path (note that $p-1 \neq i$), and thus this case is impossible.

                \item $\{X_{i-1} \cdots \leftarrow X_p \cdots X_j\}$, where $j<p$. Because of assumption \ref{assum:t2a1}, $\{X_{i-1} \leftarrow X_p\}$ is impossible. Thus, there must be a selection pair $(X_{i-1}, X_p)_\mathrm{s}$. Moreover, according to assumption \ref{con:1}, $X_p$ cannot be the effect of more than one higher-order direct relation. Thus, $X_p$ must not be a collider on the path, so we can d-separate the path by conditioning on it. Since $X_p$ is already in $\mathcal{T}$ according to the algorithm, we have $X_{i-1} \indep X_j | \{\mathbf{X}_{\text{Pre}(j) \setminus \{i-1,i\}},\mathbf{T}\}$ and thus this case is impossible.
                
                \item $\{X_{i-1} \cdots \rightarrow X_p \cdots X_j\}$, where $j<p$. In this case, $(X_{i-1}, X_p)$ is either a higher-order direct relation or a confounded pair. We further discuss the following case:
                \begin{itemize}
                    \item $X_p$ is not a cause of a selection variable. In this case, $X_p$ must be a collider on the path, and the path is d-separated when not conditioning on $X_p$. So we have $X_{i-1} \indep X_j | \{\mathbf{X}_{\text{Pre}(j) \setminus \{i-1,i\}},\mathbf{T}\}$, which indicates that this case is impossible.
                    \item $X_p$ is a cause of a selection variable. According to assumption \ref{con:1}, it cannot be the effect of multiple higher-order directions. Therefore, if $X_p$ is the cause of another selection variable, the path will be d-separated by conditioning on both $X_p$ and $X_{j+1}$. Note that conditioning on $X_{j+1}$ does not make the path d-connected, since $X_{j+1}$ cannot be d-connected with $X_{i-1}$, $X_i$, or $X_{i+1}$ via a higher-order direct relation or a confounded pair according to assumption \ref{con:2}. Thus, there is always $X_{i-1} \indep X_j | \{\mathbf{X}_{\text{Pre}(j) \setminus \{i-1,i\}},\mathbf{T}\}$, which implies that this case is impossible.            
                \end{itemize}
                
                \item $\{X_{i-1} \cdots \leftrightarrow X_p \cdots X_j\}$, where $j<p$. Because of assumption \ref{assum:t2a1}, $\{X_{i-1} \leftarrow X_p\}$ is impossible. Thus, there must be a selection pair $(X_{i-1}, X_p)_\mathrm{s}$. At the same time, since $(X_{i-1}, X_p)$ cannot be both a selection pair and confounded pair, we have $(X_{i-1}, X_p)_{\mathrm{d} \neq \mathrm{c}}$, i.e., it is both a selection pair and a direct relation. According to assumption \ref{con:1}, it cannot be the effect of multiple higher-order directions. Therefore, if $X_p$ is the cause of another selection variable, the path will be d-separated by conditioning on both $X_p$ and $X_{j+1}$. Note that conditioning on $X_{j+1}$ does not make the path d-connected, since $X_{j+1}$ cannot be d-connected with $X_{i-1}$, $X_i$, or $X_{i+1}$ via a higher-order direct relation or a confounded pair according to assumption \ref{con:2}. Thus, there is always $X_{i-1} \indep X_j | \{\mathbf{X}_{\text{Pre}(j) \setminus \{i-1,i\}},\mathbf{T}\}$, which implies that this case is impossible.

                \item $\{X_{i-1} \cdots \leftarrow X_p \cdots X_j\}$, where $p=i$. In this case, according to assumption \ref{assum:t2a1}, it is impossible to have $\{X_{i-1} \rightarrow X_p\}$. Thus, there must be a confounded pair $(X_p, X_{i-1})_\mathrm{c}$. However, since $p=i$ and $(X_i, X_j)_\mathrm{c}$, $X_p$ cannot be caused by another latent confounder. Thus, this case is impossible.
                
                \item $\{X_{i-1} \cdots \rightarrow X_p \cdots X_j\}$, where $p=i$. In this case, let us denote $X_q$ as the closest observed variable to $X_{i}$ ($X_p$) on the path other than $X_{i-1}$. Since both $X_{i-1}$ and $X_i$ are not in the condition set, the discussion on the potential cases for $X_q$ is similar to that for $X_p$, except for the case where we have $\{X_{i-1} \rightarrow X_i \rightarrow X_{q} \cdots X_j \}$, where $q = i+1$.

                If $q = i+1$, the path is $\{X_{i-1} \rightarrow X_i \rightarrow X_{i+1} \cdots X_j \}$. On this path, if $X_{i+1}$ is caused by a latent confounder, this path can be d-connected even with $X_{i+1}$ in the condition set. However, since $\mathbf{T}$ consists of all variables where $(X_{i-1},X_v)_\mathrm{s}$, $(X_{i},X_v)_\mathrm{s}$, or $(X_{i+1},X_v)_\mathrm{s}$ for $v>j$, following the previous discussions, we still have $X_{i-1} \indep X_j | \{\mathbf{X}_{\text{Pre}(j) \setminus \{i-1,i\}},\mathbf{T}\}$. Thus, this case is impossible.

                \item $\{X_{i-1} \cdots \leftrightarrow X_p \cdots X_j\}$, where $p=i$. In this case, according to assumption \ref{assum:t2a1}, it is impossible to have $\{X_{i-1} \rightarrow X_p\}$. Thus, there must be a confounded pair $(X_p, X_{i-1})_\mathrm{c}$. However, since $p=i$ and $(X_i, X_j)_\mathrm{c}$, $X_p$ cannot be caused by another latent confounder. Thus, this case is impossible.
                
            \end{enumerate}
        
    \end{enumerate}

    Therefore, after Stage Three, we identify all confounded pairs and remove spurious direct relations caused by previously unidentified confounded pairs. As a result, we have shown that all selection pairs, direct relations, and confounded pairs in the ground-truth causal graph are identifiable. 
\end{proof}
\newpage

\section{Additional Experimental Results}

\begin{figure}[ht]
    \centering
    \begin{subfigure}{0.47\linewidth}
        \includegraphics[width=\linewidth]{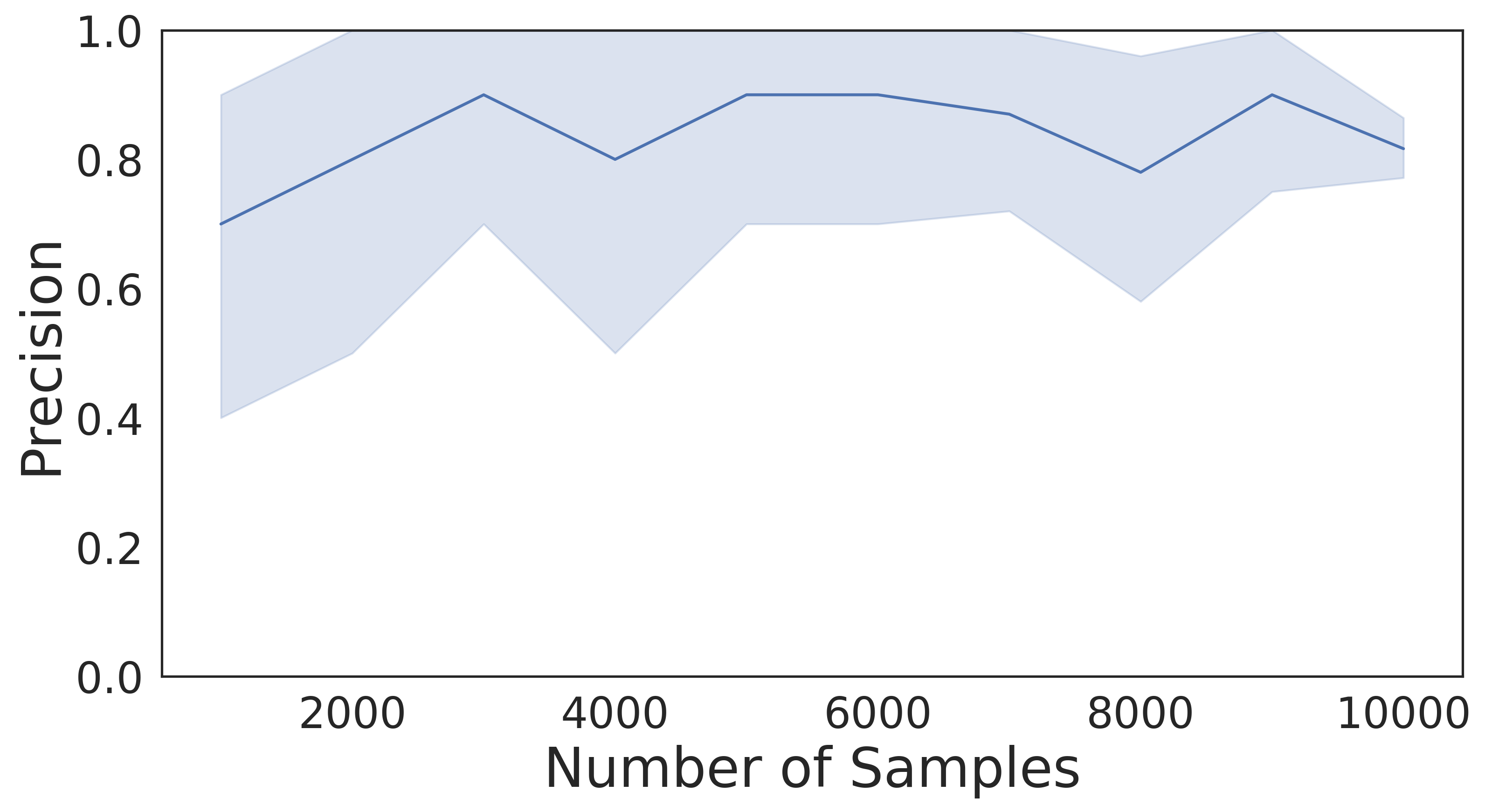}
        \label{fig:sample_precision}
    \end{subfigure}
    \hfill
    \begin{subfigure}{0.47\linewidth}
        \includegraphics[width=\linewidth]{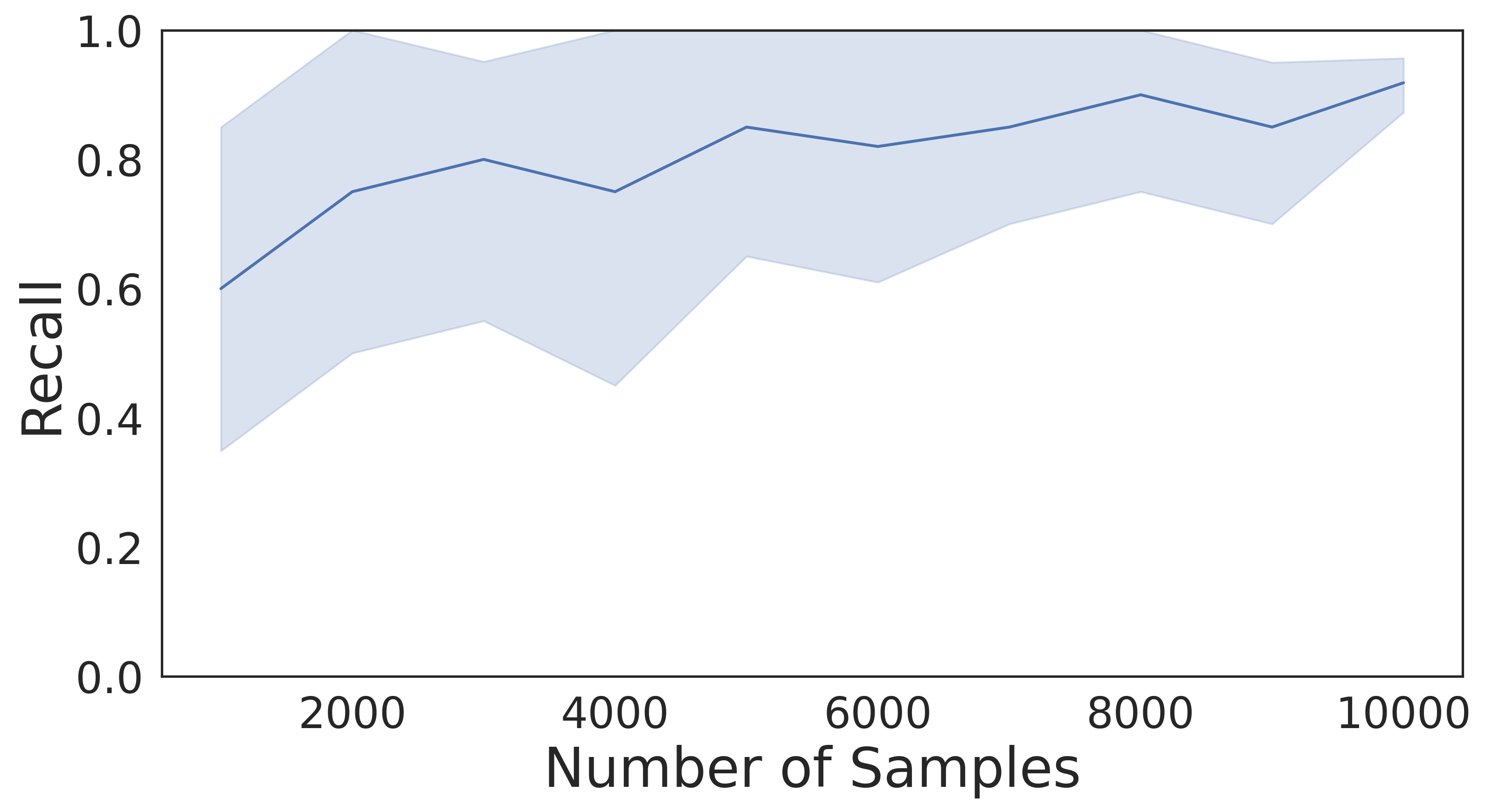}
        \label{fig:sample_recall}
    \end{subfigure}
    \vspace{-0.6em}
    \caption{Precision and recall for selection pairs w.r.t. different sample sizes.}
\label{fig:sample}
\end{figure}

\looseness=-1
In this section, we provide additional experimental results to further explore the proposed framework. We conduct experiments on synthetic datasets with sample size in $\{1000, 2000, \ldots, 10000\}$. All results are from 10 trials with different random seeds. Although our results focus on asymptotic properties, the algorithm can still identify most selection pairs when the sample size is relatively low. As expected, the variance of both precision and recall decreases as the sample size increases.




\end{document}